\documentclass[11pt]{article}
\usepackage[protrusion=false]{microtype} % to try to ensure that the margins are respected

\usepackage{xcolor}
\definecolor{darkblue}{rgb}{0.0, 0.0, 0.4}
\usepackage[colorlinks=true,allcolors=darkblue]{hyperref}

\usepackage{graphicx}
\usepackage{booktabs} % for professional tables

\usepackage{amsmath}
\usepackage{amssymb}
\usepackage{amsthm,thmtools}
\usepackage[shortlabels]{enumitem}
\usepackage{amsfonts}
\RequirePackage{algorithm}
\RequirePackage{algorithmic}
\usepackage{bm,upgreek}
\usepackage{mathrsfs}  

\usepackage[round]{natbib}

\usepackage{calc}% http://ctan.org/pkg/calc

\usepackage{subcaption}

\usepackage{authblk}

\newtheorem{lemma}{Lemma}
\newtheorem{theorem}{Theorem}

\newtheorem{definition}{Definition}

\DeclareMathOperator{\E}{E}
\DeclareMathOperator{\sym}{Sym}
\newcommand{\matr}[1]{\bm{#1}}

\DeclareFontFamily{U}{jkpmia}{}
\DeclareFontShape{U}{jkpmia}{m}{it}{<->s*jkpmia}{}
\DeclareFontShape{U}{jkpmia}{bx}{it}{<->s*jkpbmia}{}
\DeclareMathAlphabet{\mathfrak}{U}{jkpmia}{m}{it}
\SetMathAlphabet{\mathfrak}{bold}{U}{jkpmia}{bx}{it}
\usepackage[hang,flushmargin]{footmisc} % remove footnote indentation
\newcommand\blfootnote[1]{% footnote without number
  \begingroup
  \renewcommand\thefootnote{}\footnote{#1}%
  \addtocounter{footnote}{-1}%
  \endgroup
}

\title{Adversarial Monte Carlo Meta-Learning of Optimal Prediction Procedures}

\date{Version 2 (this version): September 25, 2020 \\ Version 1: February 27, 2020}

\author[1]{Alex  Luedtke}
\author[1]{Incheoul Chung}

\affil[1]{\footnotesize Department of Statistics, University of Washington, USA}

\author[2]{Oleg Sofrygin}
\affil[2]{\footnotesize   Kaiser Permanente Division of Research, Kaiser Permanente Northern California, USA} 

\begin{document}

\allowdisplaybreaks

\maketitle

\blfootnote{The authors thank Devin Didericksen for help in the early stages of this project. Generous support was provided by Amazon through an AWS Machine Learning Research Award and the NIH under award number DP2-LM013340. The content is solely the responsibility of the authors and does not necessarily represent the official views of Amazon or the NIH.}

\begin{abstract}
We frame the meta-learning of prediction procedures as a search for an optimal strategy in a two-player game. In this game, Nature selects a prior over distributions that generate labeled data consisting of features and an associated outcome, and the Predictor observes data sampled from a distribution drawn from this prior. The Predictor's objective is to learn a function that maps from a new feature to an estimate of the associated outcome. We establish that, under reasonable conditions, the Predictor has an optimal strategy that is equivariant to shifts and rescalings of the outcome and is invariant to permutations of the observations and to shifts, rescalings, and permutations of the features. We introduce a neural network architecture that satisfies these properties. The proposed strategy performs favorably compared to standard practice in both parametric and nonparametric experiments.
\end{abstract}

\section{Introduction}
\subsection{Problem Formulation}
Consider a data set consisting of observations $(X_1,Y_1),\ldots,(X_n,Y_n)$ drawn independently from a distribution $P$ belonging to some known 
model $\mathcal{P}$, where each $X_i$ is a continuously distributed feature with support contained in $\mathcal{X}:= \mathbb{R}^p$ and each $Y_i$ is an outcome with support contained in $\mathcal{Y}:= \mathbb{R}$. This data set can be written as $\matr{D}:=(\matr{X},\matr{Y})$, where $\matr{X}$ is the $n\times p$ matrix for which row $i$ contains $X_i$ and $\matr{Y}$ is the $n$-dimensional vector for which entry $i$ contains $Y_i$. The support of $\matr{D}$ is contained in $\mathcal{D}:=\mathcal{X}^n\times\mathcal{Y}^n$. The objective is to develop an estimator of the regression function $\mu_P$ that maps from $x_0$ to $\E_P[Y|X=x_0]$. An estimator $T$ belongs to the collection $\mathcal{T}$ of operators that take as input a data set $\matr{d}:=(\matr{x},\matr{y})$ and output a prediction function $T(\matr{d}) : \mathcal{X}\rightarrow\mathbb{R}$. Examples of estimators include the generalized linear models \citep{Nelder&Wedderburn1972}, random forests \citep{Breiman2001}, and gradient boosting machines \citep{Friedman2001}. We will also refer to estimators as prediction procedures. We focus on the case that the performance of an estimator is quantified via the standardized mean-squared error (MSE), namely
\begin{align*}
R(T,P) := \E_{P}\left[\int \tfrac{\left[T(\matr{D})(x_0)-\mu_P(x_0)\right]^2}{\sigma_P^{2}} dP_X(x_0)\right],
\end{align*}
where the expectation above is over the draw of $\matr{D}$ under sampling from $P$, $P_X$ denotes the marginal distribution of $X$ implied by $P$, and $\sigma_P^2$ denotes the variance of the error $\epsilon_P:=Y-\mu_P(X)$ when $(X,Y)\sim P$. Note that $\epsilon_P$ may be heteroscedastic. Throughout we assume that, for all $P\in\mathcal{P}$, $\E_P[Y^2]<\infty$ and $\epsilon_P$ is continuous. Note that the continuity of $\epsilon_P$ implies that $Y$ is continuous and that $\sigma_P^2>0$.

In practice, the distribution $P$ is not known, and therefore the risk $R(T,P)$ of a given estimator $T$ is also not known. We now describe three existing criteria for judging the performance of $T$ that do not rely on knowledge of $P$. The first criterion is the maximal risk $\sup_{P\in\mathcal{P}} R(T,P)$. If $T$ minimizes the maximal risk over $\mathcal{T}$, then $T$ is referred to as a minimax estimator. The second criterion is Bayesian in nature, namely the average of the risk $R(T,P)$ over draws of $P$ from a given prior $\Pi$ on $\mathcal{P}$. Specifically, this Bayes risk is defined as $r(T,\Pi):=\E_{\Pi}[R(T,P)]$. The estimator $T$ is referred to as a $\Pi$-Bayes estimator if it minimizes the Bayes risk over $\mathcal{T}$. A $\Pi$-Bayes estimator optimally incorporates the prior beliefs encoded in $\Pi$ with respect to the Bayes risk $r(\cdot,\Pi)$. Though this optimality property is useful in settings where $\Pi$ only encodes substantive prior knowledge, its utility is less clear otherwise. Indeed, as the function $r(\cdot,\Pi)$ generally depends on the choice of $\Pi$, it is possible that a $\Pi$-Bayes estimator $T$ is meaningfully suboptimal with respect to some other prior $\Pi'$, that is, that $r(T,\Pi)\gg \inf_{T'} r(T',\Pi')$; this phenomenon is especially prevalent when the sample size $n$ is small. Therefore, in settings where there is no substantive reason to favor a particular choice of $\Pi$, it is sensible to seek another approach for judging the performance of $T$. A natural criterion is the worst-case Bayes risk of $T$ over some user-specified collection $\Gamma$ of priors, namely $\sup_{\Pi\in \Gamma} r(T,\Pi)$. This criterion is referred to as the $\Gamma$-maximal Bayes risk of $T$. If $T$ is a minimizer of the $\Gamma$-maximal Bayes risk, then $T$ is referred to as a $\Gamma$-minimax estimator \citep{Berger1985}. Notably, in settings where $\Gamma$ contains all distributions with support in $\mathcal{P}$, 
the $\Gamma$-maximal Bayes risk is equivalent to the maximal risk. Consequently, in this special case, an estimator is $\Gamma$-minimax if and only if it is minimax. In settings where $\Gamma=\{\Pi\}$, an estimator is $\Gamma$-minimax if and only if it is $\Pi$-Bayes. Therefore, $\Gamma$-minimaxity provides a means of interpolating between the minimax and Bayesian criteria.

Though $\Gamma$-minimax estimators represent an appealing compromise between the Bayesian and minimax paradigms, they have seen limited use in practice because they are rarely available in closed form. In this work, we aim to overcome this challenge in the context of prediction by providing an iterative strategy for learning $\Gamma$-minimax prediction procedures. Due to the potentially high computational cost of this iterative scheme, a key focus of our work involves identifying conditions under which we can identify a small subclass of $\mathcal{T}$ that still contains a $\Gamma$-minimax estimator. This then makes it possible to optimize over this subclass, which we show in our experiments can dramatically improve the performance of our iterative scheme given a fixed computational budget.

Hereafter we refer to $\Gamma$-minimax estimators as `optimal', where it is to be understood that this notion of optimality relies on the choice of $\Gamma$.

\subsection{Overview of Our Strategy and Our Contributions}

Our strategy builds on two key results, each of which will be established later in this work. First, under conditions on $\mathcal{T}$ and $\Gamma$, there is a $\Gamma$-minimax estimator in the subclass $\mathcal{T}_e\subset\mathcal{T}$ of estimators that are 
equivariant to shifts and rescalings of the outcome and are invariant to permutations of the observations and to shifts, rescalings, and permutations of the features. Second, under further conditions, there is an equilibrium point $(T^\star,\Pi^\star)\in \mathcal{T}_e\times \Gamma$ such that 
\begin{align}
    \sup_{\Pi\in \Gamma} r(T^\star,\Pi)= r(T^\star,\Pi^\star) = \inf_{T\in \mathcal{T}_e} r(T,\Pi^\star). \label{eq:nash}
\end{align}
Upper bounding the right-hand side by $\sup_{\Pi\in\Gamma} \inf_{T\in \mathcal{T}_e} r(T,\Pi)$ and applying the max-min inequality shows that $T^\star$ is $\Gamma$-minimax. 
To find an equilibrium numerically, we propose to use adversarial Monte Carlo meta-learning (AMC)  \citep{Luedtkeetal2019}
to iteratively update an estimator in $\mathcal{T}_e$ and a prior in $\Gamma$. AMC is a form of stochastic gradient descent ascent \citep[e.g.,][]{Linetal2019} that can be used to learn optimal statistical procedures in general decision problems.

We make the following contributions:
\begin{itemize}[itemsep=0em,leftmargin=*]
    \item In Section~\ref{sec:characterization}, we characterize several equivariance properties of optimal estimators for a wide range of $(\mathcal{T},\Gamma)$.
    \item In Section~\ref{sec:amc}, we present a general framework for adversarially learning optimal prediction procedures.
    \item In Section~\ref{sec:architecture}, we present a novel neural network architecture for parameterizing estimators that satisfy the equivariance properties established in Section~\ref{sec:characterization}.
    \item In Section~\ref{sec:examples}, we apply our algorithm in two settings and learn estimators that outperform standard approaches in numerical experiments. In Section~\ref{sec:data}, we also evaluate the performance of these learned estimators in data experiments.
\end{itemize}
All proofs for the results in the above sections can be found in Section~\ref{sec:proofs}. Section~\ref{sec:extensions} describes possible extensions and provides concluding remarks.

To maximize the accessibility of our main theoretical results, we do not use group theoretic notation when presenting them in Sections \ref{sec:characterization} through \ref{sec:architecture}. However, when proving these results, we will heavily rely on tools from group theory; consequently, we adopt this notation in Section~\ref{sec:proofs}.

\subsection{Related Works}

The approach proposed in this work is a form of meta-learning \citep{Schmidhuber1987,Thrun&Pratt1998,Vilaltaetal2002}, where here each task is a regression problem. Our approach bears some similarity to existing approaches for learning supervised learning procedures against fixed priors, that is, in the special case that $\Gamma=\{\Pi\}$ for some fixed, user-specified prior $\Pi$. \citet{Hochreiteretal2001} advocated parameterizing $\mathcal{T}$ as a collection of long short-term (LSTM) networks \citep{Hochreiter&Schmidhuber1997}. 
A more recent work has advocated using memory-augmented neural networks rather than LSTMs in meta-learning tasks \citep{Santoroetal2016}. 
There have also been other works on the meta-learning of supervised learning procedures that are parameterized as neural networks \citep{Bosc2016,Vinyalsetal2016,Ravi&Larochelle2016}. 
Compared to these works, we \textit{adversarially} learn a prior $\Pi$ from a collection $\Gamma$ of priors, and we also formally characterize equivariance properties that will be satisfied by any optimal prediction procedure in a wide variety of problems. This characterization leads us to develop a neural network architecture designed for the prediction settings that we consider.

Model-agnostic meta-learning (MAML) is another popular meta-learning approach \citep{Finnetal2017}. In our setting, MAML aims to initialize the weights of a regression function estimate (parameterized as a neural network, for example) in such a way that, on any new task, only a limited number of gradient updates are needed. 
More recent approaches leverage the fact that, in certain settings, the initial estimate can instead be updated using a convex optimization algorithm \citep{Bertinettoetal2018,Leeetal2019}. 
To run any of these approaches, a prespecified prior over tasks is required. In our setting, these tasks take the form of data-generating distributions $P$. 
In contrast, our approach adversarially selects a prior from $\Gamma$. 

Two recent works \citep{Yinetal2018,Goldblumetal2019} developed meta-learning procedures that are trained under a different adversarial regime than that studied in the current work, namely under adversarial manipulation of one or both of the data set $\matr{D}$ and evaluation point $x_0$ \citep{Dalvietal2004}. 
This adversarial framework appears to be most useful when there truly is a malicious agent that aims to contaminate the data, which is not the case that we consider.
In contrast, in our setting, the adversarial nature of our framework allows us to ensure that our procedure will perform well regardless of the true value of $P$, while also taking into account prior knowledge that we may have.

Our approach is also related to existing works in the statistics and econometrics literatures on the numerical learning of minimax and $\Gamma$-minimax statistical decision rules. 
In finite-dimensional models, early works showed that it is possible to numerically learn minimax rules \citep{Nelson1966,Kempthorne1987} and, in settings where $\Gamma$ consists of all priors that satisfy a finite number of generalized moment conditions, $\Gamma$-minimax rules \citep{Noubiap&Seidel2001}. 
Other works have studied the $\Gamma$-minimax case where $\Gamma$ consists of priors that only place mass on a pre-specified finite set of distributions in $\mathcal{P}$, both for general decision problems \citep{Chamberlain2000} and for constructing confidence intervals \citep{Schafer2009}. Defining $\Gamma$ in this fashion modifies the statistical model $\mathcal{P}$ to only consist of finitely many distributions, which can be restrictive. A recent work introduced a new approach, termed AMC, for learning minimax procedures for general models $\mathcal{P}$ \citep{Luedtkeetal2019}. In contrast to earlier works, AMC does not require the explicit computation of a Bayes rule under any given prior, thereby improving the feasibility of this approach in moderate-to-high dimensional models. In their experiments, \citet{Luedtkeetal2019} used neural network classes to define the sets of allowable statistical procedures. 
Unlike the current work, none of the aforementioned studies identified or leveraged the equivariance properties that characterize optimal procedures. As we will see in our experiments, leveraging these properties can dramatically improve performance.

\subsection{Notation}
We now introduce the notation and conventions that we use. For a function $f : \mathcal{P}\rightarrow\mathcal{P}$, we let $\Pi\circ f^{-1}$ denote the pushforward measure that is defined as the distribution of $f(P)$ when $P\sim \Pi$. For any data set $\matr{d}=(\matr{x},\matr{y})$ and mapping $f$ with domain $\mathcal{D}$, we let $f(\matr{x},\matr{y}):= f(\matr{d})$. We take all vectors to be column vectors when they are involved in matrix operations. We write $\odot$ to mean the entrywise product and $a^{\odot 2}$ to mean $a\odot a$. For an $m_1\times m_2$ matrix $a$, we let $a_{i*}$ denote the $i^{\rm th}$ row, $a_{*j}$ denote the $j^{\rm th}$ column,  $\bar{a}:=\frac{1}{m_1}\sum_{i=1}^{m_1} a_{i*}$, and $s(a)^2:= \frac{1}{m_1}\sum_{i=1}^{m_1} (a_{i*}-\bar{a})^{\odot 2}$. When we standardize a vector $a$ as $[a-\bar{a}]/s(a)$, we always use the convention that $0/0=0$. 
We write $[a \,|\, b]$ to denote the column concatenation of two matrices. For an $m_1\times m_2\times m_3$ array $a$, we let $a_{i**}$ denote the $m_2\times m_3$ matrix with entry $(j,k)$ equal to $a_{ijk}$, $a_{i*k}$ denote the $m_2$-dimensional vector with entry $j$ equal to $a_{ijk}$, etc. For $a\in\mathbb{R}$ and $b\in\mathbb{R}^k$, we write $a+b$ to mean $a\matr{1}_k + b$.

\section{Characterization of Optimal Procedures}\label{sec:characterization}

\subsection{Optimality of Equivariant Estimators}\label{sec:optEquivar}

We start by presenting conditions that we impose on the collection of priors $\Gamma$. Let $\mathcal{A}$ denote the collection of all $n\times n$ permutation matrices, and let $\mathcal{B}$ denote the collection of all $p\times p$ permutation matrices. 
We suppose that $\Gamma$ is preserved under the following transformations:
\begin{enumerate}[series=model,label=P\arabic*.,ref=P\arabic*]
  \item \label{in:permpred} \textit{Permutations of features:} $\Pi\in\Gamma$ and $B\in\mathcal{B}$ implies that $\Pi\circ f_1^{-1}\in \Gamma$, where $f_1(P)$ is the distribution of $(B X,Y)$ when $(X,Y)\sim P$.
  \item \label{in:shiftrescalepred} \textit{Shifts and rescalings of features:} $\Pi\in\Gamma$, $a\in\mathbb{R}^p$, and $b\in(\mathbb{R}^+)^p$ implies that $\Pi\circ f_2^{-1}\in \Gamma$, where $f_2(P)$ is the distribution of $(a + b\odot X,Y)$ when $(X,Y)\sim P$.
  \item\label{in:shiftrescaleoutcome} \textit{Shift and rescaling of outcome:} $\Pi\in\Gamma$ and $\tilde{a}\in\mathbb{R}$ and $\tilde{b}>0$ implies that $\Pi\circ f_3^{-1}\in \Gamma$, where $f_3(P)$ is the distribution of $(X,\tilde{a}+\tilde{b}Y)$ when $(X,Y)\sim P$.
\end{enumerate}
The above conditions implicitly encode that $f_1(P)$, $f_2(P)$, and $f_3(P)$ all belong to $\mathcal{P}$ whenever $P\in\mathcal{P}$. 
Section~\ref{sec:groupHS} provides an alternative characterization of \ref{in:permpred}, \ref{in:shiftrescalepred}, and \ref{in:shiftrescaleoutcome} in terms of the preservation of $\Gamma$ under a certain group action.

We also assume that the signal-to-noise ratio (SNR) is finite --- 
this condition is important in light of the fact that the MSE risk that we consider standardizes by $\sigma_P^2$.
\begin{enumerate}[resume*=model]
  \item\label{as:finitesignal2} \textit{Finite SNR:} $\sup_{P\in\mathcal{P}} {\rm var}_P(\mu_P(X))/\sigma_P^2<\infty$.
\end{enumerate}

We now present conditions that we impose on the class of estimators $\mathcal{T}$. In what follows we let 
$\mathcal{D}_0:=\{(\matr{d},x_0)\in\mathcal{D}\times\mathcal{X} : s(\matr{y})\not=0,s(\matr{x})\not=\matr{0}_p\}$. For $(\matr{d},x_0)\in\mathcal{D}_0$, we let
\begin{align*}
    z(\matr{d},x_0):= \left(\frac{\matr{x}-\bar{\matr{x}}}{s(\matr{x})},\frac{\matr{y}-\bar{\matr{y}}}{s(\matr{y})},\frac{x_0-\bar{\matr{x}}}{s(\matr{x})},\bar{\matr{x}},\bar{\matr{y}},s(\matr{x}),s(\matr{y})\right),
\end{align*}
where we abuse notation and let $\frac{\matr{x}-\bar{\matr{x}}}{s(\matr{x})}$ represent the $n\times p$ matrix for which row $i$ is equal to $[x_i-\bar{\matr{x}}]/s(\matr{x})$. We let $\mathcal{Z}:=\{z(\matr{d},x_0) : (\matr{d},x_0)\in\mathcal{D}_0\}$. When it will not cause confusion, we will write $\matr{z}:=z(\matr{d},x_0)$. Fix $T\in\mathcal{T}$. Let $S_T : \mathcal{Z}\rightarrow\mathbb{R}$ denote the unique function that satisfies
\begin{align}
    T(\matr{d})(x_0) =\bar{\matr{y}} + s(\matr{y}) S_T\left(\matr{z}\right) \ \textnormal{ for all $(\matr{d},x_0)\in\mathcal{D}_0$.} \label{eq:STdef}
\end{align}
The uniqueness arises because $s(\matr{y})\not=0$ on $\mathcal{D}_0$. Because we have assumed that $X$ and $Y$ are continuous for all $P\in\mathcal{P}$, it follows that, for all $P\in\mathcal{P}$, the class $\mathcal{S}:=\{S_T : T\in\mathcal{T}\}$ uniquely characterizes the functions in $\mathcal{T}$ up to their behavior on subsets of $\mathcal{D}\times\mathcal{X}$ of $P$-probability zero. In what follows, we will impose smoothness constraints on $\mathcal{S}$, which in turn imposes constraints on $\mathcal{T}$. 
The first three conditions suffice to show that $\mathcal{S}$ is compact in the space $C(\mathcal{Z},\mathbb{R})$ of continuous $\mathcal{Z}\rightarrow \mathbb{R}$ functions equipped with the compact-open topology.
\begin{enumerate}[series=est,label=T\arabic*.,ref=T\arabic*]
  \item\label{as:ptwiseBdd} \textit{$\mathcal{S}$ is pointwise bounded:} For all $\matr{z}\in\mathcal{Z}$, $\sup_{S\in\mathcal{S}}|S(\matr{z})|<\infty$.
  \item\label{as:holder2} \textit{$\mathcal{S}$ is locally H\"{o}lder:} For all compact sets $\mathcal{K}\subset\mathcal{Z}$, there exists an $\alpha\in(0,1)$ such that
  \begin{align*}
  \sup_{S\in\mathcal{S},\matr{z}\not=\matr{z}'\in\mathcal{K}}\frac{|S(\matr{z})-S(\matr{z}')|}{\|\matr{z}-\matr{z}'\|_2^\alpha}<\infty, 
  \end{align*}
  where $\|\cdot\|_2$ denotes the Euclidean norm. We take the supremum to be zero if $\mathcal{K}$ is a singleton or is empty.
  \item\label{as:Sclosed2} \textit{$\mathcal{S}$ is sequentially closed in the topology of compact convergence:} If $\{S_j\}_{j=1}^\infty$ is a sequence in $\mathcal{S}$ and $S_j\rightarrow S$ compactly in the sense that, for all compact $\mathcal{K}\subset \mathcal{Z}$, $\sup_{\matr{z}\in\mathcal{K}}|S_j(\matr{z})-S(\matr{z})|\rightarrow 0$, then $S\in\mathcal{S}$.
\end{enumerate}
The following conditions ensure that $\mathcal{S}$ is invariant to certain preprocessings of the data, in the sense that, for any function $S\in \mathcal{S}$, the function that first preprocesses the data in an appropriate fashion and then applies $S$ to this data is itself in $\mathcal{S}$. 
When formulating these conditions, we write
$z(\matr{d},x_0)$ to mean an element of $\mathcal{Z}$. 
Because $z$ is a bijection between $\mathcal{D}_0$ and $\mathcal{Z}$,
 it is possible to recover $(\matr{d},x_0)$ from $z(\matr{d},x_0)$.  
\begin{enumerate}[resume*=est]
  \item\label{as:SpreservedPerms} \textit{Permutations:} For all $S\in\mathcal{S}$, $A\in\mathcal{A}$, and $B\in\mathcal{B}$, $z(\matr{d},x_0)\mapsto S(z((A\matr{x} B,A\matr{y}),B^\top x_0))$ is in $\mathcal{S}$.
  \item\label{as:SpreservedShifts} \textit{Shifts and rescalings:} For all $S\in\mathcal{S}$, $a\in\mathbb{R}^p$, $b\in (\mathbb{R}^+)^p$, $\tilde{a}\in\mathbb{R}$, and $\tilde{b}>0$, the function $z(\matr{d},x_0)\mapsto S(z((\matr{x}^{a,b},\tilde{a}+\tilde{b}\matr{y}),a + b\odot x_0))$ is in $\mathcal{S}$, where $\matr{x}^{a,b}$ is the $n\times p$ matrix with row $i$ equal to $a + b\odot \matr{x}_{i*}$.
\end{enumerate}
Conditions \ref{as:ptwiseBdd}-\ref{as:Sclosed2} are satisfied if, for some $c,\alpha>0$ and $F : \mathcal{Z}\rightarrow \mathbb{R}^+$, $\mathcal{S}$ is the collection of all $S : \mathcal{Z}\rightarrow \mathbb{R}$ such that $ |S(\matr{z})|\le F(\matr{z})$ and $|S(\matr{z})-S(\matr{z}')|\le c \|\matr{z}-\matr{z}'\|_2^\alpha$  for all $\matr{z},\matr{z}'\in\mathcal{Z}$.
If $F\circ z$ is also invariant to permutations, shifts, and rescalings, then \ref{as:SpreservedPerms} and \ref{as:SpreservedShifts} also hold. Conditions \ref{as:ptwiseBdd}-\ref{as:SpreservedShifts} are also satisfied by many other classes $\mathcal{S}$.

Let $\mathcal{T}_e\subseteq \mathcal{T}$ denote the class of estimators that are 
equivariant to shifts and rescalings of the outcome and are invariant to permutations of the observations and to shifts, rescalings, and permutations of the features. 
Specifically, $\mathcal{T}_e$ consists of functions in $\mathcal{T}$ satisfying the following properties for all pairs $(\matr{d},x_0)$ of data sets and features in $\mathcal{D}_0$, permutation matrices $A\in\mathcal{A}$ and $B\in\mathcal{B}$, shifts $a\in\mathbb{R}^p$ and $\tilde{a}\in\mathbb{R}$, and rescalings $b\in(\mathbb{R}^+)^p$ and $\tilde{b}>0$:
\begin{align}
    T(A\matr{x}B, A \matr{y})(B^\top x_0) &= T(\matr{d})(x_0), \label{eq:Tequivar1} \\
T(\matr{x}^{a,b}, \tilde{a} + \tilde{b} \matr{y})(a + b\odot x_0) &= \tilde{a} + \tilde{b} T(\matr{d})(x_0), \label{eq:Tequivar2}
\end{align}
The following result shows that the $\Gamma$-maximal risk is the same over $\mathcal{T}$ and $\mathcal{T}_e\subseteq\mathcal{T}$.
\begin{theorem}\label{thm:HS} Under \ref{in:permpred}-\ref{as:finitesignal2} and \ref{as:ptwiseBdd}-\ref{as:SpreservedShifts},
\begin{align*}
\inf_{T\in\mathcal{T}} \sup_{\Pi\in\Gamma} r(T,\Pi) = \inf_{T\in\mathcal{T}_e} \sup_{\Pi\in\Gamma} r(T,\Pi).
\end{align*}
\end{theorem}
The above is a variant of the Hunt-Stein theorem \citep{Hunt&Stein1946}. Our proof, which draws inspiration from \citet{LeCam2012}, consists in showing that our prediction problem is invariant to the action of an amenable group, and subsequently applying Day's fixed-point theorem \citep{Day1961} to show that, for all $T\in\mathcal{T}$, the collection of $T'$ for which $\sup_{\Pi\in\Gamma}r(T',\Pi)\le \sup_{\Pi\in\Gamma} r(T,\Pi)$ has nonempty intersection with $\mathcal{T}_e$.

\subsection{Focusing Only on Distributions with Standardized Predictors and Outcome}\label{sec:Gamma1}

Theorem~\ref{thm:HS} suggests restricting attention to estimators in $\mathcal{T}_e$ when trying to learn a $\Gamma$-minimax estimator. We now show that, once this restriction has been made, it also suffices to restrict attention to a smaller collection of priors $\Gamma_1$ when identifying a least favorable prior. In fact, we show something slightly stronger, namely that the restriction to $\Gamma_1$ can be made even if optimal estimators are sought over the richer class $\widetilde{\mathcal{T}}_e\supseteq\mathcal{T}_e$ of estimators that satisfy the equivariance property \eqref{eq:Tequivar2} but do not necessarily satisfy \eqref{eq:Tequivar1}.

We now define $\Gamma_1$. Let $h(P)$ denote the distribution of
\begin{align*}
    \left(\left(\frac{X_j-\E_P[X_j]}{{\rm var}_P(X_j)^{1/2}}\right)_{j=1}^p,\frac{Y-\E_P[Y]}{\sigma_P}\right)
\end{align*}
when $(X,Y)\sim P$. Note that here, and here only, we have written $X_j$ to denote the $j^{\rm th}$ feature rather than the $j^{\rm th}$ observation. Also let $\Gamma_1:=\{\Pi\circ h^{-1} : \Pi\in\Gamma\}$, which is a collection of priors on $\mathcal{P}_1:=\{h(P) : P\in\mathcal{P}\}$. 
\begin{theorem}\label{thm:Gamma1restriction}
If \ref{in:shiftrescalepred} and \ref{in:shiftrescaleoutcome} hold and all $T\in\mathcal{T}$ satisfy \eqref{eq:Tequivar2}, then $T^\star$ is $\Gamma$-minimax if and only if it is $\Gamma_1$-minimax.
\end{theorem}
We conclude by noting that, under \ref{in:shiftrescalepred} and \ref{in:shiftrescaleoutcome}, $\mathcal{P}_1$ consists precisely of those $P\in \mathcal{P}$ that satisfy:
\begin{align}
    &\E_P[X]=\matr{0}_p, \hspace{0.5em} \E_P[X^{\odot 2}]=\matr{1}_p, \hspace{0.5em} \E_P[Y]=0,\hspace{0.5em} \sigma_P^2=1. \label{eq:P1}
\end{align}

\subsection{Existence of an Equilibrium Point}
We also make the following additional assumption on $\mathcal{S}$.
\begin{enumerate}[resume*=est]
  \item\label{as:Sconv2} \textit{$\mathcal{S}$ is convex:} $S_1,S_2\in\mathcal{S}$ and $\alpha\in(0,1)$ implies that $\matr{z}\mapsto \alpha S_1(\matr{z}) + (1-\alpha) S_2(\matr{z})$ is in $\mathcal{S}$.
\end{enumerate}
We also impose the following condition on the size of the collection of distributions $\mathcal{P}_1$ and the collection of priors $\Gamma_1$, which in turn imposes restrictions on $\mathcal{P}$ and $\Gamma$.
\begin{enumerate}[resume*=model]
  \item\label{as:Mcompact} There exists a metric $\rho$ on $\mathcal{P}_1$ such that (i) $(\mathcal{P}_1,\rho)$ is a complete separable metric space, (ii) $\Gamma_1$ is tight in the sense that, for all $\varepsilon>0$, there exists a compact set $\mathcal{K}$ in $(\mathcal{P}_1,\rho)$ such that $\Pi(\mathcal{K})\ge 1-\varepsilon$ for all $\Pi\in\Gamma_1$, and (iii) for all $T\in\mathcal{T}_e$, $P\mapsto R(T,P)$ is upper semi-continuous and bounded from above on $(\mathcal{P}_1,\rho)$.
\end{enumerate}
In Appendix~\ref{app:metricDiscussion}, we give examples of parametric and nonparametric settings where \ref{as:Mcompact} is applicable.

So far, the only conditions that we have required on the $\sigma$-algebra $\mathscr{A}$ of $\mathcal{P}$ are that $h$ and $R(T,\cdot)$, $T\in\mathcal{T}$, are measurable. 
In this subsection, and in this subsection only, we add the assumptions that \ref{as:Mcompact} holds and that $\mathscr{A}$ is such that $\{A\cap \mathcal{P}_1 : A\in\mathscr{A}\}$ equals $\mathscr{B}_1$, where $\mathscr{B}_1$ is the collection of Borel sets on $(\mathcal{P}_1,\rho)$.

We will also assume the following two conditions on $\Gamma_1$.
\begin{enumerate}[resume*=model]
    \item \label{as:Gamma1closed} \textit{$\Gamma_1$ is closed in the topology of weak convergence:} if $\{\Pi_j\}_{j=1}^\infty$ is a sequence in $\Gamma_1$ that converges weakly to $\Pi$, then $\Pi\in\Gamma_1$.
    \item \label{as:Gamma1conv} \textit{$\Gamma_1$ is convex:} for all $\Pi_1,\Pi_2\in\Gamma$ and $\alpha\in(0,1)$, the mixture distribution $\alpha\Pi_1+(1-\alpha)\Pi_2$ is in $\Gamma$.
\end{enumerate}
Under Conditions~\ref{as:Mcompact} and \ref{as:Gamma1closed}, Prokhorov's theorem \citep{Billingsley1999} can be used to establish that $\Pi_1$ is compact in the topology of weak convergence. This compactness will be useful for proving the following result, which shows that there is an equilibrium point under our conditions.
\begin{theorem}\label{thm:equilibrium}
If \ref{as:ptwiseBdd}-\ref{as:Sclosed2}, \ref{as:Sconv2}, and \ref{in:shiftrescalepred}-\ref{as:Gamma1conv} hold, then there exists $T^\star\in \mathcal{T}_e$ and $\Pi^\star\in\Gamma_1$ such that, for all $T\in\mathcal{T}_e$ and $\Pi\in\Gamma_1$, it is true that $r(T^\star,\Pi)\le r(T^\star,\Pi^\star)\le r(T,\Pi^\star)$.
\end{theorem}
Combining the above with Lemma~\ref{lem:Gamma1} in Section~\ref{sec:Gamma1proofs} establishes \eqref{eq:nash}, that is, that the conclusion of Theorem~\ref{thm:equilibrium} remains valid if $\Pi$ varies over $\Gamma$ rather than over $\Gamma_1$.

\section{AMC Meta-Learning Algorithm}\label{sec:amc}
We now present an AMC meta-learning strategy for obtaining a $\Gamma$-minimax estimator within some class $\mathcal{T}$. Here we suppose that $\mathcal{T}=\{T_t : t\in\tau\}$, where each $T_t$ is an estimator indexed by a finite-dimensional parameter $t$ that belongs to some set $\tau$. We note that this framework encapsulates: model-based approaches \citep[e.g.,][]{Hochreiteretal2001}, where $T_t$ can be evaluated by a single pass of $(\matr{d},x_0)$ through a neural network with weights $t$; optimization-based approaches, where $t$ are the initial weights of some estimate that are subsequently optimized based on $\matr{d}$ \citep[e.g.,][]{Finnetal2017}; and metric-based approaches, where $t$ indexes a measure of similarity $\alpha_t$ that is used to obtain an estimate of the form $\sum_{i=1}^n \alpha_t(x_i,x_0) y_i$ \citep[e.g.,][]{Vinyalsetal2016}.

We suppose that all estimators in $\mathcal{T}$ satisfy the equivariance property \eqref{eq:Tequivar2}, which can be arranged by prestandardizing the outcome and features and then poststandardizing the final prediction --- see Algorithm~\ref{alg:architecture} for an example. 
Since all $T\in\mathcal{T}$ satisfy \eqref{eq:Tequivar2}, Theorem~\ref{thm:Gamma1restriction} shows that it suffices to consider a collection $\Gamma_1$ of priors with support on $\mathcal{P}_1$, that is, so that, for all $\Pi\in \Gamma_1$, $P\sim \Pi$ satisfies \eqref{eq:P1} almost surely. 
To ensure that the priors 
are easy to sample from, we parameterize them via generator functions $G_g$ \citep{Goodfellowetal2014} that are indexed by a finite-dimensional $g$ that belongs to some set $\gamma$. Each $G_g$ takes as input a source of noise $U$ drawn from a distribution $\nu_u$ and outputs the parameters indexing a distribution in $\mathcal{P}$ \citep{Luedtkeetal2019}. 
Though this form of sampling limits to parametric families $\mathcal{P}$, the number of parameters indexing this family may be much larger than the sample size $n$, which can, for all practical purposes, lead to a nonparametric estimation problem. For each $g$, we let $\Pi_g$ denote the distribution of $G_g(U)$ when $U\sim \nu_u$. We then let $\Gamma_1=\{\Pi_g : g\in\gamma\}$. It is worth noting that classes $\Gamma_1$ that are defined in this way will not generally satisfy the conditions \ref{as:Mcompact}-\ref{as:Gamma1conv} used in Theorem~\ref{thm:equilibrium}. To iteratively improve the performance of the prior, we require the ability to differentiate realized data sets through the parameters indexing the prior. To do this, we assume that, for each $P\in\mathcal{P}$, the user has access to a generator function $H_P : \mathcal{V}\rightarrow\mathbb{R}$ such that $H_P(V)$ has the same distribution as $(X,Y)\sim P$ when $V\sim \nu_v$. 

\begin{algorithm}
   \caption{Adversarially learn an estimator.}
   \label{alg:updateBoth}
\begin{algorithmic}[1]
   \STATE \textbf{Initialize} estimator $T_t$, generator $G_g$, step sizes $\eta_1,\eta_2$.
    \FOR{$K$ iterations}
      \FOR{$j=1,2$}
          \STATE Independently draw $U\sim \nu_u$ and $V_0,\ldots,V_p\overset{iid}{\sim} \nu_v$.
          \STATE\label{ln:P} Let $P=G_g(U)$.
          \STATE\label{ln:XY} Let $(X_i,Y_i)= H_P(V_i)$, $i=0,1,\ldots,n$.
          \STATE\label{ln:D} Let $\matr{D}$ be the data set containing $(X_i,Y_i)_{i=1}^n$.
          \STATE\label{ln:loss} Let ${\rm Loss}= [T_t(\matr{D})(X_0)-\mu_{P}(X_0)]^2$
          \IF{j=1} 
            \STATE Update estimator:\hspace{.15em} $t = t - \eta_1 \nabla_t {\rm Loss}$.
            \STATE $\triangleright$ ${\rm Loss}$ depends on $t$ through $T_t$.
          \ELSE 
            \STATE\label{ln:priorUpdate} Update prior:\hspace{.15em} 
     $g=g + \eta_2 \nabla_{g} {\rm Loss}$.
            \STATE $\triangleright$ ${\rm Loss}$ depends on $g$ through the definitions of $P$, $(X_i,Y_i)$, and $\matr{D}$.
        \ENDIF
        \ENDFOR
    \ENDFOR
\end{algorithmic}
\end{algorithm}

The AMC learning strategy is presented in Algorithm~\ref{alg:updateBoth}. The algorithm takes stochastic gradient steps on the parameters indexing an estimator and prior generator to iteratively reduce and increase the Bayes risk, respectively. We caution that, when the outcome or some of the features are discrete, $\nabla_g {\rm Loss}$ will not generally represent an unbiased estimate of the gradient of $g\mapsto r(T_t,\Pi_g)$, which can cause Algorithm~\ref{alg:updateBoth} to perform poorly. To handle these cases, the algorithm can be modified to instead obtain an unbiased gradient estimate using the likelihood ratio method \citep{Glynn1987}.

Though studying the convergence properties of the minimax optimization in Algorithm~\ref{alg:updateBoth} is not the main focus of this work, we now provide an overview of how results from \cite{Linetal2019} can be used to provide some guarantees for this algorithm. When doing so, we focus on the special case where there exists some $\ell<\infty$ such that, for all $g$, $t\mapsto r(T_t,G_g)$ is differentiable with $\ell$-Lipschitz gradient and, for some finite (but potentially large) collection $\mathcal{P}_D:=\{P_1,\ldots,P_D\}\subset \mathcal{P}$, $\Gamma$ is the collection of all mixtures of distributions in $\mathcal{P}_D$. We also suppose that the parameter $g$ indexing the generator $G_g$ takes values on the $D-1$ simplex and that this generator is parameterized in such a way that $\nu_u\circ G_g^{-1}$ has the same distribution as the mixture of distributions in $\mathcal{P}_D$ that places mass $g_j$ on distribution $P_j$, $j=1,\ldots,D$. In this case, provided the learning rates $\eta_1$ and $\eta_2$ are chosen appropriately,  Theorem~4.5 in \cite{Linetal2019} gives guarantees on the number of iterations required to return an $\epsilon$-stationary point $T_{t_K}$ (idem, Definition~3.7) within a specified number of iterations --- this stationary point is such that there exists a $t'$ near $t_K$ at which the function $t\mapsto \sup_{\pi\in\Gamma} r(T_{t},\Pi)$ has at least one small subgradient (idem, Lemma~3.8, for details). If, also, $t\mapsto T_t(\matr{d})$ is convex for all $\matr{d}$, then this also implies that $T_{t_K}$ is nearly $\Gamma$-minimax. If, alternatively, the prior update step in Algorithm~\ref{alg:updateBoth} (line~\ref{ln:priorUpdate}) is replaced by an oracle optimizer such that, at each iteration, $g$ is defined as a true maximizer of the Bayes risk $g\mapsto r(T,\Pi_g)$, then Theorem~E.4 of \cite{Linetal2019} similarly guarantees that an $\epsilon$-stationary point will be reached within a specified number of iterations.

\section{Proposed Class of Estimators}\label{sec:architecture}

\subsection{Equivariant Estimator Architecture}

Algorithm~\ref{alg:architecture} presents our proposed estimator architecture, 
which relies on four modules. Each module $k$ can be represented as a function $m_k$ belonging to a collection $\mathcal{M}_k$ of functions mapping from $\mathbb{R}^{a_k}$ to $\mathbb{R}^{b_k}$, where the values of $a_k$ and $b_k$ can be deduced from Algorithm~\ref{alg:architecture}. For given data $\matr{d}$, a prediction at a feature $x_0$ can be obtained by sequentially calling the modules and, between calls, either mean pooling across one of the dimensions of the output or concatenating the evaluation point as a new column in the output matrix.

\begin{algorithm}
   \caption{Use data $\matr{d}$ to obtain prediction at $x_0$.}
   \label{alg:architecture}
\begin{algorithmic}[1]
   \STATE \textbf{Preprocess:} Let $x_0^0 := \frac{x_0-\bar{\matr{x}}}{s(\matr{x})}$ and define $\matr{d}^0\in\mathbb{R}^{n\times p\times 2}$ so that $\matr{d}_{i*1}^0=\frac{x_i-\bar{\matr{x}}}{s(\matr{x})}$ for all $i=1,\ldots,n$ and $\matr{d}_{*j2}^0=\frac{\matr{y}-\bar{\matr{y}}}{s(\matr{y})}$ for all $j=1,\ldots,p$.\vspace{0.5em}\label{ln:preprocess}
            \STATE \textbf{Module 1:} $\matr{d}^1:= m_1(\matr{d}^0)$. \hfill $\matr{d}^1\in\mathbb{R}^{n\times p\times o_1}$
            \STATE \textbf{Mean Pool:} $\bar{\matr{d}}^1:=n^{-1}\sum_{i=1}^n \matr{d}_{i**}^1$.\vspace{0.5em} \hfill $\bar{\matr{d}}^1\in\mathbb{R}^{p\times o_1}$\vspace{0.25em}
            \STATE \textbf{Module 2:} $\matr{d}^2:=m_2(\bar{\matr{d}}^1)$. \hfill $\matr{d}^2\in\mathbb{R}^{p\times o_2}$\label{ln:module2}
            \STATE \textbf{Augment:} $\tilde{\matr{d}}^2:=[\matr{d}^2\ |\ x_0^0]$. \hfill $\tilde{\matr{d}}^2\in\mathbb{R}^{p\times (o_2+1)}$\vspace{0.5em}
            \STATE \textbf{Module 3:} $\matr{d}^3:=m_3(\tilde{\matr{d}}^2)$. \hfill $\matr{d}^3\in\mathbb{R}^{p\times o_3}$
            \STATE \textbf{Mean Pool:} $\bar{\matr{d}}^3:=p^{-1}\sum_{j=1}^p \matr{d}_{j*}^3$.\vspace{0.5em} \hfill $\bar{\matr{d}}^3\in\mathbb{R}^{o_3}$
            \STATE \textbf{Module 4:} $\matr{d}^4:=m_4(\bar{\matr{d}}^3)$. \hfill $\matr{d}^4\in\mathbb{R}$
            \STATE \textbf{return} $\bar{\matr{y}} + s(\matr{y}) \matr{d}^4$.
\end{algorithmic}
\end{algorithm}

We let $\mathcal{T}_{\mathcal{M}}$ represent the collection of all prediction procedures described by Algorithm~\ref{alg:architecture}, where here $(m_k)_{k=1}^4$ varies over $\prod_{k=1}^4 \mathcal{M}_k$. We now give conditions under which the proposed architecture yields an equivariant estimator. 

\begin{enumerate}[noitemsep,series=modules,label=M\arabic*),ref=M\arabic*]
    \item\label{it:M1equivar} $m_1(A v B)_{**\ell}= A[m_1(v)_{**\ell}] B$ for all $m_1\in\mathcal{M}_1$, $A\in \mathcal{A}$, $B\in \mathcal{B}$, $v\in\mathbb{R}^{n\times p\times 2}$, and $\ell\in\{1,\ldots,o_1\}$.
    \item\label{it:M2equivar} $m_2(B v) = B m_2(v)$ for all $m_2\in\mathcal{M}_2$, $B\in \mathcal{B}$, and $v\in\mathbb{R}^{p\times o_1}$.
    \item\label{it:M3equivar} $m_3(B v) = B m_3(v)$ for all $m_3\in\mathcal{M}_3$, $B\in \mathcal{B}$, and $v\in\mathbb{R}^{p\times o_2}$.
\end{enumerate}

\begin{theorem}\label{thm:TMequivar}
If \ref{it:M1equivar}-\ref{it:M3equivar}, then all $T\in \mathcal{T}_{\mathcal{M}}$ satisfy \eqref{eq:Tequivar1} and \eqref{eq:Tequivar2}.
\end{theorem}

\subsection{Neural Network Parameterization}\label{sec:neuralnetparam}
In our experiments, we choose the four module classes $\mathcal{M}_k$, $k=1,2,3,4$, indexing our estimator architecture to be collections of neural networks. For each $k$, we let $\mathcal{M}_k$ contain the neural networks consisting of $h_k$ hidden layers of widths $w_k^1,w_k^2,\ldots,w_k^{h_k}$, where the types of layers used depends on the module $k$. When $k=1$, multi-input-output channel equivariant layers as defined in \citet{Hartford2018} are used. 
In particular, for $j=1,\ldots,h_1+1$, we let $\mathcal{L}_1^{j}$ denote the collection of all such layers that map from $\mathbb{R}^{n\times p\times w_1^{j-1}}$ to $\mathbb{R}^{n\times p\times w_1^j}$, where we let $w_1^0=2$ and $w_1^{h_1+1}=o_1$. For each $j$, each member $L_1^j$ of $\mathcal{L}_1^j$ is equivariant in the sense that, for all $A\in\mathcal{A}$, $B\in\mathcal{B}$, and $v\in\mathbb{R}^{n\times p\times w_1^{j-1}}$, $L_1^j(A v B)_{**\ell} = A L_1^j(v)_{**\ell} B$ for all $\ell=1,\ldots,o_1$. When $k=2,3$, multi-input-output channel equivariant layers as described in Eq.~22 of \citet{Zaheer2017} are used, except that we replace the sum-pool term in that equation by a mean-pool term (see the next subsection for the rationale). In particular, for $j=1,\ldots,h_k+1$, we let $\mathcal{L}_k^j$ denote the collection of all such equivariant layers that map from $\mathbb{R}^{p\times w_k^{j-1}}$ to $\mathbb{R}^{p\times w_k^{j}}$. For each $j$, each member $L_k^j$ of $\mathcal{L}_k^j$ is equivariant in the sense that, for all $B\in\mathcal{B}$ and $v\in\mathbb{R}^{p\times w_k^{j-1}}$, $L_k^j(Bv) = B L_k^j(v)$. When $k=4$, standard linear layers mapping from $\mathbb{R}^{w_4^{j-1}}$ to $\mathbb{R}^{w_4^{j}}$ are used for each $j=1,\ldots,h_4+1$, where $w_4^0=o_3$ and $w_4^{h_4+1}=1$. For each $j$, we let $\mathcal{L}_4^j$ denote the collection of all such layers. For a user-specified activation function $q$, we then define the module classes as follows for $k=1,2,3,4$:
\begin{align*}
    &\mathcal{M}_k:=\{v\mapsto q\circ L_k^{h_k+1}\circ q\circ L_k^{h_k}\circ \ldots \circ q\circ L_k^{1}(v) : L_k^j\in \mathcal{L}_k^{j}, j=1,2,\ldots,h_k+1\}.
\end{align*}
Notably, $\mathcal{M}_1$ satisfies \ref{it:M1equivar} \citep{Ravanbakhsh2017,Hartford2018}, and $\mathcal{M}_2$ and $\mathcal{M}_3$ satisfy \ref{it:M2equivar} and \ref{it:M3equivar}, respectively \citep{Ravanbakhsh2016,Zaheer2017}. Each element of $\mathcal{M}_4$ is a multilayer perceptron.

\subsection{Pros and Cons of Proposed Architecture}\label{sec:proscons}

A benefit of using the proposed architecture in Algorithm~\ref{alg:architecture} is that Modules 1 and 2 can be evaluated without knowing the feature $x_0$ at which a prediction is desired. As a consequence, these modules can be precomputed before making predictions at new feature values, which can lead to substantial computational savings when the number of values at which predictions will be made is large. 
Another advantage of the proposed architecture is that it can be evaluated on a data set that has a different sample size $n$ than did the data sets used during meta-training. In the notation of Eq.~4 from \citeauthor{Hartford2018}, this corresponds to noting that the weights from an $\mathbb{R}^{N\times M\times k}\rightarrow \mathbb{R}^{N\times M\times o}$ multi-input-output channel layer can be used to define an $\mathbb{R}^{N'\times M\times k}\rightarrow \mathbb{R}^{N'\times M\times o}$ layer for which the output $Y_{n,m}^{\langle o \rangle}$ is given by the same symbolic expression as that displayed in Eq.~4 from that work, but now with $n$ ranging over $1,\ldots,N'$. We will show in our upcoming experiments that procedures trained using 500 observations can perform well even when evaluated on data sets containing only 100 observations. It is similarly possible to evaluate the proposed architecture on data sets containing a different number of features than did the data sets used during meta-training --- again see Eq.~4 in \citet{Hartford2018}, and also see Eq.~22 in \citet{Zaheer2017}, but with the sum-pool term replaced by a mean-pool term. The rationale for replacing the sum-pool term by a mean-pool term is that this will ensure that the scale of the hidden layers will remain fairly stable when the number of testing features differs somewhat from the number of training features.

A disadvantage of the proposed architecture is that it currently has no established universality guarantees. 
Such guarantees have been long available for standard multilayer perceptrons \citep[e.g.,][]{Cybenko1989,Hornik1991}, and have recently also become available for certain invariant architectures \citep{Maron2019}. In future work, it would be interesting to see if the arguments in \citet{Maron2019} can be modified to provide universality guarantees for our architecture. Establishing such results may also help us to overcome a second disadvantage of our architecture, namely that the resulting neural network classes will not generally satisfy the convexity condition \ref{as:Sconv2} used in Theorem~\ref{thm:equilibrium}. If a network class $\mathcal{T}_{\mathcal{M}}$ that we have proposed can be shown to satisfy a universality result for some appropriate convex class $\mathcal{T}_c$, and if $\mathcal{T}_{\mathcal{M}}$ is itself a subset of $\mathcal{T}_c$, then perhaps it will be possible to invoke Theorem~\ref{thm:equilibrium} to establish an equilibrium result over the class of estimators $\mathcal{T}_c$, and then to use this result to establish an (approximate) equilibrium result for $\mathcal{T}_{\mathcal{M}}$. To ensure that conditions \ref{as:ptwiseBdd}-\ref{as:Sclosed2} are satisfied, such an argument will likely require that the weights of the networks in $\mathcal{T}_{\mathcal{M}}$ be restricted to belong to some compact set.

\section{Numerical Examples}\label{sec:examples}

\subsection{Preliminaries}
In this section, we present the results from two sets of numerical experiments. 
In each example, the collection of estimators $\mathcal{T}$ is parameterized as the network architecture introduced in Section~\ref{sec:neuralnetparam} with $o_1=o_2=50$, $o_3=10$, $h_1=h_3=10$, $h_2=h_4=3$, and, for $k=1,2,3,4$, $w_k=100$. For each module, we use the leaky ReLU activation $q(z):=\max\{z,0\}+0.01\min\{z,0\}$.

The examples differ in the definitions of the model $\mathcal{P}$ 
and the collection $\Gamma$ of priors on $\mathcal{P}$. In each case, $\Gamma$ satisfies the invariance properties \ref{in:permpred}, \ref{in:shiftrescalepred}, and \ref{in:shiftrescaleoutcome}. By the equivariance of the estimators in $\mathcal{T}$, Theorem~\ref{thm:Gamma1restriction} shows that it suffices to consider a collection of priors $\Gamma_1$ with support on $\mathcal{P}_1$. Hence, it suffices to define the collection $\mathcal{P}_1\subset \mathcal{P}$ of distributions $P$ satisfying \eqref{eq:P1}. By \ref{in:shiftrescalepred} and \ref{in:shiftrescaleoutcome}, we see that $\mathcal{P}=\cup_{P\in\mathcal{P}_1} \mathcal{P}(P)$, where $ \mathcal{P}(P)$ consists of the distributions of $(a+b\odot X,\tilde{a} + \tilde{b} Y)$ when $(X,Y)\sim P$; here, $a$, $b$, $\tilde{a}$, and $\tilde{b}$ vary over $\mathbb{R}^p$, $(\mathbb{R}^+)^p$, $\mathbb{R}$, and $\mathbb{R}^+$, respectively. In each setting, the submodel $\mathcal{P}_1$ takes the form
\begin{align*}
    \mathcal{P}_1&:=\left\{P : \mu_P\in \mathcal{R}, P_X\in\mathcal{P}_X, \epsilon_P|X\overset{P}{\sim} N(0,1)\right\},
\end{align*}
and the $p=10$ dimensional features $X$ are known to be drawn from a distribution in the set $\mathcal{P}_X$ of $N(\matr{0}_p,\Sigma)$ distributions, where $\Sigma$ varies over all positive-definite $p\times p$ covariance matrices with diagonal equal to $\matr{1}_p$. The collections $\mathcal{R}$ of regression functions differ in the examples and are detailed in the coming subsections. These collections are indexed by a sparsity parameter $\mathfrak{s}$ that specifies the number of features that may contribute to the regression function $\mu_P$. In each setting, we considered all four combinations of $\mathfrak{s}\in\{1,5\}$ and $n\in\{100,500\}$, where $n$ denotes the number of observations in the data sets $\matr{d}$ used to evaluate the performance of the final learned estimators. For each $n$, we evaluated the performance of AMC meta-trained with data sets of size $n_{mt}=100$ observations (AMC100) and $n_{mt}=500$ observations (AMC500). 

All experiments were run in Pytorch 1.0.1 on Tesla V100 GPUs using Amazon Web Services.  The code used to conduct the experiments can be found at \url{https://github.com/alexluedtke12/amc-meta-learning-of-optimal-prediction-procedures}. Further experimental details can be found in Appendix~\ref{app:numexp}. 

\subsection{Sparse Linear Regression}\label{sec:linreg}
We first considered the setting where $\mu_P$ belongs to a sparse linear model. In particular,
\begin{align}
    \mathcal{R}:=\{x\mapsto \beta^\top x : \|\beta\|_0\le \mathfrak{s},\|\beta\|_1\le 5\}, \label{eq:RSparseLinReg}
\end{align}
where $\|a\|_0:=\#\{j : a_j\not=0\}$ and $\|a\|_1:=\sum_{j=1}^p |a_j|$. The collection $\Gamma$ is described in Appendix~\ref{app:numexp}. 

For each sparsity level $\mathfrak{s}\in\{1,5\}$, we evaluated the performance of the prediction procedure trained at sparsity level $\mathfrak{s}$ using two priors. Both priors sample the covariance matrix of the feature distribution $P_X$ from the Wishart prior $\Pi_X$ described in Appendix~\ref{app:numericalPreliminaries} and let $\beta=(\alpha,0)$ for a random $\alpha$ satisfying $\|\alpha\|_1\le 5$. They differ in how $\alpha$ is drawn. Both make use of a uniform draw $Z$ from $\ell_1$ ball $\{a\in\mathbb{R}^{\mathfrak{s}} : \|a\|_1=5\}$. The first sets $\alpha=Z$, whereas the second sets $\alpha=U Z$ for $U\sim {\rm Unif}(0,1)$ drawn independently of $Z$. We will refer to the two settings as `boundary' and `interior', respectively. We refer to the $\mathfrak{s}=1$ and $\mathfrak{s}=5$ cases as the `sparse' and `dense' settings, respectively. Further details can be found in Appendix~\ref{app:linearSettings}.

In this example, AMC leverages knowledge of the underlying sparse linear regression model by generating synthetic training data from distributions $P$ for which $E_P[Y|X=\cdot\,]$ belongs to the class $\mathcal{R}$ defined in Eq.~\ref{eq:RSparseLinReg} (see line~\ref{ln:P} of Algorithm~\ref{alg:updateBoth}). Therefore, we aimed to compare AMC's performance to that of estimators that also take advantage of this linearity. Ideally, we would compare AMC's performance to that of the true $\Gamma$-minimax estimator. Unfortunately, as is the case in most problems, the form of this estimator is not known in this sparse linear regression setting. 
Therefore, we instead compared AMC's performance to ordinary least squares (OLS) and the lasso \citep{Tibshirani1996} with tuning parameter selected by 10-fold cross-validation, as implemented in \texttt{scikit-learn} \citep{scikit-learn}.

Table~\ref{tab:linear}a displays performance for the sparse setting. We see that AMC outperformed OLS and lasso for the boundary priors, and was outperformed for the interior priors. Surprisingly, AMC500 outperformed AMC100 for the interior prior when $n=100$ observations were used to evaluate performance. The fact that AMC100 was trained specifically for the $n=100$ case suggests that a suboptimal equilibrium may have been reached in this setting. 
Table~\ref{tab:linear}b displays performance for the dense setting. Here AMC always performed at least as well as OLS and lasso when $n_{mt}=n$, and performed comparably even when $n_{mt}\not=n$.

\begin{table}
    \centering
    \begin{subtable}{\linewidth}\centering
    \caption{Sparse signal}
    \begin{tabular}{lllll}
    & \multicolumn{2}{c}{Boundary} & \multicolumn{2}{c}{Interior} \\
     & $n$=100 & 500 & 100 & 500 \\\hline
        OLS & 0.12 & \phantom{$<$}0.02 & 0.12 & 0.02  \\
        Lasso & 0.06 & \phantom{$<$}0.01 & 0.06 & 0.01  \\
        AMC100 (ours) & 0.02 & $<$0.01 & 0.11 & 0.09  \\
        AMC500 (ours) & 0.02 & $<$0.01   & 0.07 & 0.04 \\[0.55em]
    \end{tabular}
    \end{subtable}
    \begin{subtable}{\linewidth}\centering
    \caption{Dense signal}
    \begin{tabular}{lllll}
    & \multicolumn{2}{c}{Boundary} & \multicolumn{2}{c}{Interior} \\
     & $n$=100 & 500 & 100 & 500 \\\hline
        OLS & 0.13 & 0.02 & 0.13 & 0.02  \\
        Lasso & 0.11 & 0.02 & 0.09 & 0.02  \\
        AMC100 (ours) & 0.10 & 0.04 & 0.08 & 0.02  \\
        AMC500 (ours) & 0.09 & 0.02 & 0.09 & 0.02
    \end{tabular}
    \end{subtable}
    \caption{MSEs based on data sets of size $n$ in the linear regression settings. All standard errors are less than 0.001.}
    \label{tab:linear}
\end{table}

\subsection{Fused Lasso Additive Model}
We next considered the setting where $P$ belongs to a variant of the fused lasso additive model (FLAM) \citep{Petersenetal2016}. This model enforces that $\mu_P$ belong to a generalized additive model, that only a certain number of the components can be different from the zero function, and that the sum of the total variations of the remaining components is not too large. We recall that the total variation $V(f)$ of $f : \mathbb{R}\rightarrow\mathbb{R}$ is equal to the supremum of $\sum_{\ell=1}^k |f(a_{\ell+1})-f(a_\ell)|$ over all $(a_\ell)_{\ell=1}^{k+1}$ such that $k\in\mathbb{N}$ and $a_1< a_2 < \ldots < a_{k+1}$ \citep{Cohn2013}. Let $v(\mu):= (V(\mu_j))_{j=1}^p$. Writing $x_j$ to denote feature $j$, the model we considered imposes that $\mu_P$ falls in
\begin{align*}
    \mathcal{R}&:= \left\{x\mapsto \sum_{j=1}^p \mu_j(x_j) : \|v(\mu)\|_1\le 10, \left\|v(\mu)\right\|_0\le \mathfrak{s}\right\}.
\end{align*}
The collection $\Gamma$ is described in Appendix~\ref{app:numexp}. 

In this example, we preprocessed the features before supplying them to the estimator. In particular, we replaced each entry with its rank statistic among the $n$ observations so that, for each $i\in\{1,\ldots,n\}$ and $j\in\{1,\ldots,p\}$, we replaced $\matr{x}_{ij}$ by $\sum_{k=1}^n I\{\matr{x}_{ij}\ge \matr{x}_{kj}\}$ and $x_{0j}$ by $\sum_{k=1}^n I\{x_{0j}\ge \matr{x}_{kj}\}$. This preprocessing step is natural given that the FLAM estimator \citep{Petersenetal2016} also only depends on the features through their ranks. 
An advantage of making this restriction is that, by the homoscedasticity of the errors and the invariance of the rank statistics and total variation to strictly increasing transformations, the learned estimators should perform well even if the feature distributions do not belong to a Gaussian model, but instead belong to a much richer Gaussian copula model.

\begin{figure}
    \centering
    \includegraphics[width=0.6\textwidth]{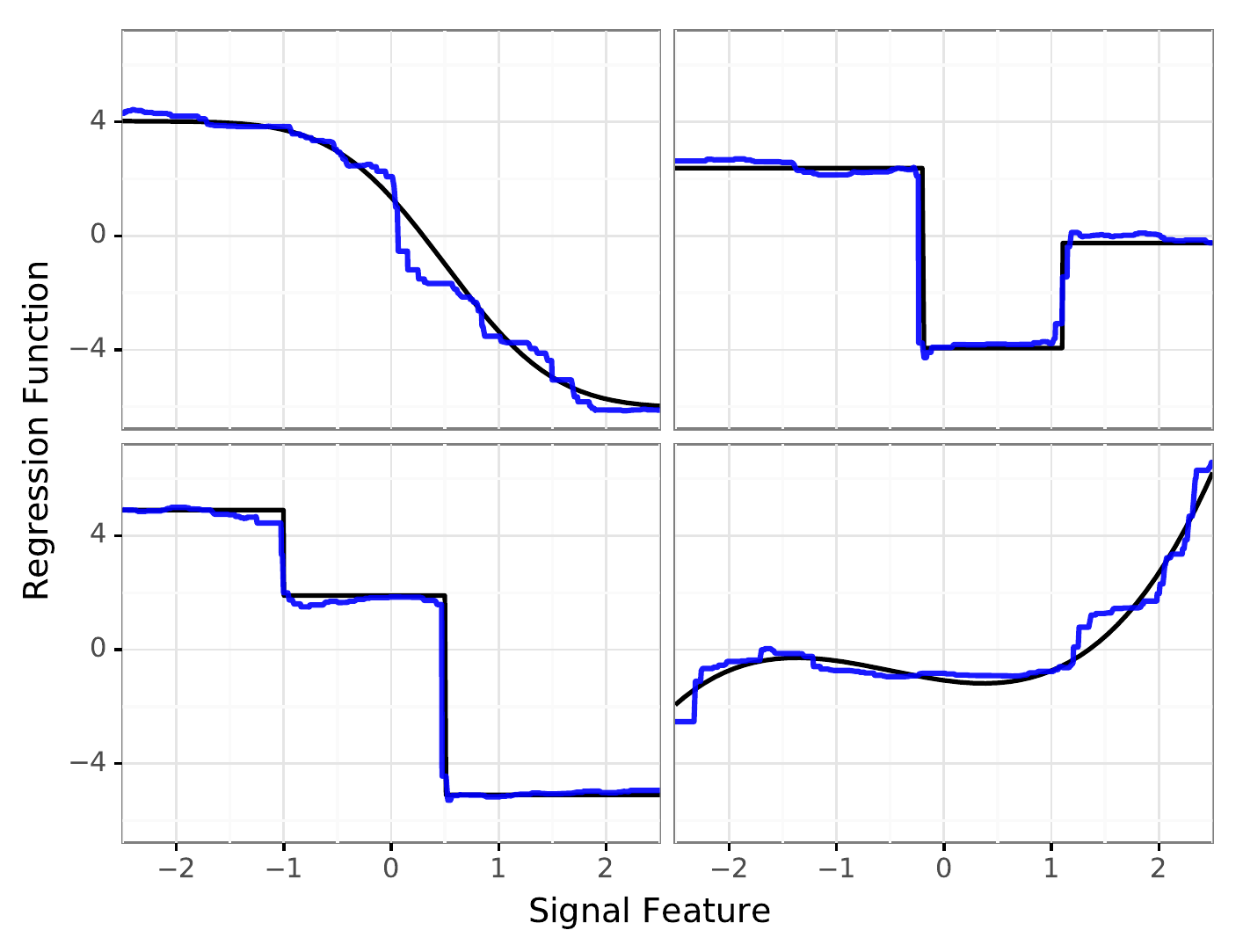}
    \caption{Examples of AMC500 fits (blue) based on $n=500$ observations drawn from distributions at sparsity level $\mathfrak{s}=1$ with four possible signal components (black). Predictions obtained at different signal feature values with all 9 other features set to zero.}
    \label{fig:flam}
\end{figure}

We evaluated the performance of the learned estimators using variants of simulation scenarios 1-4 from \citet{Petersenetal2016}. The level of smoothness varies across the settings (see Fig.~2 in that work). In the variants we considered, the true regression function either contains $\mathfrak{s}_0=1$ (`sparse') or $\mathfrak{s}_0=4$ (`dense') nonzero components. In the sparse setting, we evaluated the performance of the estimators that were meta-trained at sparsity level $\mathfrak{s}=1$, and, in the dense setting, we evaluated the performance of the estimators that were meta-trained at $\mathfrak{s}=5$. Further details can be found in Appendix~\ref{app:flamSettings}.

Similarly to as in the previous example, AMC leverages knowledge of the possible forms of the regression function that is imposed by $\mathcal{R}$ --- in this case, the model for the regression function is nonparametric, but does impose that this function belong to a particular sparse generalized additive model. Though there does not exist a competing estimator that is designed to optimize over $\mathcal{R}$, the FLAM estimator \citep{Petersenetal2016} optimizes over the somewhat larger, non-sparse model where $\mathfrak{s}= p$. We, therefore, compared the performance of AMC to this estimator as a benchmark, with the understanding that AMC is slightly advantaged in that it has knowledge of the underlying sparsity pattern. Nevertheless, we view this experiment as an important proof-of-concept, as it is the first, to our knowledge, to evaluate whether it is feasible to adversarially meta-learn a prediction procedure within a nonparametric regression model.

\begin{table*}
    \centering
    \begin{subtable}{\linewidth}\centering
    \caption{Sparse signal}
    \begin{tabular}{lllllllll}
    & \multicolumn{2}{c}{Scenario 1} & \multicolumn{2}{c}{Scenario 2} & \multicolumn{2}{c}{Scenario 3} & \multicolumn{2}{c}{Scenario 4} \\
     & $n$=100 & 500 & 100 & 500 & 100 & 500 & 100 & 500 \\\hline
        FLAM & 0.44 & 0.12 & 0.47 & 0.17 & 0.38 & 0.11 & 0.51 & 0.19  \\
        AMC100 (ours) & 0.34 & 0.20 & 0.18 & 0.08 & 0.27 & 0.14 & 0.17 & 0.08  \\
        AMC500 (ours) & 0.48 & 0.12 & 0.19 & 0.06 & 0.35 & 0.10 & 0.23 & 0.08 \\[0.5em]
    \end{tabular}
    \end{subtable}
    \begin{subtable}{\linewidth}\centering
    \caption{Dense signal}
    \begin{tabular}{lllllllll}
    & \multicolumn{2}{c}{Scenario 1} & \multicolumn{2}{c}{Scenario 2} & \multicolumn{2}{c}{Scenario 3} & \multicolumn{2}{c}{Scenario 4} \\
     & $n$=100 & 500 & 100 & 500 & 100 & 500 & 100 & 500 \\\hline
        FLAM & 0.59 & 0.17 & 0.65 & 0.24  & 0.53 & 0.16 & 0.76 & 0.36 \\
        AMC100 (ours) & 1.20 & 0.91 & 0.47 & 0.39  & 0.87 & 0.57 & 0.30 & 0.30 \\
        AMC500 (ours) & 0.58 & 0.15 & 0.37 & 0.08  & 0.46 & 0.12 & 0.36 & 0.09
    \end{tabular}
    \end{subtable}
    \caption{MSEs based on data sets of size $n$ in the FLAM settings. Standard errors for FLAM all $<\,$0.04 and for AMC all $<\,$0.01.}
    \label{tab:flamPetersens5}
\end{table*}

To illustrate the kinds of functions that AMC can approximate, Fig.~\ref{fig:flam} displays examples of AMC500 fits from scenario 3 when $(n,\mathfrak{s})=(500,1)$. 
Table~\ref{tab:flamPetersens5} provides a more comprehensive view of the performance of AMC and compares it to that of FLAM. Table~\ref{tab:flamPetersens5}a displays performance for the sparse setting. The AMC procedures meta-trained with $n_{mt}=n$ observations outperformed FLAM for all of these settings. Interestingly, AMC procedures meta-trained with $n_{mt}\not=n$ also outperformed FLAM in a majority of these settings, suggesting that learned procedures can perform well even at different sample sizes from those at which they were meta-trained. In the dense setting (Table~\ref{tab:flamPetersens5}b), AMC500 outperformed both AMC100 and FLAM in all but one setting (scenario 4, $n=100$), and in this setting both AMC100 and AMC500 dramatically outperformed FLAM. The fact that AMC500 also sometimes outperformed AMC100 when $n=100$ in the linear regression setting suggests that there may be some benefit to training a procedure at a larger sample size than that at which it will be evaluated. We leave an investigation of the generality of this phenomenon to future work.

\subsection{Importance of Permutation Invariance}
We numerically evaluated the performance of the architecture in Algorithm~\ref{alg:architecture} when Module 1 is not required to be equivariant to permutations of the observations, that is, when \ref{it:M1equivar} is replaced by the condition that $m_1(v B)_{**\ell}= [m_1(v)_{**\ell}] B$ for all $m_1\in\mathcal{M}_1$, $B\in \mathcal{B}$, $v\in\mathbb{R}^{n\times p\times 2}$, and $\ell = 1,\ldots,o_1$. We repeated the $n=100$ FLAM sparse setting with $n_{mt}=100$ and the same architecture as was used in our earlier experiment, except that each layer in Module 1 was replaced by a multi-input-output channel layer that is equivariant to permutations of the $p$ features \citep{Zaheer2017}, and the output of the final layer was of dimension $\mathbb{R}^{p\times o_1}$ so that the subsequent mean pooling layer could be removed. We saw the following multiplicative increases in the MSE across scenarios 1-4 relative to the performance of AMC100 reported in Table~\ref{tab:flamPetersens5}a: 7-fold, 6-fold, 5-fold, and 4-fold, respectively. These results suggest that \textit{a priori} enforcing that the estimator be invariant to permutations of the observations can dramatically improve performance.

\section{Data Experiments}\label{sec:data}

We further evaluated the performance of the AMC100 estimators learned in our numerical experiments using real datasets. Because there is no reason to assume \textit{a priori} that the true data-generating distribution will belong to any given one of the models that we have used to train our AMC estimators, we also evaluated the performance of stacked ensembles that combine the predictions of these base estimators using 10-fold cross-validation. Following the recommendation of \cite{breiman1996stacked}, we employed a non-negative least squares estimator for this combination step.  We compared the performance of our estimators to the estimators from our numerical experiments that were implemented in Python, namely the OLS and lasso estimators. We also compared to random forest as implemented in \texttt{scikit-learn} \citep{scikit-learn}, with 1,000 trees and otherwise using the default settings.

We built our experiments using three datasets, all available through the University of California, Irvine (UCI) Machine Learning Repository \citep{Dua2019}. The first dataset was originally used to develop quantitative structure-activity relationship (QSAR) models to predict acute aquatic toxicity towards the fathead minnow. This dataset contains 908 total observations, each of which corresponds to a distinct chemical. The outcome is the LC$_{50}$ for that chemical, which represents the concentration of the chemical that is lethal for 50\% of test fish over 96 hours. Six features that describe the molecular characteristics of the chemical are available --- see the UCI Machine Learning Repository and \cite{cassotti2015similarity} for details. The second dataset is from the National Aeronautics and Space Administration (NASA) that contains information on 1,503 airfoils at various wind tunnel speeds and angles of attack \citep{brooks1989airfoil}. The objective is to estimate the scaled sound level in decibels. Five features are available, namely frequency, angle of attack, chord length, free-stream velocity, and suction side displacement thickness. The third dataset contains information on 308 sailing yachts. The objective is to learn to predict a ship's performance in terms of residuary resistance. Six features describing a ship's dimensions and velocity are available, namely: the longitudinal position of the center of buoyancy, the prismatic coefficient, the length-displacement ratio, the beam-draught ratio, the length-beam ratio, and the Froude number. See \cite{gerritsma1981geometry} for more information on these features.

\begin{figure}[tb]
    \centering
    \includegraphics[width=\textwidth]{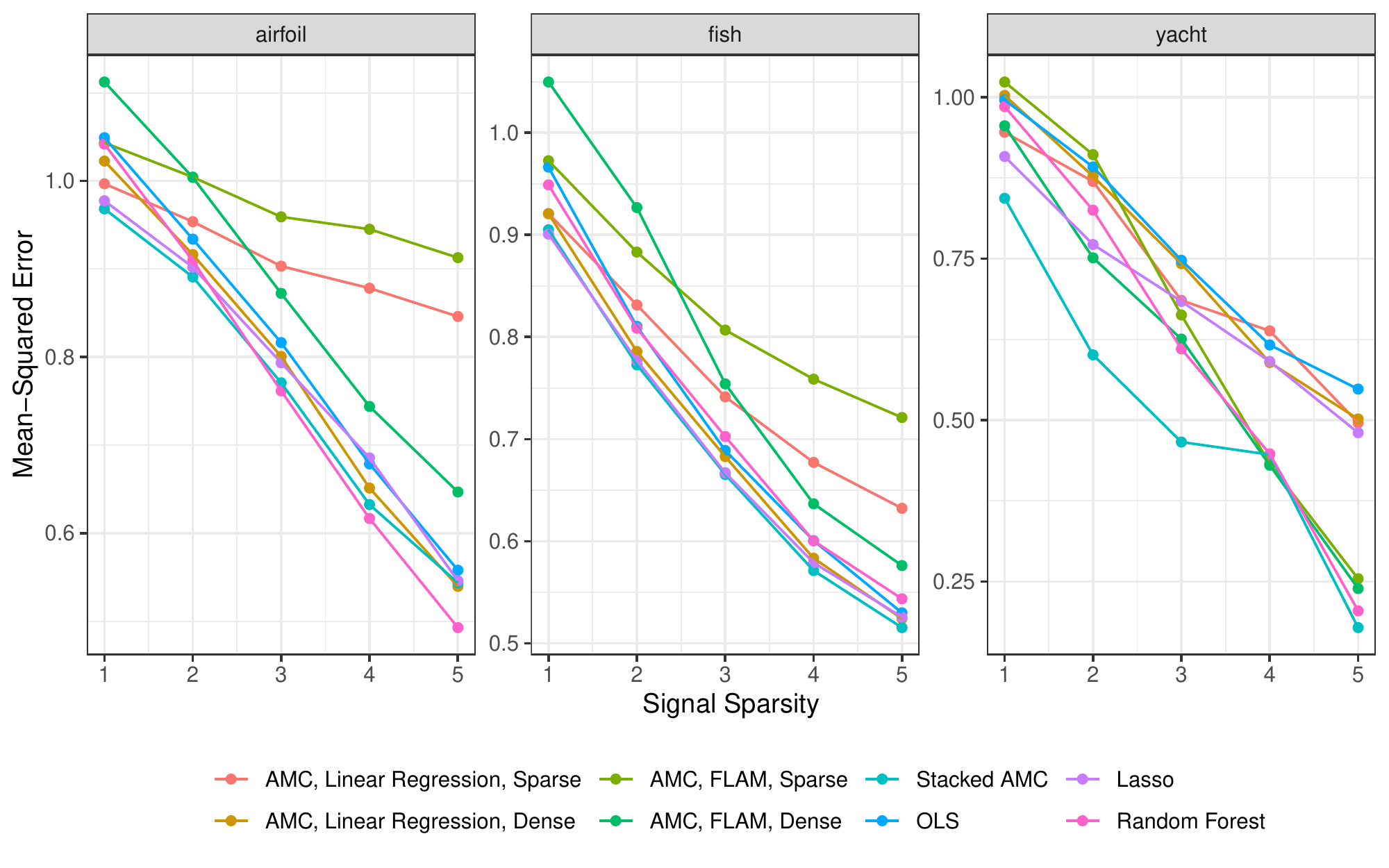}
    \caption{Cross-validated MSEs of the various estimators based on data from three different datasets. For each training-validation split of the data, between 1 and 5 features are selected at random from the original dataset (x-axis), and then all other features are replaced by Gaussian noise. Therefore, the signal is expected to become denser and stronger as the x-axis value increases. Stacked AMC performs well compared to all other approaches across all settings.}
    \label{fig:mses}
\end{figure}

When evaluating the performance of each estimator on these three datasets, we trained them using 100 randomly selected observations and evaluate mean-squared error on the held out observations. We considered varying levels of signal sparsity on these hundred observations. Specifically, for each training-test split of the data, we selected $\mathfrak{s}$ total features from the dataset, remove the remaining features, and then include $(10-\mathfrak{s})$ Gaussian noise features so that the dimension of the feature is always $p=10$. 
For each signal sparsity level $\mathfrak{s}\in\{1,2,3,4,5\}$, we repeat this process 200 times, and report the mean-squared error averaged across the 200 replications.

The performances of the various estimators considered are reported in Fig.~\ref{fig:mses}. The AMC estimator trained within the dense linear regression model tended to outperform OLS. Though there was no strict ordering between OLS and the AMC estimator trained within the sparse linear model, OLS tended to perform better when the signal was denser, as would be expected. Both of these AMC estimators tended to be outperformed by lasso. The AMC estimators trained within the fused lasso additive model tended to perform poorly in the airfoil and fish datasets, but they performed well relative to most other algorithms in the yacht dataset. For dense signals (4-5 non-noise features), random forest performed about as well or better than the best individual AMC algorithm, whereas for sparser signals it tended to perform comparably. The stacked ensemble that combined the predictions of the four AMC estimators performed well across all datasets and sparsity settings considered, and, in particular, always had the lowest or second lowest MSE among all eight estimators. Moreover, it dramatically outperformed all competitors in the yacht dataset when the signal was sparse.

\section{Proofs}\label{sec:proofs}
\subsection{A Study of Group Actions that are Useful for Our Setting}\label{sec:groupHS}

To prove Theorem~\ref{thm:HS}, it will be convenient to use tools from group theory to describe and study the behavior of our estimation problem under the shifts, rescalings, and permutations that we consider. For $k \in \mathbb{N}$, let $\sym(k)$ be the symmetric group on $k$ symbols. Let $\mathbb{R} \rtimes \mathbb{R}^{+}$ be the semidirect product of the real numbers with the positive real numbers with the group multiplication
$$(a_1,b_1) (a_2,b_2) = (a_1 + b_1 a_2, b_1 b_2).$$
Define $\mathcal{G}_0 := (\mathbb{R} \rtimes \mathbb{R}^{+}) \times [(\mathbb{R} \rtimes \mathbb{R}^{+})^{p} \rtimes \sym(p)] \times \sym(n)$. Let $\mathcal{O}_n := \{a \in \mathbb{R}^n : \bar{a} = 0, s(a) = 1\}$. Throughout we equip $\mathcal{G}_0$ with the product topology.

We note that the quantity $\mathcal{Z}$ defined in Section~\ref{sec:optEquivar} writes as
\begin{align}
   \mathcal{Z} = \mathcal{O}_n^p \times \mathcal{O}_n \times \mathbb{R}^p \times \mathbb{R}^p \times \mathbb{R} \times (\mathbb{R}^{+})^p \times \mathbb{R}^{+}. \label{eq:Zequivdef}
\end{align}
Denote the generic group element $g = ((g^{j+}, g^{j\times})_{j=0}^{p},\tau_g,\eta_g)$ where $(g^{j+},g^{j\times}) \in \mathbb{R} \rtimes \mathbb{R}^{+}$, $\tau_g \in \sym(p)$, and $\eta_g \in \sym(n)$. Denote the generic element $\matr{z} \in \mathcal{Z}$ by 
$$\matr{z} = ((z^{x,1,j},\cdots,z^{x,n,j})_{j=1}^p,(z^{y,1},\cdots,z^{y,n}),  (z^{x,0,j})_{j=1}^p, (z^{\bar{x},j})_{j=1}^p, z^{\bar{y}}, (z^{s(x),j})_{j=1}^p, z^{s(y)}) .$$
For $g_1 = ((g_1^{j+},g_1^{j\times})_{j=0}^p,\tau_1,\eta_1)$, $g_2 = ((g_2^{j+},g_2^{j\times})_{j=0}^p,\tau_2,\eta_2)$, two arbitrary elements in $\mathcal{G}_0$, define the group multiplication as
$$g_1 g_2 = \left(g_1^{0+}+g_1^{0\times}g_2^{0+},g_1^{0\times}g_2^{0\times},\big(g_1^{j +}+g_1^{j \times}g_2^{\tau_1^{-1}(j)+},g_1^{j \times}g_2^{\tau_1^{-1}(j)\times}\big)_{j=1}^p,\tau_1\tau_2,\eta_1\eta_2\right).$$
Define the group action $\mathcal{G}_0 \times \mathcal{Z} \to \mathcal{Z}$ by
\begin{align*}
    (g \cdot \matr{z})^{x,i,j} &= z^{x,\eta_g^{-1}(i),\tau_g^{-1}(j)} \\
    (g \cdot \matr{z})^{y,i} &= z^{y,\eta_g^{-1}(i)} \\
    (g \cdot \matr{z})^{x,0,j} &= g^{j+} + g^{j\times} z^{x,0,\tau_g^{-1}(j)} \\
    (g \cdot \matr{z})^{\bar{x},j} &= g^{j+} + g^{j\times} z^{\bar{x},\tau_g^{-1}(j)} \\
    (g \cdot \matr{z})^{\bar{y}} &= g^{0+} + g^{0\times}z^{\bar{y}} \\
    (g \cdot \matr{z})^{s(x),j} &= g^{j\times} z^{s(x),\tau_g^{-1}(j)} \\
    (g \cdot \matr{z})^{s(y)} &= g^{0 \times} z^{s(y)},
\end{align*}
where $i \in \{1,2,\cdots,n\}$ and $j \in \{1,2,\cdots,p\}$.

We make use of the below result without statement in the remainder of this section.
\begin{lemma}\label{lem:leftGroupAction} The map defined above is a left group action.
\end{lemma}
\begin{proof}
The identity axiom, namely that $e\cdot \matr{z}=\matr{z}$ when $e$ is the identity element of $\mathcal{G}_0$, is straightforward to verify and so we omit the arguments. Fix $g_1,g_2\in\mathcal{G}_0$ and $\matr{z}\in\mathcal{Z}$. We establish compatibility by showing that $g_1 g_2\cdot \matr{z}=g_1\cdot (g_2\cdot \matr{z})$. To see that this is indeed the case, note that, for all $i\in\{1,\ldots,n\}$ and $j\in\{1,\ldots,p\}$:
\begin{align*}
(g_1 g_2 \cdot \matr{z})^{y,i} &= z^{y,(\eta_1 \eta_2)^{-1}(i)} = z^{y,\eta_2^{-1}\eta_1^{-1}(i)} = (g_2 \cdot \matr{z})^{y,\eta_1^{-1}(i)} = (g_1 \cdot (g_2 \cdot \matr{z}))^{y,i} \\
(g_1 g_2 \cdot \matr{z})^{x,i,j} &= z^{x,\eta_2^{-1}\eta_1^{-1}(i),\tau_2^{-1}\tau_1^{-1}(j)} = (g_2 \cdot \matr{z})^{x,\eta_1^{-1}(i),\tau_1^{-1}(j)} = (g_1 \cdot (g_2 \cdot \matr{z}))^{x,i,j} \\
(g_1 g_2 \cdot \matr{z})^{\bar{y}} &= g_1^{0+} + g_1^{0\times} g_2^{0+} + g_1^{0\times} g_2^{0\times} z^{\bar{y}} = g_1^{0+} + g_1^{0\times}(g_2 \cdot \matr{z})^{\bar{y}} = (g_1 \cdot (g_2 \cdot \matr{z}))^{\bar{y}}  \\
(g_1 g_2 \cdot \matr{z})^{s(y)} &= g_1^{0\times} g_2^{0\times} z^{s(y)} = g_1^{0\times} (g_2 \cdot \matr{z})^{s(y)} = (g_1 \cdot (g_2 \cdot \matr{z}))^{s(y)} \\
(g_1 g_2 \cdot \matr{z})^{\bar{x},j} &= g_1^{j+} + g_1^{j\times}(g_2^{\tau_1^{-1}(j)+} + g_2^{\tau_1^{-1}(j)\times}z^{\bar{x},\tau_2^{-1} \tau_1^{-1}(j)}) = g_1^{j+} + g_1^{j\times} (g_2 \cdot \matr{z})^{\bar{x},\tau_1^{-1}(j)} \\
&= (g_1 \cdot (g_2 \cdot \matr{z}))^{\bar{x},j}  \\
(g_1 g_2 \cdot \matr{z})^{s(x),j} &= g_1^{j \times} g_2^{j \times} z^{s(x),\tau_2^{-1}\tau_1^{-1}(j)} = g_1^{j\times} (g_2 \cdot \matr{z})^{s(x),\tau_1^{-1}(j)} = (g_1 \cdot (g_2 \cdot \matr{z}))^{s(x),j}.
\end{align*}
The equation for $(g_1 g_2 \cdot \matr{z})^{x,0,j} = (g_1 \cdot (g_2 \cdot \matr{z}))^{x,0,j}$ is analogous to that of $(g_1 g_2 \cdot z)^{\bar{x},j}$ and is therefore omitted.
\end{proof}
We now introduce several group actions that we will make heavy use of in our proof of Theorem~\ref{thm:HS} and in the lemmas that precede it. We first define $\mathcal{G}_0\times \mathcal{S}\rightarrow\mathcal{S}$. For $S \in \mathcal{S}$ and $g \in \mathcal{G}_0$, define $g \cdot S$ to be $(g \cdot S)(\matr{z}) = S(g \cdot \matr{z})$. Conditions \ref{as:SpreservedPerms} and \ref{as:SpreservedShifts} can be restated as $g \cdot S \in \mathcal{S}$ for all $g \in \mathcal{G}_0$ and $S \in \mathcal{S}$. It can then readily be shown that, under these conditions, the defined map is a left group action. For $T\in\mathcal{T}$, we will write $g\cdot T$ to denote the $\mathcal{D}\rightarrow (\mathcal{X}\rightarrow\mathbb{R})$ operator defined so that
\begin{align*}
   (g\cdot T)(\matr{d}) : x_0\mapsto \begin{cases}
    \bar{\matr{y}} + s(\matr{y})(g\cdot S_T)(z(\matr{d},x_0)),&\mbox{ if } (\matr{d},x_0)\in\mathcal{D}_0, \\
    0,&\mbox{ otherwise.}
   \end{cases}
\end{align*}
It is possible that $g\cdot T$ does not belong to $\mathcal{T}$ due to its behavior when $(\matr{d},x_0)\not\in\mathcal{D}_0$, and therefore that the defined map is not a group action. Nonetheless, because $\mathcal{D}_0$ has $P$-probability one for any $P\in\mathcal{P}$, this fact will not pose any difficulties in our arguments.

We now define the group action $\mathcal{G}_0\times (\mathcal{Y}\times\mathcal{X})\rightarrow (\mathcal{Y}\times\mathcal{X})$. For $(y,x) \in \mathbb{R} \times \mathbb{R}^p$, define $g \cdot (y,x)$ as
$$g \cdot (y,x) = (g^{0+} + g^{0\times} y, (g^{i +} + g^{i \times} x^{\tau_g^{-1}(i)})_{i=1}^p). $$
Similar arguments to those used to prove Lemma~\ref{lem:leftGroupAction} show that the map defined above is a left group action. We now define the group action $\mathcal{G}_0\times \mathcal{P}\rightarrow \mathcal{P}$. For $P \in \mathcal{P}$, $g \in \mathcal{G}_0$, define $g \cdot P = P \circ g^{-1}$ by $(g \cdot P)(U) = P(g^{-1}(U))$, where
$$g^{-1}(U) = \{(y,x) \in \mathbb{R}^{p+1} : g \cdot (y,x) \in U\}.$$
Under \ref{in:permpred}, \ref{in:shiftrescalepred}, and \ref{in:shiftrescaleoutcome}, which, as noted in the Section~\ref{sec:optEquivar}, implicitly encode that $P\circ g^{-1}\in \mathcal{P}$, it can readily be shown that the defined map is a left group action. Finally, we define the group action $\mathcal{G}_0\times \Gamma\rightarrow \Gamma$. For $\Pi \in \Gamma$, $g \in \mathcal{G}_0$, define $g \cdot \Pi = \Pi \circ g^{-1}$ by $(g \cdot \Pi)(U) = \Pi(g^{-1}(U))$ where
$$g^{-1}(U) = \{P \in \mathcal{P} : g \cdot P \in U\}.$$
We can restate \ref{in:permpred}, \ref{in:shiftrescalepred}, and \ref{in:shiftrescaleoutcome} as $\Pi \circ g^{-1} \in \Gamma$ for all $\Pi \in \Gamma$, $g \in \mathcal{G}_0$. Under these conditions, it can be shown that the defined map is a left group action.

We now show that $\mathcal{G}_0$ is amenable --- see Appendix~\ref{app:amenability} for a review of this concept. Establishing this fact will allow us to apply Day's fixed point theorem (Theorem~\ref{thm:MK} in Appendix~\ref{app:amenability}) in the upcoming proof of Theorem~\ref{thm:HS}.

\begin{lemma}\label{thm:G0amenable}
$\mathcal{G}_0$ is amenable.
\end{lemma}
\begin{proof}
Because $\sym(p)$ and $\sym(n)$ are finite groups, they are compact, and therefore amenable. Because $\mathbb{R}$ and $\mathbb{R}^+$ are Abelian, they are also amenable. By Theorem \ref{group:ext}, group extensions of amenable groups are amenable.
\end{proof}

\subsection{Proofs of Theorems \ref{thm:HS} through \ref{thm:TMequivar}}

This section is organized as follows. Section~\ref{sec:prelimMain} introduces three general lemmas that will be useful in proving the results from the main text. Section~\ref{sec:HS} proves several lemmas, proves the variant of the Hunt-Stein theorem from the main text (Theorem~\ref{thm:HS}), and concludes with a discussion of the relation of this result to those in \citet{LeCam2012}. Section~\ref{sec:Gamma1proofs} establishes a preliminary lemma and then proves that, when the class of estimators is equivariant, it suffices to restrict attention to priors in $\Gamma_1$ when aiming to learn a $\Gamma$-minimax estimator (Theorem~\ref{thm:Gamma1restriction}). Section~\ref{sec:equilibrium} establishes several lemmas, including a minimax theorem for our setting, before proving the existence of an equilibrium point (Theorem~\ref{thm:equilibrium}). Section~\ref{sec:TM} establishes the equivariance of our proposed neural network architecture (Theorem~\ref{thm:TMequivar}).

In this section, we always equip $C(\mathcal{Z},\mathbb{R})$ with the topology of compact convergence and, whenever \ref{as:holder2} holds so that $\mathcal{S}\subset C(\mathcal{Z},\mathbb{R})$, we equip $\mathcal{S}$ with the subspace topology. For a fixed compact $\mathcal{K}\subset\mathcal{Z}$ and a function $h\in C(\mathcal{Z},\mathbb{R})$, we also let $\|h\|_{\infty,\mathcal{K}}:= \sup_{\matr{z}\in\mathcal{K}}|h(\matr{z})|$.

\subsubsection{Preliminary lemmas}\label{sec:prelimMain}

We now prove three lemmas that will be used in our proofs of Theorems~\ref{thm:HS} and \ref{thm:equilibrium}.

\begin{lemma}\label{lem:Cmetrizable}
$C(\mathcal{Z},\mathbb{R})$ with the compact-open topology is metrizable.
\end{lemma}
\begin{proof}
See Example IV.2.2 in \citet{conway2010course}.
\end{proof}
As a consequence of the above, we can show that a subset of $C(\mathcal{Z},\mathbb{R})$ is closed by showing that it is sequentially closed, and we can show that a subset of $C(\mathcal{Z},\mathbb{R})$ is continuous by showing that it is sequentially continuous.

\begin{lemma}\label{lem:Scompact2}
If \ref{as:ptwiseBdd}, \ref{as:holder2}, and \ref{as:Sclosed2} hold, then $\mathcal{S}$ is a compact subset of $C(\mathcal{Z},\mathbb{R})$.
\end{lemma}
\begin{proof} 
By \ref{as:ptwiseBdd}, $\mathcal{S}$ is pointwise bounded. Moreover, the local H\"{o}lder condition \ref{as:holder2} implies that $\mathcal{S}$ is equicontinuous, in the sense that, for every $\epsilon>0$ and every $\matr{z}\in\mathcal{Z}$ there exists an open neighborhood $\mathcal{U}\subset\mathcal{Z}$ of $\matr{z}$ such that, for all $S\in\mathcal{S}$ and all $\matr{z}'\in\mathcal{U}$, it holds that $|S(\matr{z})-S(\matr{z}')|< \epsilon$. Hence, by the Arzel\`{a}-Ascoli theorem (see Theorem~47.1 in \citealp{Munkres2000} for a convenient version), $\mathcal{S}$ is a relatively compact subset of $C(\mathcal{Z},\mathbb{R})$. By \ref{as:Sclosed2}, $\mathcal{S}$ is closed, and therefore $\mathcal{S}$ is compact.
\end{proof}

We now show that the group action $\mathcal{G}_0\times \mathcal{S}\rightarrow\mathcal{S}$ is continuous under conditions that we assume in Theorem~\ref{thm:HS}. Establishing this continuity condition is necessary for our use of Day's fixed point theorem in the upcoming proof of that result.
\begin{lemma} \label{ContAction}
If \ref{as:holder2}, \ref{as:SpreservedPerms}, and \ref{as:SpreservedShifts} hold, then the group action $\mathcal{G}_0 \times \mathcal{S} \to \mathcal{S}$ is continuous.
\end{lemma}
\begin{proof}
By \ref{as:SpreservedPerms} and \ref{as:SpreservedShifts}, $\mathcal{G}_0 \times \mathcal{S} \to \mathcal{S}$ is indeed a group action. Also, by \ref{as:holder2} and Lemma~\ref{lem:Cmetrizable}, $\mathcal{S}$ is metrizable. Recall the expression for $\mathcal{Z}$ given in \eqref{eq:Zequivdef} and that
\begin{align*}
    \mathcal{G}_0 := (\mathbb{R} \rtimes \mathbb{R}^{+}) \times [(\mathbb{R} \rtimes \mathbb{R}^{+})^p \rtimes \sym(p)] \times \sym(n).
\end{align*}
The product topology is compatible with semidirect products, and so the fact that each multiplicand is a metric space implies that $\mathcal{G}_0$ is a metric space. Hence, it suffices to show sequential continuity. Let $\{(g_k,S_k)\}_{k=1}^\infty$ be a sequence in $\mathcal{G}_0\times \mathcal{S}$ such that $(g_k,S_k) \to (g,S)$, where $(g,S)\in\mathcal{G}_0\times \mathcal{S}$. By the definition of the product metric, $g_k \to g$ and $S_k \to S$. Let $\mathcal{K}_1 \subseteq \mathcal{O}_n^p$, $\mathcal{K}_2 \subseteq \mathcal{O}_n$, $\mathcal{K}_3 \subset \mathbb{R}^p$, $\mathcal{K}_4 \subset \mathbb{R}^p$, $\mathcal{K}_5 \subset \mathbb{R}$, $\mathcal{K}_6 \subset (\mathbb{R}^+)^p$, and $\mathcal{K}_7 \subset \mathbb{R}^{+}$ be compact spaces. 
Since each compact space $\mathcal{K} \subset \mathcal{Z}$ is contained in such a $\prod_{i=1}^7 \mathcal{K}_i$, it suffices to show that 
$$\sup_{\matr{z} \in \prod_{i=1}^7 \mathcal{K}_i} |(g_k \cdot S_k)(\matr{z}) - (g \cdot S)(\matr{z})| = \|g_k \cdot S_k - g \cdot S\|_{\infty, \prod_{i=1}^7 \mathcal{K}_i} \to 0$$ for arbitrary compact sets $\mathcal{K}_1, \cdots, \mathcal{K}_7$. To show this, we will use the decomposition $g_k = (g_{k,1},g_{k,2},g_{k,3},g_{k,4})$, where $g_{k,1} \in \mathbb{R} \rtimes \mathbb{R}^{+}$, $g_{k,2} \in (\mathbb{R} \rtimes \mathbb{R}^{+})^{p}$, $g_{k,3} \in \sym(p)$, and $g_{k,4} \in \sym(n)$. We similarly use the decomposition $g=(g_1,g_2,g_3,g_4)$. For all $N$ large enough, all of the statements are true for all $k > N$: $g_{k,3} = g_3$, $g_{k,4} = g_4$, $g_{k,1}$ is contained in a compact neighbourhood $C_1$ of $g_1$, and $g_{k,2}$ is contained in a compact neighbourhood $C_2$ of $g_2$.

Since permutations are continuous, $g_4 \mathcal{K}_1 g_3:= \{ g_4 w g_3 : w\in\mathcal{K}_1\}$, $g_4 \mathcal{K}_2:= \{ g_4 w : w\in\mathcal{K}_2\}$, and $\mathcal{K}_j g_3:= \{  wg_3  : w\in\mathcal{K}_j\}$, $j=3,4,6$, are compact. In the following we use the decomposition $g':=(g_1',g_2',g_3',g_4')$ for an arbitrary element $g'\in \mathcal{G}$. Since addition and multiplication are continuous, $C_2 \odot (\mathcal{K}_3 g_3) := \{g_2' \cdot w : g_2' \in C_2, w \in \mathcal{K}_3 g_3\}$, $C_2 \odot (\mathcal{K}_4 g_3):=\{g_2' \cdot w : g_2' \in C_2, w \in \mathcal{K}_4 g_3\}$, $C_1 \odot \mathcal{K}_5:=\{g_1' \cdot w : g_1'\in C_1, w \in \mathcal{K}_5\}$, $C_2 \odot (\mathcal{K}_6 g_3):=\{g_2' \cdot w : g_2' \in C_2, w \in \mathcal{K}_6 g_3\}$, and $C_1 \odot \mathcal{K}_7:=\{g_1' \cdot w : g_1' \in C_1, w \in \mathcal{K}_7\}$ are compact. Define $\mathcal{K}^\circ$ to be the compact set
$$\mathcal{K}^\circ = g_4 \mathcal{K}_1 g_3 \times g_4 \mathcal{K}_2 \times C_2 \odot (\mathcal{K}_3 g_3) \times C_2 \odot (\mathcal{K}_4 g_3) \times C_1 \odot \mathcal{K}_5 \times C_2 \odot (\mathcal{K}_6 g_3) \times C_1 \odot \mathcal{K}_7 $$
Then,
$$ \|g_k \cdot S_k - g \cdot S\|_{\infty, \prod_{i=1}^7 \mathcal{K}_i} \leq \|S_k - S\|_{\infty,\mathcal{K}^\circ} \to 0. $$
\end{proof}

\subsubsection{Proof of Theorem~\ref{thm:HS}}\label{sec:HS}

We begin this subsection with four lemmas and then we prove Theorem~\ref{thm:HS}. Following this proof, we briefly describe how the argument relates to that given in \citet{LeCam2012}. 
In the proof of Theorem~\ref{thm:HS}, we will use notation that we established about the group $\mathcal{G}_0$ in Section~\ref{sec:groupHS}. We refer the reader to that section for details.

\begin{lemma} \label{lem:Priskinvar}
For any $g \in \mathcal{G}_0, T \in \mathcal{T}$, and $P \in \mathcal{P}$, $R(g \cdot T,P) = R(T,g\cdot P)$
\end{lemma}
\begin{proof}
Fix $T\in\mathcal{T}$ and $P\in\mathcal{P}$, and let $S:=S_T$, where $S_T$ is defined in (\ref{eq:STdef}). 
By the change-of-variables formula,
\begin{align*}
    R(g \cdot T, P)&= \E_{P}\left[ \int \sigma_{ P}^{-2}\left\{\bar{\matr{Y}} + s(\matr{Y})S(g \cdot \matr{Z}) - \mu_{P}(x_0)\right\}^2 dP_X(x_0) \right]  \\
    &= \E_{P\circ g^{-1}} \left[\int \sigma_{ P}^{-2}\left\{g^{-1}\cdot \bar{\matr{Y}} + s(g^{-1}\cdot \matr{Y})S(\matr{Z}) - \mu_{P}(g^{-1}\cdot x_0)\right\}^2 d(P_X\circ g^{-1})(x_0)\right].
\end{align*}
Plugging the fact that $g^{-1}\cdot \matr{y} = (\matr{y}-g^{0+})/g^{0\times}$ and that
\begin{align*}
    \mu_P(g^{-1}\cdot x_0)&=\E_P[Y|X_0=g^{-1} \cdot x_0]=\E_P[Y|g\cdot X_0=x_0] \\
    &= \frac{\E_P[g\cdot Y|g\cdot X_0=x_0]-g^{0+}}{g^{0\times}} =\frac{\mu_{P\circ g^{-1}}(x_0)-g^{0+}}{g^{0\times}}
\end{align*}
into the right-hand side of the preceding display yields that
\begin{align*}
    &R(g \cdot T, P) \\
    &= \E_{P\circ g^{-1}} \left[ \int \sigma_{ P}^{-2}\left\{\frac{\bar{\matr{Y}}-g^{0+}}{g^{0\times}} + s\left(\frac{\bar{\matr{Y}}-g^{0+}}{g^{0\times}}\right)S(\matr{Z}) - \frac{\mu_{P\circ g^{-1}}(x_0)-g^{0+}}{g^{0\times}}\right\}^2 d(P_X\circ g^{-1})(x_0) \right] \\
    &= \E_{P\circ g^{-1}} \left[ \int \sigma_{ P}^{-2}\left\{\frac{\bar{\matr{Y}}}{g^{0\times}} + s\left(\frac{\bar{\matr{Y}}-g^{0+}}{g^{0\times}}\right)S(\matr{Z}) - \frac{\mu_{P\circ g^{-1}}(x_0)}{g^{0\times}}\right\}^2 d(P_X\circ g^{-1})(x_0) \right].
    \intertext{By the shift and scale properties of the standard deviation and variance, the above continues as}
    &= \E_{P\circ g^{-1}} \left[ \int \sigma_{ P}^{-2}\left\{\frac{\bar{\matr{Y}}}{g^{0\times}} + \frac{s\left(\bar{\matr{Y}}\right)}{g^{0\times}}S(\matr{Z}) - \frac{\mu_{P\circ g^{-1}}(x_0)}{g^{0\times}}\right\}^2 d(P_X\circ g^{-1})(x_0) \right] \\
    &= \E_{P\circ g^{-1}} \left[ \int \sigma_{ P\circ g^{-1}}^{-2}\left\{\bar{\matr{Y}} + s\left(\bar{\matr{Y}}\right) S(\matr{Z}) - \mu_{P\circ g^{-1}}(x_0)\right\}^2 d(P_X\circ g^{-1})(x_0) \right] \\
    &= R(T,g\cdot P).
\end{align*}
\end{proof}

\begin{lemma} \label{Piriskinvar}
For any $g\in\mathcal{G}_0$,  $T\in\mathcal{T}$, and $\Pi\in\Gamma$, it holds that $r(g \cdot T, \Pi)=r(T, g \cdot \Pi)$.
\end{lemma}
\begin{proof}
This result follows quickly from Lemma~\ref{lem:Priskinvar}. Indeed, for any $g\in\mathcal{G}_0$,  $T\in\mathcal{T}$, and $\Pi\in\Gamma$,
\begin{align*}
    r(g\cdot T,\Pi)&= \int R(g\cdot T,P) d\Pi(P) = \int R(T,g\cdot P) d\Pi(P) \\
    &=\int R(T,P) d(\Pi\circ g^{-1})(P) =  r(T, g\cdot \Pi).
\end{align*}
\end{proof}

Let $\mathcal{S}_e:=\{S\in\mathcal{S} : g\cdot S = S\textnormal{ for all }g\in\mathcal{G}_0\}$ consists of the $\mathcal{G}_0$-invariant elements of $\mathcal{S}$. The following fact will be useful when proving Theorem~\ref{thm:HS}, and also when proving results in the upcoming Section~\ref{sec:Gamma1proofs}.
\begin{lemma}\label{lem:SeTe}
It holds that $\mathcal{S}_e=\{S_T : T\in\mathcal{T}_e\}$.
\end{lemma}
\begin{proof}
Fix $S\in\mathcal{S}_e$ and $g\in\mathcal{G}_0$. By the definition of $\mathcal{S}:=\{S_T : T\in \mathcal{T}\}$, there exists a $T\in\mathcal{T}$ such that $S=S_T$. For this $T$, the fact that $S_T(\matr{z})=S_T(g\cdot \matr{z})$ implies that
\begin{align*}
    T(g\cdot \matr{z})&= (g^{0+} + g^{0\times} \bar{\matr{y}}) + g^{0\times} s(\matr{y}) S_T(g \cdot \matr{z}) = (g^{0+} + g^{0\times} \bar{\matr{y}}) + g^{0\times} s(\matr{y}) S_T(\matr{z}) \\
    &= g^{0+} + g^{0\times} [\bar{\matr{y}} + s(\matr{y})S_T(\matr{z})] = g^{0+} + g^{0\times} T(\matr{z}).
\end{align*}
As $g$ was arbitrary, $T\in\mathcal{T}_e$. Hence, $\mathcal{S}_e\subseteq\{S_T : T\in\mathcal{T}_e\}$.

Now fix $T\in\mathcal{T}_e$ and $g\in\mathcal{G}_0$. Note that $S_T(\matr{z})=[T(\matr{z})-\bar{\matr{y}}]/s(\matr{y})$. Using that $T\in\mathcal{T}_e$ implies that $T(g\cdot \matr{z})=g^{0+} + g^{0\times} T(\matr{z})$, we see that
\begin{align*}
    S_T(g\cdot \matr{z})&= \frac{ T(g\cdot\matr{z})- g^{0+} - g^{0\times}\bar{\matr{y}}}{s(g\cdot \matr{y})} = \frac{ T(g\cdot \matr{z})- g^{0+} - g^{0\times}\bar{\matr{y}}}{g^{0\times} s(\matr{y})} \\
    &= \frac{ g^{0+} + g^{0\times} T(\matr{z})- g^{0+} - g^{0\times}\bar{\matr{y}}}{g^{0\times} s(\matr{y})} = \frac{T(\matr{z}) - \bar{\matr{y}}}{s(\matr{y})} = S_T(\matr{z}).
\end{align*}
As, $g$ was arbitrary, $S_T\in\mathcal{S}_e$, and so $\mathcal{S}_e\supseteq\{S_T : T\in\mathcal{T}_e\}$.
\end{proof}

We define $r_0 : \mathcal{S}\times\Gamma\rightarrow[0,\infty)$ as follows:
\begin{align}
    r_0(S,\Pi):= \int \E_{P}\left[\int_{x_0 : (\matr{D},x_0)\in\mathcal{D}_0}  \frac{\{\bar{\matr{Y}} + s(\matr{Y})S(z(\matr{D},x_0))-\mu_P(x_0)\}^2}{\sigma_P^2}dP_X(x_0)\right] d\Pi(P). \label{eq:r0def}
\end{align}
Because $\mathcal{D}_0$ occurs with $P$-probability one (for any $P\in\mathcal{P}$), it holds that $r(T,\Pi)=r_0(S_T,\Pi)$ for any $T\in\mathcal{T}$.

\begin{lemma}\label{lem:lscr0}
Fix $\Pi\in\Gamma$. If \ref{as:ptwiseBdd}, \ref{as:holder2}, and \ref{as:finitesignal2} hold, then $r_0(\cdot,\Pi) : \mathcal{S}\rightarrow\mathbb{R}$ is lower semicontinuous. 
\end{lemma}
\begin{proof}
Fix $\Pi\in\Gamma$. For any compact $\mathcal{K}\subset\mathcal{Z}$, we define $f_{\mathcal{K}} : \mathcal{S}\rightarrow\mathbb{R}$ by
\begin{align*}
    f_{\mathcal{K}}(S):=\int \E_{P}\left[\int_{\mathcal{X}_{\matr{D},\mathcal{K}}} \sigma_P^{-2}\left[\bar{\matr{Y}} + s(\matr{Y})S(\matr{Z})-\mu_P(x_0)\right]^2 dP_X(x_0)\right] d\Pi(P),
\end{align*}
where here and throughout in this proof we let $\matr{Z}:=z(\matr{D},x_0)$ and $\mathcal{X}_{\matr{D},\mathcal{K}}:=\{x_0 : (\matr{D},x_0)\in \mathcal{K}\cap\mathcal{D}_0\}\subseteq\mathcal{X}$. Recalling that there exists an increasing sequence of compact subsets $\mathcal{K}_1 \subset \mathcal{K}_2 \subset \cdots$ such that 
$\bigcup_{j=1}^\infty \mathcal{K}_j = \mathcal{Z}$, we see that $\sup_{j \in \mathbb{N}} f_{\mathcal{K}_j}(\cdot) = r_0(\cdot, \Pi)$ by the monotone convergence theorem.
Moreover, as suprema of collections of continuous functions are lower semicontinuous, we see that $f$ is lower semicontinuous if $f_{\mathcal{K}}$ is continuous for every $\mathcal{K}$. In the remainder of this proof, we will show that this is indeed the case.

By Lemma~\ref{lem:Cmetrizable}, it suffices to show that $f_{\mathcal{K}}$ is sequentially continuous. Fix $S_1,S_2\in \mathcal{S}$. By Jensen's inequality,
\begin{align}
    &\left|f_{\mathcal{K}}(S_1)-f_{\mathcal{K}}(S_2)\right| \nonumber \\
    &= \Bigg|\int \E_{P}\Bigg[\int_{\mathcal{X}_{\matr{D},\mathcal{K}}} \sigma_P^{-2} \Big(\left[\bar{\matr{Y}}+s(\matr{Y})S_1(\matr{Z})-\mu_P(x_0)\right]^2 \nonumber \\
    &\hspace{10.5em}- \left[\bar{\matr{Y}}+s(\matr{Y})S_2(\matr{Z})-\mu_P(x_0)\right]^2\Big)dP_X(x_0)\Bigg] d\Pi(P)\Bigg| \nonumber  \\
    &\le \int \sigma_P^{-2} \E_{P}\Bigg[\int_{\mathcal{X}_{\matr{D},\mathcal{K}}}  \Big|\left[\bar{\matr{Y}}+s(\matr{Y})S_1(\matr{Z})-\mu_P(x_0)\right]^2 \nonumber \\
    &\hspace{10.5em}- \left[\bar{\matr{Y}}+s(\matr{Y})S_2(\matr{Z})-\mu_P(x_0)\right]^2\Big|dP_X(x_0)\Bigg] d\Pi(P). \label{eq:lscDisp1}
\end{align}
In what follows, we will bound the right-hand side above by some finite constant times \newline
$\|S_1-S_2\|_{\mathcal{K},\infty}$. We start by noting that, for any $(\matr{d},x_0)\in\mathcal{K}\cap \mathcal{D}_0$,
\begin{align*}
    \Big|&\left[\bar{\matr{y}}+s(\matr{y})S_1(\matr{z})-\mu_P(x_0)\right]^2 - \left[\bar{\matr{y}}+s(\matr{y})S_2(\matr{z})-\mu_P(x_0)\right]^2\Big| \\
    &=\Big|s(\matr{y}) \left[2\bar{\matr{y}} + s(\matr{y})\{S_1(\matr{z}) + S_2(\matr{z})\} -2\mu_P(x_0)\right]\left[S_1(\matr{z})-S_2(\matr{z})\right]\Big| \\
    &\le \|S_1-S_2\|_{\infty,\mathcal{K}}s(\matr{y})\Big| 2\bar{\matr{y}} + s(\matr{y})\{S_1(\matr{z}) + S_2(\matr{z})\} -2\mu_P(x_0)\Big| \\
    &\le \|S_1-S_2\|_{\infty,\mathcal{K}}\left(s(\matr{y})^2[\|S_1\|_{\mathcal{K},\infty} + \|S_2\|_{\mathcal{K},\infty}]  + 2s(\matr{y})|\bar{\matr{y}} -\mu_P(x_0)|\right) \\
    &\le \|S_1-S_2\|_{\infty,\mathcal{K}}\left(s(\matr{y})^2[\|S_1\|_{\mathcal{K},\infty} + \|S_2\|_{\mathcal{K},\infty}]  + 2s(\matr{y})|\bar{\matr{y}} -\E_P[Y]| + 2s(\matr{y})|\mu_P(x_0)-\E_P[Y]|\right) \\
    &\le 2\|S_1-S_2\|_{\infty,\mathcal{K}}\left(C_1 s(\matr{y})^2  + s(\matr{y})|\bar{\matr{y}} -\E_P[Y]| + s(\matr{y})|\mu_P(x_0)-\E_P[Y]|\right),
\end{align*}
where $C_1:=\sup_{S\in\mathcal{S}}\|S\|_{\mathcal{K},\infty}$ is finite by \ref{as:ptwiseBdd} and \ref{as:holder2}. Integrating both sides shows that
\begin{align}
    \E_{P}&\left[\int_{\mathcal{X}_{\matr{D},\mathcal{K}}}  \left|\left[\bar{\matr{Y}}+s(\matr{Y})S_1(\matr{Z})-\mu_P(x_0)\right]^2 - \left[\bar{\matr{Y}}+s(\matr{Y})S_2(\matr{Z})-\mu_P(x_0)\right]^2\right|dP_X(x_0)\right] \nonumber \\
    &\le 2\|S_1-S_2\|_{\infty,\mathcal{K}}\Bigg(C_1\E_{P}\left[\int_{\mathcal{X}_{\matr{D},\mathcal{K}}} s(\matr{Y})^2 dP_X(x_0)\right] + \E_{P}\left[\int_{\mathcal{X}_{\matr{D},\mathcal{K}}} s(\matr{Y})|\bar{\matr{Y}}-\E_P[Y]| dP_X(x_0)\right] \nonumber \\
    &\hspace{9em}+ \E_{P}\left[\int_{\mathcal{X}_{\matr{D},\mathcal{K}}} s(\matr{Y})|\mu_P(x_0)-\E_P[Y]| dP_X(x_0)\right]\Bigg) \nonumber \\
    &\le 2\|S_1-S_2\|_{\infty,\mathcal{K}}\Bigg(C_1\E_{P}\left[s(\matr{Y})^2\right] + \E_{P}\left[s(\matr{Y})|\bar{\matr{Y}}-\E_P[Y]|\right] \nonumber \\
&\hspace{9em}+ 
     \E_{P}\left[ s(\matr{Y})\int |\mu_P(x_0)-\E_P[Y]|dP_X(x_0)\right]\Bigg). \label{eq:lscDisp2}
\end{align}
We now bound the three expectations on the right-hand side by finite constants that do not depend on $S_1$ or $S_2$. All three bounds make use of the bound on the first expectation, namely $\E_{P}\left[s(\matr{Y})^2\right]=\frac{n-1}{n}{\rm Var}_P(Y)\le \frac{n-1}{n}C_2\sigma_P^2$, where $C_2:=\sup_{P\in\mathcal{P}} {\rm Var}_P(Y)/\sigma_P^2$. We note that \eqref{as:finitesignal2} can be used to show that $C_2<\infty$. Indeed,
\begin{align*}
    \E_P \left[{\rm Var}_P(Y \mid X)\right] =  \E_P \left[{\rm Var}_P (\epsilon_P \mid X)\right] = \E_P [\epsilon_P^2] = \sigma_P^2,
\end{align*}
and so, by the law of total variance and \eqref{as:finitesignal2}, $C_2 = 1 + \sup_{P \in \mathcal{P}} {\rm Var}_P(\mu_P(X))/\sigma_P^2 < \infty$. 
By Cauchy-Schwarz, the second expectation on the right-hand side of \eqref{eq:lscDisp2} bound as
\begin{align*}
    \E_{P}\left[s(\matr{Y})|\matr{Y}-\E_P[Y]|\right]&\le \E_{P}\left[s(\matr{Y})^2\right]^{1/2}\E_{P}\left[\{\matr{Y}-\E_P[Y]\}^2\right]^{1/2}
    = \E_{P}\left[s(\matr{Y})^2\right]^{1/2}\sigma_P \\
    &= \sqrt{\frac{n-1}{n}} \sqrt{C_2} \sigma_P^2,
\end{align*}
and the third expectation bounds as
\begin{align*}
    \E_{P}\left[ s(\matr{Y})|\mu_P(x_0)-\E_P[Y]|\right]&\le \E_{P}\left[ s(\matr{Y})^2\right]^{1/2}\E_{P}\left[ \int \{\mu_P(x_0)-\E_P[Y]\}^2 dP_{X_0}\right]^{1/2} \\
    &\le \E_{P}\left[ s(\matr{Y})^2\right]^{1/2}{\rm Var}_P(Y)^{1/2}\le \sqrt{\frac{n-1}{n}} \sqrt{C_2} \sigma_P {\rm Var}_P(Y)^{1/2} \\
    &\le \sqrt{\frac{n-1}{n}} C_2 \sigma_P^2.
\end{align*}
Plugging these bounds into \eqref{eq:lscDisp2}, we see that
\begin{align*}
    \E_{P}&\left[\int_{\mathcal{X}_{\matr{D},\mathcal{K}}}  \left|\left[\bar{\matr{Y}}+s(\matr{Y})S_1(\matr{Z})-\mu_P(x_0)\right]^2 - \left[\bar{\matr{Y}}+s(\matr{Y})S_2(\matr{Z})-\mu_P(x_0)\right]^2\right|dP_X(x_0)\right] \nonumber \\
    &\le 2\|S_1-S_2\|_{\infty,\mathcal{K}} \sigma_P^2 \sqrt{\frac{n-1}{n}}C_2^{1/2}\Bigg(C_1C_2^{1/2}\sqrt{\frac{n-1}{n}} + C_2^{1/2} + 1 \Bigg).
\end{align*}
Plugging this into \eqref{eq:lscDisp1}, we have shown that
\begin{align*}
    \left|f_{\mathcal{K}}(S_1)-f_{\mathcal{K}}(S_2)\right|&\le 2\|S_1-S_2\|_{\infty,\mathcal{K}} \sqrt{\frac{n-1}{n}}C_2^{1/2} \Bigg(C_1C_2^{1/2} \sqrt{\frac{n-1}{n}} + C_2^{1/2} +  1 \Bigg).
\end{align*}
We now conclude the proof by showing that the above implies that $f_{\mathcal{K}}$ is sequentially continuous at every $S\in\mathcal{S}$, and therefore is sequentially continuous on $\mathcal{S}$. Fix $S$ and a sequence $\{S_j\}$ such that $S_j\rightarrow S$ compactly. This implies that $\|S_j-S\|_{\infty,\mathcal{K}}\rightarrow 0$, and so the above display implies that $f_{\mathcal{K}}(S_j)\rightarrow f_{\mathcal{K}}(S)$, as desired.
\end{proof}

We now prove Theorem~\ref{thm:HS}.
\begin{proof}[Theorem~\ref{thm:HS}]
Fix $T_0\in\mathcal{T}$ and let $S_0:=S_{T_0}\in\mathcal{S}$. Let $\mathcal{K}$ be the set of all elements $S \in \mathcal{S}$ that satisfy
$$\sup_{\Pi \in \Gamma} r_0(S,\Pi) \leq \sup_{\Pi \in \Gamma} r_0(S_0,\Pi).$$
For fixed $\Pi_0 \in \Gamma$, the set of $S\in\mathcal{S}$ that satisfy $r_0(S,\Pi_0) \leq \sup_{\Pi \in \Gamma} r_0(S_0,\Pi)$ is closed due to the lower semicontinuity of the risk function (Lemma~\ref{lem:lscr0}) and contains $S_0$. 
The intersection of such sets is closed and contains $S_0$ so that $\mathcal{K}$ is a nonempty closed subset of the compact Hausdorff set $\mathcal{S}$, implying that $\mathcal{K}$ is compact. 
By the convexity of $x \mapsto \left(\frac{x-a}{b}\right)^2$, the risk function $S \mapsto r_0(S,\Pi)$ is convex. Hence, $\mathcal{K}$ is convex. 
If $S \in \mathcal{K}$, then Lemma~\ref{Piriskinvar} shows that, for any $g \in \mathcal{G}_0$,
$$r_0(g \cdot S, \Pi_0) = r_0(S,g \cdot \Pi_0) \leq \sup_{\Pi \in \Gamma} r_0(S_0,\Pi).$$
Thus, $g \cdot S \in \mathcal{K}$ and $\mathcal{G}_0 \times \mathcal{K} \to \mathcal{K}$ is an affine group action on a nonempty, convex, compact subset of a locally compact topological vector space. Combining this with the fact that $\mathcal{G}_0$ is amenable (Lemma~\ref{thm:G0amenable}) shows that we may apply Day's fixed point theorem (Theorem~\ref{thm:MK}) to see that there exists an $S_e \in \mathcal{S}$ such that, for all $g \in \mathcal{G}_0, g \cdot S_e = S_e$ and $$\sup_{\Pi \in \Gamma}r_0(S_e,\Pi) \leq \sup_{\Pi \in \Gamma}r_0(S_0,\Pi).$$
The conclusion is at hand. By Lemma~\ref{lem:SeTe}, there exists a $T_e\in\mathcal{T}_e$ such that $S_e=S_{T_e}$. Furthermore, as noted below \eqref{eq:r0def}, $r_0(S_{T_e},\Pi)=r(T_e,\Pi)$ and $r_0(S_{T_0},\Pi)=r(T_0,\Pi)$ for all $\Pi\in\Gamma$. Recalling that $S_0:=S_{T_0}$, the above shows that $\sup_{\Pi \in \Gamma}r(T_e,\Pi) \leq \sup_{\Pi \in \Gamma}r(T_0,\Pi)$. As $T_0\in\mathcal{T}$ was arbitrary and $T_e\in\mathcal{T}_e$, we have shown that $\inf_{T_e\in \mathcal{T}_e}\sup_{\Pi \in \Gamma}r(T_e,\Pi) \leq \inf_{T_0\in\mathcal{T}}\sup_{\Pi \in \Gamma}r(T_0,\Pi)$.
\end{proof}

The proof of Theorem~1 is inspired by that of the Hunt-Stein theorem given in \citet{LeCam2012}. Establishing this result in our context required making meaningful modifications to these earlier arguments. Indeed, \citet{LeCam2012} uses transitions, linear maps between L-spaces, to characterize the space of decision procedures. This more complicated machinery makes it possible to broaden the set of procedures under consideration. Indeed, with this characterization, it is possible to describe decision procedures that cannot even be represented as randomized decision procedures via a Markov kernel, but instead come about as limits of such decision procedures. Despite the richness of the space of decision procedures considered, \citeauthor{LeCam2012} is still able to show that this space is compact by using a coarse topology, namely the topology of pointwise convergence. Unfortunately, this topology appears to generally be too coarse for our Bayes risk function $r_0(\cdot,\Pi)$ to be lower semi-continuous, which is a fact that we used at the beginning of our proof of Theorem~\ref{thm:HS}. 
Another disadvantage to this formulation is that it makes it difficult to enforce any natural conditions or structure, such as continuity, on the set of estimators. It is unclear whether it would be possible to implement a numerical strategy optimizing over a class of estimators that lacks such structure. In contrast, we showed that, under appropriate conditions, it is indeed possible to prove a variant of the Hunt-Stein theorem in our setting even once natural structure is imposed on the class of estimators. To show compactness of the space of estimators that we consider, we applied the Arzel\`{a}-Ascoli theorem.

\subsubsection{Proof of Theorem~\ref{thm:Gamma1restriction}}\label{sec:Gamma1proofs}

We provide one additional lemma before proving Theorem~\ref{thm:Gamma1restriction}. The lemma relates to the class $\widetilde{\mathcal{T}}_e$ of estimators in $\mathcal{T}$ that satisfy the equivariance property (\ref{eq:Tequivar2}) but do not necessarily satisfy (\ref{eq:Tequivar1}). Note that $\mathcal{T}_e\subseteq \widetilde{\mathcal{T}}_e\subseteq\mathcal{T}$.
\begin{lemma}\label{lem:Gamma1}
If \ref{in:shiftrescalepred} and \ref{in:shiftrescaleoutcome} hold, then, for all $T\in\widetilde{\mathcal{T}}_e$,
\begin{align*}
    r(T,\Pi)=r(T,\Pi\circ h^{-1})\ \textnormal{ for all $\Pi\in\Gamma$},
\end{align*}
and so $\sup_{\Pi\in\Gamma} r(T,\Pi)= \sup_{\Pi\in\Gamma_1} r(T,\Pi)$.
\end{lemma}
\begin{proof}[Lemma~\ref{lem:Gamma1}]
Let $e$ be the identity element in $\operatorname{Sym}(n) \times \operatorname{Sym}(p)$.
For each $P \in \mathcal{P}$, define $g_P \in \mathcal{G}_0$ to be
\begin{align*}
    g_P := \left( -\frac{E_P[Y]}{\sigma_P}, \frac{1}{\sigma_P}, \left(-\frac{E_P[X_j]}{\sqrt{{\rm Var}_P(X_j)}}\right)_{j=1}^p, 
\left(\frac{1}{\sqrt{{\rm Var}_P(X_j)}}\right)_{j=1}^p,
e \right).
\end{align*}
It holds that
\begin{align*}
    R(T,\Pi \circ h^{-1}) &= \int R(T,P)d(\Pi \circ h^{-1})(P) \\
    &= \int R(T,P \circ g_P^{-1})d\Pi(P) \quad \text{by the definition of } h \\
    &= \int R(g_P \cdot T, P)d\Pi(P) \quad \text{by Lemma \ref{lem:Priskinvar}} \\
    &= \int R(T,P)d\Pi(P) = r(T,\Pi) \quad \text{since } T \in \tilde{\mathcal{T}}_e.
\end{align*}
\end{proof}
We conclude by proving Theorem~\ref{thm:Gamma1restriction}.
\begin{proof}[Theorem~\ref{thm:Gamma1restriction}]
Under the conditions of the theorem, $\widetilde{\mathcal{T}}_e=\mathcal{T}$. Recalling that $\Gamma_1:=\{\Pi\circ h^{-1} : \Pi\in\Gamma\}$, Lemma~\ref{lem:Gamma1} yields that, for any $T\in\mathcal{T}$, $\sup_{\Pi\in\Gamma} r(T,\Pi)=\sup_{\Pi\in\Gamma} r(T,\Pi\circ h^{-1})=\sup_{\Pi\in\Gamma_1} r(T,\Pi)$. Hence, an estimator $T\in\mathcal{T}$ is $\Gamma$-minimax if and only if it is $\Gamma_1$-minimax.
\end{proof}

\subsubsection{Proof of Theorem~\ref{thm:equilibrium}}\label{sec:equilibrium}

In this subsection, we assume (without statement) that all $\Pi\in\Gamma$ are defined on the measurable space $(\mathcal{P},\mathscr{A})$, where $\mathscr{A}$ is such that $\{A\cap \mathcal{P}_1 : A\in\mathscr{A}\}$ equals $\mathscr{B}_1$, where $\mathscr{B}_1$ is the collection of Borel sets on the metric space $(\mathcal{P}_1,\rho)$ described in \ref{as:Mcompact}. Under \ref{in:shiftrescalepred} and \ref{in:shiftrescaleoutcome}, which we also assume without statement throughout this subsection, it then follows that each $\Pi_1\in\Gamma_1$ is defined on the measurable space $(\mathcal{P}_1,\mathscr{B}_1)$, where $\mathscr{B}_1$ is the collection of Borel sets on $(\mathcal{P}_1,\rho)$. Let $\Gamma_0$ denote the collection of all distributions on $(\mathcal{P}_1,\mathscr{B}_1)$. For each $A\in\mathscr{B}_1$, define the $\epsilon$-enlargement of $A$ by $A^\epsilon:=\{P\in\mathcal{P}_1 : \exists P'\in A \textrm{ such that } \rho(P,P')<\epsilon\}$. Further let $\xi$ denote the L\'{e}vy-Prokhorov metric on $\Gamma_0$, namely
\begin{align*}
\xi(\Pi,\Pi'):=\inf\big\{\epsilon>0 :\;&\Pi(A)\le \Pi'(A^\epsilon) + \epsilon\,\textnormal{ and } \Pi'(A)\le \Pi(A^\epsilon) + \epsilon \,\textrm{ for all } A\in\mathscr{B}_1\big\}.
\end{align*}

\begin{lemma}\label{lem:Picompact}
If \ref{as:Mcompact} and \ref{as:Gamma1closed}, then $(\Gamma_1,\xi)$ is a compact metric space.
\end{lemma}
\begin{proof}[Lemma~\ref{lem:Picompact}]
By Prokhorov's theorem (see Theorem~5.2 in \citealp{vanGaans2003} for a convenient version, or see Theorems~1.5.1 and 1.6.8 in \citealp{Billingsley1999}), \ref{as:Mcompact} implies that $\Gamma_1$ is relatively compact in $(\Gamma_0,\xi)$. The fact that $\Gamma_1$ is closed (\ref{as:Gamma1closed}) implies the result.
\end{proof}

We now define $r_1 : \mathcal{S}_e\times \Gamma_1\rightarrow [0,\infty)$, which is the analogue of $r_0 : \mathcal{S}\times \Gamma\rightarrow [0,\infty)$ from Section~\ref{sec:HS}:
\begin{align}
    r_1(S,\Pi):= \int \E_{P}\left[\int_{x_0 : (\matr{D},x_0)\in\mathcal{D}_0}  \{\bar{\matr{Y}} + s(\matr{Y})S(z(\matr{D},x_0))-\mu_P(x_0)\}^2dP_X(x_0)\right] d\Pi(P). \label{eq:r1def}
\end{align}
Note that, because each distribution in $\mathcal{P}$ is continuous, each distribution in $\mathcal{P}_1$ is also continuous. Hence, $\mathcal{D}_0$ occurs with $P$-probability one for all $P\in\mathcal{P}_1$, and so the definition of $r_1$ combined with Lemma~\ref{lem:SeTe} shows that $r(T,\Pi)=r_1(S_T,\Pi)$ for any $T\in\mathcal{T}_e$ and $\Pi\in\Gamma_1$.

\begin{lemma}\label{lem:usc}
If \ref{as:Mcompact}, then, for each $S\in\mathcal{S}_e$, $r_1(S,\cdot)$ is upper semicontinuous on $(\Gamma_1,\xi)$.
\end{lemma}
\begin{proof}[Lemma~\ref{lem:usc}]
Fix $S\in\mathcal{S}_e$, and note that, by Lemma~\ref{lem:SeTe}, there exists a $T\in\mathcal{T}_e$ such that $S=S_T$. Let $\{\Pi_j\}_{j=1}^\infty$ be such that $\Pi_j\overset{k\rightarrow\infty}{\longrightarrow} \Pi$ in $(\Gamma_1,\xi)$ for some $\Pi\in \Gamma_1$. Because $\xi$ metrizes weak convergence \citep[Theorem~1.6.8 in][]{Billingsley1999}, the Portmanteau theorem shows that $\limsup_{k\rightarrow\infty} \E_{\Pi_j}[f(P)]\le \E_{\Pi}[f(P)]$ for every $f : \mathcal{P}_1\rightarrow \mathbb{R}$ that is upper semicontinuous and bounded from above on $(\mathcal{P}_1,\rho)$. By part (iii) of \ref{as:Mcompact}, we can apply this result at $f : P\mapsto R(T,P)$ to see that $\limsup_{k\rightarrow\infty} r(T,\Pi_j)\le r(T,\Pi)$. As $\{\Pi_j\}_{j=1}^\infty$ was arbitrary, $r(T,\cdot)$ is upper semicontinuous on $(\Gamma_1,\xi)$. Because $r(T,\cdot)=r_1(S_T,\cdot)$ and $S=S_T$, we have this shown that $r_1(S,\cdot)$ is upper semicontinuous on $(\Gamma_1,\xi)$.
\end{proof}

\begin{lemma}\label{lem:Secompact}
Under the conditions of Lemma~\ref{lem:Scompact2}, $\mathcal{S}_e$ is a compact subset of $C(\mathcal{Z},\mathbb{R})$.
\end{lemma}
\begin{proof}
By Lemma~\ref{lem:Scompact2}, $\mathcal{S}_e\subset \mathcal{S}$ is relatively compact. Hence, it suffices to show that $\mathcal{S}_e$ is closed. By Lemma~\ref{lem:Cmetrizable}, a subset of $C(\mathcal{Z},\mathbb{R})$ is closed in the topology of compact convergence if it is sequentially closed. Let $\{S_j\}_{j=1}^\infty$ be a sequence on $\mathcal{S}_e$ such that $S_j\rightarrow S$ compactly. Because $\mathcal{S}_e\subset \mathcal{S}$ and $\mathcal{S}$ is closed by \ref{as:Sclosed2}, we see that $S\in\mathcal{S}$. We now wish to show that $S\in\mathcal{S}_e$. Fix $\matr{z}\in\mathcal{Z}$ and $g\in\mathcal{G}_0$. Because the doubleton set $\{\matr{z},g\cdot \matr{z}\}$  is compact, $S_j(\matr{z})\rightarrow S(\matr{z})$ and $S_j(g\cdot \matr{z})\rightarrow S(g\cdot \matr{z})$, and thus $S_j(\matr{z})-S_j(g\cdot \matr{z})\rightarrow S(\matr{z})-S(g\cdot \matr{z})$. Moreover, because $S_j\in\mathcal{S}_e$, $S_j(g\cdot \matr{z})=S_j(\matr{z})$ for all $j$. Hence, $S_j(\matr{z})-S_j(g\cdot\matr{z})\rightarrow 0$. As these two limits must be equal, we see that $S(\matr{z})=S(g\cdot\matr{z})$. Because $\matr{z}\in\mathcal{Z}$ and $g\in\mathcal{G}_0$ were arbitrary, $S\in\mathcal{S}_e$. 
\end{proof}

\begin{lemma}\label{lem:lscr1}
Fix $\Pi\in\Gamma_1$. If \ref{as:ptwiseBdd}, \ref{as:holder2}, and \ref{as:finitesignal2} hold, then $r_1(\cdot,\Pi) : \mathcal{S}_e\rightarrow\mathbb{R}$ is lower semicontinuous. 
\end{lemma}
\begin{proof}
The proof is similar to that of Lemma~\ref{lem:lscr0} and is therefore omitted.
\end{proof}

\begin{lemma}[Minimax theorem]\label{lem:Fan}
Under the conditions of Theorem~\ref{thm:equilibrium},
\begin{align}
\min_{S\in\mathcal{S}_e}\max_{\Pi\in\Gamma_1} r_1(S,\Pi) = \max_{\Pi\in\Gamma_1}\min_{S\in\mathcal{S}_e} r_1(S,\Pi). \label{eq:minmaxequality}
\end{align}
\end{lemma}
\begin{proof}[Lemma~\ref{lem:Fan}]
We will show that the conditions of Theorem~1 in \citet{Fan1953} are satisfied. By Lemma~\ref{lem:Cmetrizable}, $C(\mathcal{Z},\mathbb{R})$ is metrizable by some metric $\rho_0$. By Lemma~\ref{lem:Secompact}, $(\mathcal{S}_e,\rho_0)$ is a compact metric space. Moreover, by Lemma~\ref{lem:Picompact}, $(\Gamma_1,\xi)$ is a compact metric space. As all metric spaces are Hausdorff, $(\mathcal{S}_e,\rho_0)$ and $(\Gamma_1,\xi)$ are Hausdorff. By Lemma~\ref{lem:usc}, for each for each $S\in\mathcal{S}_e$, $r_1(S,\cdot)$ is upper semicontinuous on $(\Gamma_1,\xi)$. By Lemma~\ref{lem:lscr1}, for each $\Pi\in \Gamma_1$, $r_1(\cdot,\Pi)$ is lower semicontinuous on $(\mathcal{S}_e,\rho_0)$. It remains to show that $r_1$ is concavelike on $\Gamma_1$ (called ``concave on'' $\Gamma_1$ by \citeauthor{Fan1953}) and that $r_1$ is convexlike on $\mathcal{S}_e$ (called ``convex on'' $\mathcal{S}_e$ by \citeauthor{Fan1953}). To see that $r_1$ is concavelike on $\Gamma_1$, note that $\Gamma_1$ is convex (\ref{as:Gamma1conv}), and also that, for all $S\in\mathcal{S}_e$, $r_1(S,\cdot)$ is linear, and therefore concave, on $\Gamma_1$. Hence, $r_1$ is concavelike on $\Gamma_1$ \citep[page 409 of][]{Terkelsen1973}. To see that $r_1$ is convexlike on $\mathcal{S}_e$, note that $\mathcal{S}_e$ is convex (\ref{as:Sconv2}), and also that, for all $\Pi\in\Gamma_1$, $r_1(\cdot,\Pi)$ is convex on $\mathcal{S}_e$. Hence, $r_1$ is convexlike on $\mathcal{S}_e$ (ibid.). Thus, by Theorem~1 in \cite{Fan1953}, \eqref{eq:minmaxequality} holds.
\end{proof}

We conclude by proving Theorem~\ref{thm:equilibrium}.
\begin{proof}[Theorem~\ref{thm:equilibrium}]
We follow arguments given on page 93 of \cite{Chang2006} to show that, under the conditions of this theorem, \eqref{eq:minmaxequality} implies that there exists an $S^\star\in \mathcal{S}_e$ and a $\Pi^\star\in\overline{\Gamma}_1$ such that
\begin{align}
   \max_{\Pi\in \Gamma_1} r_1(S^\star,\Pi) = r_1(S^\star,\Pi^\star) = \min_{S\in \mathcal{S}_e} r_1(S,\Pi^\star). \label{eq:Sequilib}
\end{align}
Noting that pointwise maxima of lower semicontinuous functions are themselves lower semicontinuous, Lemma~\ref{lem:lscr1} implies that $\max_{\Pi\in \Gamma_1} r_1(\cdot,\Pi)$ is lower semicontinuous. Because $\mathcal{S}_e$ is compact (Lemma~\ref{lem:Secompact}), there exists an $S^\star\in\mathcal{S}_e$ such that
\begin{align*}
\max_{\Pi\in \Gamma_1} r_1(S^\star,\Pi)=\min_{S\in \mathcal{S}_e} \max_{\Pi\in \Gamma_1} r_1(S,\Pi).
\end{align*}
Similarly, Lemma~\ref{lem:usc} implies that $\min_{S\in\mathcal{S}_e} r_1(S,\cdot)$ is upper semicontinuous on $(\Gamma_1,\xi)$. Because $(\Gamma_1,\xi)$ is compact (Lemma~\ref{lem:Picompact}), there exists a $\Pi^\star\in\Gamma_1$ such that
\begin{align*}
\min_{S\in\mathcal{S}_e} r_1(S,\Pi^\star)=\max_{\Pi\in \Gamma_1} \min_{S\in \mathcal{S}_e} r_1(S,\Pi).
\end{align*}
By Lemma~\ref{lem:Fan}, the above two displays show that $\max_{\Pi\in \Gamma_1} r_1(S^\star,\Pi)= \min_{S\in\mathcal{S}_e} r_1(S,\Pi^\star)$. Combining this result with the elementary fact that $\min_{S\in \mathcal{S}_e} r_1(S,\Pi^\star) \le r_1(S^\star,\Pi^\star)\le \max_{\Pi\in \Gamma_1} r_1(S^\star,\Pi)$ shows that \eqref{eq:Sequilib} holds.

Recall from below \eqref{eq:r1def} that $r_1(S_{T},\Pi)=r(T,\Pi)$ for all $\Pi\in\Gamma_1$ and $T\in\mathcal{T}_e$. Moreover, since $\mathcal{S}_e=\{S_T : T\in\mathcal{T}_e\}$ (Lemma~\ref{lem:SeTe}), there exists a $T^\star\in\mathcal{T}_e$ such that $S=S_{T^\star}$. Combining these observations shows that (i) $\max_{\Pi\in \Gamma_1} r_1(S^\star,\Pi) = \max_{\Pi\in \Gamma_1} r_1(S_{T^\star},\Pi) = \max_{\Pi\in \Gamma_1} r(T^\star,\Pi)$; (ii) $r_1(S^\star,\Pi^\star)=r_1(S_{T^\star},\Pi^\star)=r(T^\star,\Pi^\star)$; and (iii) $\min_{S\in \mathcal{S}_e} r_1(S,\Pi^\star) = \min_{T\in \mathcal{T}_e} r_1(S_T,\Pi^\star) = \min_{T\in \mathcal{T}_e} r(T,\Pi^\star)$. Hence, by \eqref{eq:Sequilib}, $ \max_{\Pi\in\Pi^\star} r(T^\star,\Pi)=r_1(T^\star,\Pi^\star)=\min_{T\in\mathcal{T}_e} r_1(T,\Pi^\star)$. Equivalently, for all $T\in\mathcal{T}_e$ and $\Pi\in\Gamma_1$, $r(T^\star,\Pi)\le r(T^\star,\Pi^\star)\le r(T,\Pi^\star)$.
\end{proof}

\subsubsection{Proof of Theorem~\ref{thm:TMequivar}}\label{sec:TM}

\begin{proof}[Theorem~\ref{thm:TMequivar}]
Fix $T\in\mathcal{M}$, and let $(m_1,m_2,m_3,m_4)\in\prod_{k=1}^4\mathcal{M}_k$ be the corresponding modules. Recall from Algorithm~\ref{alg:architecture} that, for a given $(\matr{d},x_0)$, $x_0^0 := \frac{x_0-\bar{\matr{x}}}{s(\matr{x})}$ and $\matr{d}^0\in\mathbb{R}^{n\times p\times 2}$ is defined so that $\matr{d}_{i*1}^0=\frac{x_i-\bar{\matr{x}}}{s(\matr{x})}$ for all $i=1,\ldots,n$ and $\matr{d}_{*j2}^0=\frac{\matr{y}-\bar{\matr{y}}}{s(\matr{y})}$ for all $j=1,\ldots,p$. Now, for any $(\matr{d},x_0)\in\mathcal{D}_0$,
$$T(\matr{d})(x_0) = \bar{\matr{y}} + s(\matr{y}) m_4\left(\frac{1}{p}\sum_{j=1}^p m_3\left(\left[m_2\left(\frac{1}{n}\sum_{i=1}^n m_1(\matr{d}^0)_{i \ast \ast}\right)\ \middle|\ x_0^0\right]\right)_{j \ast}\right), $$
and so $S_T$ takes the form
$$S_T(z(\matr{d},x_0)) = m_4\left(\frac{1}{p}\sum_{j=1}^p m_3\left(\left[m_2\left(\frac{1}{n}\sum_{i=1}^n m_1(\matr{d}^0)_{i \ast \ast}\right)\ \middle|\ x_0^0\right]\right)_{j \ast}\right).$$
Because $S_T$ does not depend on the last four arguments of $z(\matr{d},x_0)$, we know that $T$ satisfies (\ref{eq:Tequivar2}), that is, is invariant to shifts and rescalings of the features and is equivariant to shifts and rescalings of the outcome. It remains to show permutation invariance , namely (\ref{eq:Tequivar1}). By the permutation invariance of the sample mean and sample standard deviation, it suffices to establish the analogue of this property for $S_T$, namely that $S_T(z(A\matr{d}B,B x_0))=S_T(z(\matr{d},x_0))$ for all $(\matr{d},x_0)\in\mathcal{D}_0$, $A\in\mathcal{A}$, and $B\in\mathcal{B}$. For an array $M$ of size $\mathbb{R}^{n\times p\times o}$, we will write $AMB$ to mean the $\mathbb{R}^{n\times p\times o}$ array for which $(AMB)_{**\ell}=AM_{**\ell}B$ for all $\ell=1,2,\ldots,o$. Note that
\begin{align*}
S_T(z(A\matr{d}B,B x_0)) &= m_4\left(\frac{1}{p}\sum_{j=1}^p m_3\left(\left[m_2\left(\frac{1}{n}\sum_{i=1}^n m_1(A\matr{d}^0B)_{i \ast \ast}\right)\ \middle|\ B^\top x_0^0\right]\right)_{j \ast}\right) \\
&= m_4\left(\frac{1}{p}\sum_{j=1}^p m_3\left(\left[m_2\left(\frac{1}{n}\sum_{i=1}^n (A m_1(\matr{d}^{0}) B)_{i \ast \ast}\right)\ \middle|\ B^\top x_0^0\right]\right)_{j \ast}\right) \tag{by \ref{it:M1equivar}} \\
&= m_4\left(\frac{1}{p}\sum_{j=1}^p m_3\left(\left[m_2\left(B^\top \frac{1}{n}\sum_{i=1}^n (A m_1(\matr{d}^{0}))_{i \ast \ast}\right)\ \middle|\ B^\top x_0^0\right]\right)_{j \ast}\right) \\
&= m_4\left(\frac{1}{p}\sum_{j=1}^p m_3\left(\left[m_2\left(B^\top \frac{1}{n}\sum_{i=1}^n m_1(\matr{d}^{0})_{i \ast \ast}\right)\ \middle|\ B^\top x_0^0\right]\right)_{j \ast}\right) \\
&= m_4\left(\frac{1}{p}\sum_{j=1}^p m_3\left(\left[B^\top m_2\left(\frac{1}{n}\sum_{i=1}^n m_1(\matr{d}^0)_{i \ast \ast}\right)\ \middle|\  B^\top x_0^0\right]\right)_{j \ast}\right) \tag{by \ref{it:M2equivar}} \\
&= m_4\left(\frac{1}{p}\sum_{j=1}^p m_3\left(B^\top\left[m_2\left(\frac{1}{n}\sum_{i=1}^n m_1(\matr{d}^0)_{i \ast \ast}\right)\ \middle|\  x_0^0 \right]\right)_{j \ast}\right) \\
&= m_4\left(\frac{1}{p}\sum_{j=1}^p \left(B^\top m_3\left(\left[m_2\left(\frac{1}{n}\sum_{i=1}^n m_1(\matr{d}^0)_{i \ast \ast}\right)\ \middle|\  x_0^0\right]\right)\right)_{j \ast}\right) \tag{by \ref{it:M3equivar}} \\
&= m_4\left(\frac{1}{p}\sum_{j=1}^p m_3\left(\left[m_2\left(\frac{1}{n}\sum_{i=1}^n m_1(\matr{d}^0)_{i \ast \ast}\right)\ \middle|\  x_0^0\right]\right)_{j \ast}\right) \\
&= S_T(z(\matr{d},x_0)).
\end{align*}
Hence, $T$ satisfies (\ref{eq:Tequivar1}).
\end{proof}

\section{Extensions and Discussion}\label{sec:extensions}

We have focused on a particular set of invariance properties on the collection of priors $\Gamma$, namely \ref{in:permpred}-\ref{in:shiftrescaleoutcome}. Our arguments can be generalized to handle other properties. As a simple example, suppose \ref{in:shiftrescaleoutcome} is strengthened so that $\Gamma$ is invariant to nonzero (rather than only nonnegative) rescalings $\tilde{b}$ of the outcome -- this property is in fact satisfied in all of our experiments. Under this new condition, the results in Section~\ref{sec:characterization} remain valid with the definition of the class of equivariant estimators $\mathcal{T}_e$ defined in \eqref{eq:Tequivar1} and \eqref{eq:Tequivar2} modified so that $\tilde{b}$ may range over $\mathbb{R}\backslash\{0\}$. Moreover, for any $T$, Jensen's inequality shows that the $\Gamma$-maximal risk of the symmetrized estimator that averages $T(\matr{x},\matr{y})(x_0)$ and negative $T(\matr{x},-\matr{y})(x_0)$ is no worse than that of $T$. To assess the practical utility of this observation, we numerically evaluated the performance of symmetrizations of the estimators learned in our experiments. Symmetrizing improved performance across most settings (see Appendix~\ref{app:symmetrization}). We, therefore, recommend carefully characterizing the invariance properties of a given problem when setting out to meta-learn an estimator.

Much of this work has focused on developing and studying a framework for meta-learning a $\Gamma$-minimax estimator for a single, prespecified collection of priors $\Gamma$. In some settings, it may be difficult to \textit{a priori} specify a single such collection that is both small enough so that the $\Gamma$-minimax estimator is not too conservative while also being rich enough so that the priors in this collection actually place mass in a neighborhood of the true data-generating distribution. Two approaches for overcoming this challenge seem to warrant further consideration. The first would be to employ an empirical Bayes approach \citep{efron1972limiting}, wherein a large dataset from a parallel situation can be used to inform about the possible forms that the prior might take; this, in turn, would also inform about the form that the collection $\Gamma$ should take. Recent advances also make it possible to incorporate knowledge about the existence of qualitatively different categories of features when performing empirical Bayes prediction \citep{nabi2020decoupling}. The second approach involves using AMC to approximate $\Gamma$-minimax estimators over various choices of $\Gamma$, and then to use a stacked ensemble to combine the predictions from these various base estimators. 
In our data experiments, we saw that a simple version of this ensemble that combined four base AMC estimators consistently performed at least as well as the best of these base estimators. 
These results are compatible with existing oracle inequalities for stacked ensembles \citep{van2006oracle,van2007super}, which suggest that including additional base learners should not meaningfully harm performance and can improve performance when the new base learners are predictive of the outcome. Together, these theoretical and experimental results suggest that the proposed AMC meta-learning strategy provides a valuable new tool for improving upon existing prediction frameworks even in settings where \textit{a priori} specification of a single collection of priors $\Gamma$ is not possible. 

In this work, we have focused on the case where the problem of interest is a supervised learning problem and the objective is to predict a continuous outcome based on iid data. While the AMC algorithm generalizes naturally to a variety of other sampling schemes and loss functions \citep[see][]{Luedtkeetal2019}, our characterization of the equivariance properties of an optimal estimator was specific to the iid regression setting that we considered. In future work, it would be interesting to characterize these properties in greater generality, including in classification settings and inverse reinforcement learning settings \citep[e.g.,][]{russell1998learning,geng2020deep}.

% \section*{Acknowledgements} The authors thank Devin Didericksen for help in the early stages of this project. The authors gratefully acknowledge the support of Amazon through an AWS Machine Learning Research Award and the NIH under award number DP2-LM013340. The content is solely the responsibility of the authors and does not necessarily represent the official views of the NIH or Amazon.

\appendix 
\section*{\LARGE Appendices}

\renewcommand{\thesection}{\Alph{section}}

\setcounter{equation}{0}
\renewcommand{\theequation}{S\arabic{equation}}
% Ditto for theorems
\setcounter{theorem}{0}
\renewcommand{\thetheorem}{S\arabic{theorem}}
\renewcommand{\thecorollary}{S\arabic{corollary}}
\renewcommand{\thelemma}{S\arabic{lemma}}
\renewcommand{\thetable}{S\arabic{table}}
\renewcommand{\thefigure}{S\arabic{figure}}

\section{Review of amenability}\label{app:amenability}

In this appendix, we review the definition of an amenable group, an important implication of amenability, and also some sufficient conditions for establishing that a group is amenable. This material will prove useful in our proof of Theorem~\ref{thm:HS} (see Section~\ref{sec:HS}). We refer the reader to \citet{Pier1984} for a thorough coverage of amenability.

\begin{definition}[Amenability]
Let $\mathcal{G}$ be a locally compact, Hausdorff group and let $L^\infty(\mathcal{G})$ be the space of Borel measurable functions that are essentially bounded with respect to the Haar measure. A mean on $L^\infty(\mathcal{G})$ is defined as a linear functional $M \in L^\infty(\mathcal{G})^\ast$ such that $M(\lambda) \geq 0$ whenever $\lambda \geq 0$ and  $M(1_{\mathcal{G}})=1$. 
A mean $M$ is said to be left invariant for a group $\mathcal{G}$ if and only if $M(\delta_g \ast \lambda) = M(\lambda)$ for all $\lambda \in L^\infty(\mathcal{G})$, where $(\delta_g \ast \lambda)(h) = \lambda(g^{-1} h)$. The group $\mathcal{G}$ is said to be amenable if and only if there is a left invariant mean on $L^\infty(\mathcal{G})$.
\end{definition}

We now introduce the fixed point property, and subsequently present a result showing its close connection to the definition given above. Throughout this work, we equip all group actions $\mathcal{G}\times\mathcal{W}\rightarrow\mathcal{W}$ with the product topology.

\begin{definition}[Fixed point property]
We say that a locally compact, Hasudorff group $\mathcal{G}$ has the fixed point property if, whenever $\mathcal{G}$ acts affinely on a compact convex set $\mathcal{K}$ in a locally convex topological vector space $E$ with the map $\mathcal{G} \times \mathcal{K} \to \mathcal{K}$ continuous, there is a point in $x_0\in \mathcal{K}$ fixed under the action of $\mathcal{G}$.
\end{definition}

\begin{theorem}[Day's Fixed Point Theorem]\label{thm:MK}
A locally compact, Hausdorff group $\mathcal{G}$ has the fixed point property if and only if $\mathcal{G}$ is amenable.
\end{theorem}
\begin{proof}
See the proof of Theorem~5.4 in \citet{Pier1984}.
\end{proof}

The following results are useful for establishing amenability.
\begin{lemma} Any compact group is amenable.
\end{lemma}
\begin{proof}
Take the normalized Haar measure as an invariant mean.
\end{proof}

\begin{lemma} Any locally compact Abelian group is amenable.
\end{lemma}
\begin{proof}
See the proof of Proposition 12.2 in \citet{Pier1984}.
\end{proof}

\begin{lemma} \label{group:ext}
Let $\mathcal{G}$ be a locally compact group and $\mathcal{N}$ a closed normal subgroup of $\mathcal{G}$. If $\mathcal{N}$ and $\mathcal{G}/\mathcal{N}$ are amenable, then $\mathcal{G}$ is amenable.
\end{lemma}
\begin{proof}
Assume that a continuous affine action of $\mathcal{G}$ on a nonempty compact convex set $\mathcal{K}$ is given. Let $\mathcal{K}^{\mathcal{N}}$ be the set of all fixed points of $\mathcal{N}$ in $\mathcal{K}$. Since $\mathcal{N}$ is amenable, Theorem~\ref{thm:MK} implies that $\mathcal{K}^\mathcal{N}$ is nonempty. Since the group action is continuous, $\mathcal{K}^\mathcal{N}$ is a closed subset of $\mathcal{K}$ and hence is compact. Since the action is affine, $\mathcal{K}^\mathcal{N}$ is convex. Now, note that, for all $x \in \mathcal{K}^\mathcal{N}$, $g \in \mathcal{G}$, and $n \in \mathcal{N}$, the fact that $g^{-1}ng \in \mathcal{N}$ implies that $g^{-1}ngx = x$ which implies $ngx = gx$. Hence, $\mathcal{K}^\mathcal{N}$ is preserved by the action of $\mathcal{G}$. The action of $\mathcal{G}$ on $\mathcal{K}^\mathcal{N}$ factors to an action of $\mathcal{G}/\mathcal{N}$ on $\mathcal{K}^\mathcal{N}$, which has a fixed point $x_0$ since $\mathcal{G}/\mathcal{N}$ is amenable. But then $x_0$ is fixed by each $g \in \mathcal{G}$. Hence, $\mathcal{G}$ is amenable.
\end{proof}

\section{Examples where \ref{as:Mcompact} holds}\label{app:metricDiscussion}

We now describe settings where \ref{as:Mcompact} is often applicable. We will specify $\mathcal{P}_1$ in each of these settings, and the model $\mathcal{P}$ is then defined by expanding $\mathcal{P}_1$ to contain the distributions of all possible shifts and rescalings of a random variate drawn from some $P_1\in\mathcal{P}_1$. The first class of models for which \ref{as:Mcompact} is often satisfied is parametric in nature, with each distribution $P_\theta\in\mathcal{P}_1$ indexed smoothly by a finite dimensional parameter $\theta$ belonging to a subset $\Theta$ of $\mathbb{R}^k$. We note here that, because the sample size $n$ is fixed in our setting, we can obtain an essentialy unrestricted model by allowing $k$ to be large relative to $n$. In parametric settings, $\rho$ can often be defined as $\rho(P_\theta,P_{\theta'})=\|\theta-\theta'\|_2$, where we recall that $\|\cdot\|_2$ denotes the Euclidean norm. If $\Gamma_1$ is uniformly tight, which certainly holds if $\Theta$ is bounded, then \ref{as:Mcompact} holds provided $\theta\mapsto R(T,P_\theta)$ is upper-semicontinuous for all $T\in\mathcal{T}_e$. For a concrete example where the conditions of \ref{as:Mcompact} are satisfied, consider the case that $\Theta=\{\theta : \|\theta\|_0\le \mathfrak{s}_0, \|\theta\|_1\le \mathfrak{s}_1\}$ for sparsity parameters $\mathfrak{s}_0$ and $\mathfrak{s}_1$ on $\|\theta\|_0:=\#\{j : \theta_j\not=0\}$ and $\|\theta\|_1:=\sum_j |\theta_j|$, and $P_\theta$ is the distribution for which $X\sim N(\matr{0}_p,{\rm Id}_p)$, and $Y|X\sim N(\theta^\top X,1)$. This setting is closely related to the sparse linear regression example that we study numerically in Section~\ref{sec:linreg}.

Condition \ref{as:Mcompact} also allows for nonparametric regression functions. Define $\phi^p$ to be the $p$-dimensional standard Gaussian measure.
Define $L^2_0(\phi^p) = \{f \in L^2(\phi^p) \mid \int f(x) d\phi^p(x) = 0\}$.
Let $\mathcal{F} \subset L^2_0(\phi^p)$ satisfy the following conditions:
\begin{enumerate}[(i)]
\item $\mathcal{F}$ is bounded. $\sup_{f \in \mathcal{F}} \|f\|_{L^2(\phi^p)} < \infty$.
\item $\mathcal{F}$ is uniformly equivanishing. $\lim_{N \to \infty} \sup_{f \in \mathcal{F}} \|f 1_{B(0,N)^c} \|_{L^2(\phi^p)} = 0$.
\item $\mathcal{F}$ is uniformly equicontinuous. $\lim_{r \searrow 0} \sup_{f \in \mathcal{F}} \sup_{y \in B(0,r)} \|\tau_y f - f\|_{L^2(\phi^p)} = 0$ where $\tau_y$ is the translation by $y$ operator.
\item $\mathcal{F}$ is closed in $L^2(\phi^p)$.
\item There exists $q' > 2$ such that $\mathcal{F} \subset L^{q'}(\phi^p)$.
\end{enumerate}
By a generalization of the Riesz-Kolmogorov theorem as seen in \cite{guo2019improvement}, $\mathcal{F}$ is compact under assumptions (i) through (iv). Let $c > 0$, $\alpha \in (0,1]$. We suppose that $\mathcal{S} = \mathcal{S}^0$ where $\mathcal{S}^0$ is the set of all functions $S: \mathcal{Z} \to \mathbb{R}$ such that $|S(\matr{z})| \leq F(\matr{z})$, $|S(\matr{z})-S(\matr{z}')| \leq c\|\matr{z} - \matr{z}'\|_2^\alpha$ for all $\matr{z}, \matr{z}' \in \mathcal{Z}$. Assume further that $F$ is bounded, i.e.
\begin{align}
    \sup_{z \in \mathcal{Z}} |F(z)| = B_{\mathcal{S}^0} < \infty, \label{eq:Bdef}
\end{align}
and also that $F$ is constant in the orbits induced by the group action on $\mathcal{Z}$ defined in Section~\ref{sec:groupHS}. 

For each $f \in \mathcal{F}$, let $P_f$ denote the distribution of $X \sim N(0,{\rm Id}_p)$, $Y \mid X \sim N(f(X),1)$. Suppose that $\mathcal{P}_1 = \{P_f \mid f \in \mathcal{F}\}$. With the metric $\rho(f,g) = \|f-g\|_{L^2(\phi^p)}$, $(\mathcal{P}_1,\rho)$ is a complete separable compact metric space.
We also see that $P \mapsto R(T,P)$ is continuous.
\begin{lemma}
For all $T \in \mathcal{T}_e$, $P \mapsto R(T,P)$ is continuous in this example.
\end{lemma}
\begin{proof}
To ease presentation, we introduce some notation. 
For $f \in \mathcal{F}$, let $f(\matr{x}) := (f(x_i))_{i=1}^n$, $\bar{f}(\matr{x}) := \frac{1}{n} \sum_{i=1}^n f(x_i)$, $s_f(\matr{d}) := s(\matr{y}+f(\matr{x}))$, and $\bar{y}(\matr{y}) := \bar{\matr{y}}$. Let $S_{T,f}$ denote the map $(\matr{d},x_0)\mapsto S_T(z_f(\matr{d},x_0))$, where $z_f(\matr{d},x_0)$ takes the same value as $z_f(\matr{d},x_0)$ except that the entry  $\frac{\matr{y}-\bar{\matr{y}}}{s(y)}$ is replaced with $\frac{\matr{y}+f(\matr{x})-\bar{\matr{y}}-\bar{f}(\matr{x})}{s_f}$. Also let $\phi^\star := \phi^{p(n+1)+n}$. For $q\in[1,\infty)$ and a function $f: \mathcal{D}\times\mathcal{X}$, we let $\|f\|_{L^q(\phi^\star)}:= [\int |f(\matr{x},\matr{y},x_0)|^q \phi^\star(d\matr{x},d\matr{y},dx_0)]^{1/q}$. We let $\|f\|_{L^\infty(\phi^\star)}:=\inf\{c\ge 0 : f(\matr{x},\matr{y},x_0)\le c \ \ \phi^\star{\rm -a.s.} \}$. For $f : \mathcal{D}\rightarrow\mathbb{R}$, we write $\|f\|_{L^q(\phi^\star)}$ to mean $\|(\matr{d},x_0)\mapsto f(\matr{d}) \|_{L^q(\phi^\star)}$, and follow a similar convention for functions that only take as input $\matr{x}$, $x_i$, $\matr{y}$, or $x_0$. We will write $\lesssim$ to mean inequality up to a positive multiplicative constant that may only depend on $\mathcal{S}$ or $\mathcal{F}$. 

Fix $\varepsilon\in(0,1)$ and $T\in\mathcal{T}_e$. Now, for any $f \in \mathcal{F}$, a change of variables shows that
\begin{align*}
    R(T,P_f) &= E_{P_f}\left[\int \left[T(\matr{X},\matr{Y})(x_0)-f(x_0)\right]^2 d\phi^p(x_0) \right] \\
    &= \int \left[T(\matr{x},\matr{y})(x_0)-f(x_0)\right]^2 (2\pi)^{-\frac{n}{2}} \exp \left[-\frac{1}{2} \sum_{i=1}^n \{y_i - f(x_{i \cdot})\}^2 \right] \phi^{p(n+1)}(d\matr{x},dx_0) d\matr{y} \\
    &= \int \left[T(\matr{x},\matr{y}+f(\matr{x}))(x_0)-f(x_0)\right]^2 \phi^\star(d\matr{x},dx_0,d\matr{y}) \\
    &= \int \left[\bar{\matr{y}} + s(\matr{y}+f(\matr{x}))S_{T,f}(\matr{d},x_0)+\bar{f}(\matr{x})-f(x_0)\right]^2 \phi^\star(d\matr{x},dx_0,d\matr{y}).
\end{align*}
Hereafter we write $d\phi^\star$ to denote $\phi^\star(d\matr{x},dx_0,d\matr{y})$.

Fix $f,g\in\mathcal{F}$. Most of the remainder of this proof will involve establishing that $R(T,P_f)-R(T,P_g)\lesssim \varepsilon^{-2}\|f-g\|_{L^2(\phi^p)} + \varepsilon$. By symmetry, it will follow that $|R(T,P_f)-R(T,P_g)|\le \varepsilon^{-2}\|f-g\|_{L^2(\phi^p)} + \varepsilon$.

In what follows we will use the notation $(g-f)(x_0)$ to mean $g(x_0)-f(x_0)$, $(\bar{g}-\bar{f})(\matr{x})$ to mean $\bar{g}(\matr{x})-\bar{f}(\matr{x})$, etc. The above yields that
\begin{align}
    R&(T,P_f)-R(T,P_g) \nonumber\\
    &= \int \left[(\bar{f}(\matr{x})-f(x_0))^2-(\bar{g}(\matr{x})-g(x_0))^2 \right] d\phi^\star \label{eq:firstTerm} \\
    &\quad+ 2\int \bar{\matr{y}}\left[ (g-f)(x_0)-(\bar{g}-\bar{f})(\matr{x}) \right] d\phi^\star \label{eq:secondTerm} \\
    &\quad+ 2 \int \bar{\matr{y}}\left[s_f(\matr{d}) S_{T,f}(\matr{d},x_0) - s_g(\matr{d}) S_{T,g}(\matr{d},x_0)\right] d\phi^\star
    \label{eq:thirdTerm} \\
    &\quad+ \int \left[s_f^2(\matr{d}) S_{T,f}(\matr{d},x_0)^2 - s_g^2(\matr{d}) S_{T,g}(\matr{d},x_0)^2\right] d\phi^\star \label{eq:fourthTerm} \\
    &\quad+ 2\int \left[(\bar{f}(\matr{x})-f(x_0))s_f(\matr{d}) S_{T,f}(\matr{d},x_0) - (\bar{g}(\matr{x})-g(x_0))s_g(\matr{d}) S_{T,g}(\matr{d},x_0)\right] d\phi^\star. \label{eq:fifthTerm}
\end{align}
We bound the labeled terms on the right-hand side separately. After some calculations, it can be seen that \eqref{eq:firstTerm} and \eqref{eq:secondTerm} are bounded by a constant multiplied by $\|f-g\|_{L^2(\phi^p)}$. These calculations, which are omitted, involve several applications of the triangle inequality, the Cauchy-Schwarz inequality, and condition (i).

The integral in \eqref{eq:thirdTerm} bounds as follows:
\begin{align}
    \int &\bar{\matr{y}}\left[s_f(\matr{d}) S_{T,f}(\matr{d},x_0) - s_g(\matr{d}) S_{T,g}(\matr{d},x_0)\right] d\phi^\star \nonumber \\
    &= \int \bar{\matr{y}} S_{T,f}(\matr{d},x_0) [s_f(\matr{d})-s_g(\matr{d})] d\phi^\star + \int \bar{\matr{y}}s_g(\matr{d}) [S_{T,f}(\matr{d},x_0)-S_{T,g}(\matr{d},x_0)] d\phi^\star \nonumber \\
    &\le \|\bar{y} S_{T,f} [s_f-s_g]\|_{L^1(\phi^\star)} + \|\bar{y}s_g [S_{T,f}-S_{T,g}]\|_{L^1(\phi^\star)}. \label{eq:thirdExpand}
\end{align}
We start by studying first term of the right-hand side above. Note that, by \eqref{eq:Bdef} and the assumption that $|S(\matr{z})|\le F(\matr{z})$ for all $\matr{z}\in\mathcal{Z}$ and $S\in\mathcal{S}$, we have that $|S_{T,f}(\matr{d},x_0)| \leq B_{\mathcal{S}^0}$. Combining this with Cauchy-Schwarz, the first term on the right-hand side above bounds as
\begin{align}
    \|\bar{y}S_T[s_f-s_g]\|_{L^1(\phi^\star)} \leq B_{\mathcal{S}_0} \|\bar{y}\|_{L^2(\phi^\star)} \|s_f-s_g\|_{L^2(\phi^\star)}. \label{eq:thirdExpandFirstPrelim}
\end{align}
To continue the above bound, we will show that $\|s_f-s_g\|_{L^2(\phi^\star)}\lesssim \|f-g\|_{L^2(\phi^p)}^{1/2}$. Noting that
\begin{align*}
    s^2_f(\matr{d}) - s^2_g(\matr{d})&= \frac{1}{n}\sum_{i=1}^n \Bigg[f(x_i)^2 - g(x_i)^2 + 2(y_i-\bar{\matr{y}})[f(x_i)-g(x_i) + \bar{g}(\matr{x})-\bar{f}(\matr{x})] \\
    &\hspace{5em}+ 2[g(x_i)\bar{g}(\matr{x})-f(x_i)\bar{f}(\matr{x})] + \bar{f}(\matr{x})^2 - \bar{g}(\matr{x})^2\Bigg]
\end{align*}
we see that, by the triangle inequality and the Cauchy-Schwarz inequality,
$$\|s^2_f - s^2_g\|_{L^1(\phi^\star)} \lesssim \|f-g\|_{L^2(\phi^p)}.$$
For $a > 0, b > 0$, $|\sqrt{a} - \sqrt{b}| \leq \sqrt{|a-b|}$, and so $|s_f(\matr{d}) - s_g(\matr{d})| \leq \sqrt{|s_f^2(\matr{d}) - s_g^2(\matr{d})|}$, which implies that $|s_f(\matr{d}) - s_g(\matr{d})|^2 \leq |s_f^2(\matr{d}) - s_g^2(\matr{d})|$, which in turn implies that $\|s_f - s_g\|_{L^2(\phi^\star)}^2 \leq \|s_f^2 - s_g^2\|_{L^1(\phi^\star)}$. Combining this with the above and taking square roots of both sides gives the desired bound, namely
\begin{align}
    \|s_f - s_g\|_{L^2(\phi^\star)} \lesssim \|f-g\|_{L^2(\phi^p)}^{1/2}. \label{eq:sfL2}
\end{align}
Recalling \eqref{eq:thirdExpandFirstPrelim}, we then see that the first term on the right-hand side of \eqref{eq:thirdExpand} satisfies
\begin{align*}
    \|\bar{y}S_{T,f}[s_f-s_g]\|_{L^1(\phi^\star)}\lesssim \|f-g\|_{L^2(\phi^p)}^{1/2}.
\end{align*}
We now study the second term in \eqref{eq:thirdExpand}. Before beginning our analysis, we note that, for all $\matr{d}$,
\begin{align}
1 \leq 1_{\{s_g(\matr{d}) \leq \varepsilon\}} + 1_{\{s_g(\matr{d}) > \varepsilon\} \cap \{|s_g(\matr{d}) - s_f(\matr{d})| < \varepsilon/2\}} + 1_{\{|s_g(\matr{d}) - s_f(\matr{d})| \geq \varepsilon/2\}}. \label{eq:indicators}
\end{align}
Combining the above with the triangle inequality, the second term in \eqref{eq:thirdExpand} bounds as: 
\begin{align}
    \|\bar{y}s_g [S_{T,f}-S_{T,g}]\|_{L^1(\phi^\star)}&\le  \|\bar{y}s_g [S_{T,f}-S_{T,g}]1_{\{s_g \leq \varepsilon\}}\|_{L^1(\phi^\star)} \nonumber \\
    &\quad+  \|\bar{y}s_g [S_{T,f}-S_{T,g}]1_{\{s_g > \varepsilon\} \cap \{|s_f - s_g| < \varepsilon/2\}}\|_{L^1(\phi^\star)} \nonumber \\
    &\quad+  \|\bar{y}s_g [S_{T,f}-S_{T,g}]1_{\{|s_g - s_f| \geq \varepsilon/2\}}\|_{L^1(\phi^\star)}. \label{eq:eventStudy}
\end{align}
In the above normed quantities, expressions like $1_{\{s_g \leq \varepsilon\}}$ should be interpreted as functions, e.g. $1_{\{s_g(\cdot) \leq \varepsilon\}}$. By \eqref{eq:Bdef}, the first term on the right-hand side bounds as
\[
\| \bar{y} s_g [S_{T,f}-S_{T,g}] 1_{s_g \leq \varepsilon} \|_{L^1(\phi^\star)} \lesssim \varepsilon.
\]
For the second term, we start by noting that 
\begin{align*}
&\|z_f(\matr{d}) - z_g(\matr{d})\|_2 \\
&= \left\|\frac{(s_g - s_f)(\matr{d})}{s_g(\matr{d}) s_f(\matr{d})}(\matr{y} - \bar{\matr{y}}) + \frac{1}{s_f(\matr{d}) s_g(\matr{d})}[s_f(\matr{d})(f-g+\bar{g}-\bar{f})(\matr{x})+(s_g-s_f)(\matr{d})(f-\bar{f})(\matr{x})]\right\|_2.    
\end{align*}
Using that $(a+b+c)^\kappa \leq a^\kappa + b^\kappa + c^\kappa$ whenever $a,b,c > 0$ and $\kappa \in (0,1]$, this then implies that
\begin{align*}
\|z_f(\matr{d}) - z_g(\matr{d})\|_2^\alpha &\leq \left\|\frac{(s_g - s_f)(\matr{d})}{s_g(\matr{d}) s_f(\matr{d})}(\matr{y} - \bar{\matr{y}})\right\|_2^\alpha + \left\|\frac{(f-g+\bar{g}-\bar{f})(\matr{x})}{s_g(\matr{d}) }\right\|_2^\alpha \\
&\quad+\left\|\frac{(s_g-s_f)(\matr{d})(f-\bar{f})(\matr{x})}{s_f(\matr{d}) s_g(\matr{d})}\right\|_2^\alpha,
\end{align*}
where above $\alpha$ is the exponent from the H\"{o}lder condition satisfied by $\mathcal{S}^0$. Combining the H\"{o}lder condition with the above, we then see that
\begin{align*}
    \left|S_{T,f}(\matr{d},x_0)-S_{T,g}(\matr{d},x_0)\right|&\lesssim \left\|\frac{(s_g - s_f)(\matr{d})}{s_g(\matr{d}) s_f(\matr{d})}(\matr{y} - \bar{\matr{y}})\right\|_2^\alpha + \left\|\frac{(f-g+\bar{g}-\bar{f})(\matr{x})}{s_g(\matr{d})}\right\|_2^\alpha \\
    &\quad+\left\|\frac{(s_g-s_f)(\matr{d})(f-\bar{f})(\matr{x})}{s_f(\matr{d}) s_g(\matr{d})}\right\|_2^\alpha.
\end{align*}
Multiplying both sides by $|\bar{\matr{y}} s_g(\matr{d}) 1_{\{s_g(\matr{d}) > \varepsilon, |(s_f-s_g)(\matr{d})| < \varepsilon/2\}}|$, we then see that
\begin{align*}
    &\left|\bar{y} s_g(\matr{d})  [S_{T,f}(\matr{d},x_0)-S_{T,g}(\matr{d},x_0)] 1_{\{s_g(\matr{d}) > \varepsilon, |(s_f-s_g)(\matr{d})| < \varepsilon/2\}}\right| \\
    &\lesssim |\bar{\matr{y}}| s_g(\matr{d})\left\| \frac{(s_g - s_f)(\matr{d})}{s_g(\matr{d}) s_f(\matr{d})}(\matr{y} - \bar{\matr{y}})\right\|_2^\alpha 1_{\{s_g(\matr{d}) > \varepsilon, |(s_f-s_g)(\matr{d})| < \varepsilon/2\}} \\
    &\quad+ |\bar{\matr{y}}| s_g(\matr{d}) \left\| \frac{(f-g+\bar{g}-\bar{f})(\matr{x})}{s_g(\matr{d})} \right\|_2^\alpha 1_{\{s_g(\matr{d}) > \varepsilon, |(s_f-s_g)(\matr{d})| < \varepsilon/2\}} \\
    &\quad+|\bar{\matr{y}}| s_g(\matr{d})  \left\| \frac{(s_g-s_f)(\matr{d})(f-\bar{f})(\matr{x})}{s_f(\matr{d}) s_g(\matr{d})} \right\|_2^\alpha 1_{\{s_g(\matr{d}) > \varepsilon, |(s_f-s_g)(\matr{d})| < \varepsilon/2\}} \\
    &\lesssim  \varepsilon^{-\alpha}|\bar{\matr{y}}| s_g(\matr{d})^{1-\alpha}\left\| \matr{y} - \bar{\matr{y}}\right\|_2^\alpha |(s_g - s_f)(\matr{d})|^{\alpha}  \\
    &\quad+ |\bar{\matr{y}}| s_g(\matr{d})^{1-\alpha} \left\| (f-g+\bar{g}-\bar{f})(\matr{x}) \right\|_2^\alpha \\
    &\quad+ \varepsilon^{-\alpha}|\bar{\matr{y}}| s_g^{1-\alpha}  \left\| (f-\bar{f})(\matr{x}) \right\|_2^\alpha |(s_g-s_f)(\matr{d})|^\alpha .
\end{align*}
The inequality above remains true if we integrate both sides against $\phi^\star$. The resulting three terms on the right-hand side can be bounded using  H\"{o}lder's inequality. In particular, we have that
\begin{align*}
   &\varepsilon^{-\alpha} \Big\| |\bar{y}|^\alpha \|y-\bar{y}\|_2^\alpha |s_g - s_f|^\alpha |\bar{y}|^{1-\alpha} s_g^{1-\alpha} \Big\|_{L^1(\phi^\star)} \leq \varepsilon^{-\alpha} \Big\|\bar{y} \|y-\bar{y}\|_2 (s_g - s_f) \Big\|_{L^1(\phi^\star)}^\alpha \|\bar{y} s_g \Big\|_{L^1(\phi^\star)}^{1-\alpha} \\
    &\hspace{4em}\lesssim \varepsilon^{-\alpha} \|f-g\|_{L^2(\phi^p)}^{\alpha/2}, \\
   &\Big\| \bar{y} s_g^{1-\alpha} \|(f-g+\bar{g}-\bar{f})(x)\|_2^\alpha \Big\|_{L^1(\phi^\star)} \leq \|\bar{y} s_g\|_{L^1(\phi^\star)}^{1-\alpha} \Big\|\bar{y} \|(f-g+\bar{g}-\bar{f})(x)\|_2 \Big\|_{L^1(\phi^\star)}^\alpha  \\
   &\hspace{4em}\lesssim  \|f-g\|_{L^2(\phi^p)}^{\alpha/2}, \\
   &\varepsilon^{-\alpha} \Big\|\bar{y}s_g^{1-\alpha} \|(f-\bar{f})(x)\|_2^\alpha \|s_g - s_f\|^\alpha \Big\|_{L^1(\phi^\star)} \leq \varepsilon^{-\alpha} \|\bar{y}s_g\|_{L^1(\phi^\star)}^{1-\alpha} \Big\| \|(f-\bar{f})(x)\|_2 |s_g - s_f| \Big\|_{L^1(\phi^\star)}^{\alpha} \\
   &\hspace{4em}\lesssim \varepsilon^{-\alpha} \|f-g\|_{L^2(\phi^p)}^{\alpha/2}.
\end{align*}
Hence, we have shown that the second term on the right-hand side of \eqref{eq:eventStudy} satisfies
\[
\Big\|\bar{y}s_g[S_{T,f}-S_{T,g}]1_{s_g > \varepsilon, |s_g - s_f| < \varepsilon/2}\Big\|_{L^1(\phi^\star)} \lesssim \varepsilon^{-\alpha} \|f-g\|_{L^2(\phi^p)}^{\alpha/2}.
\]
We now study the third term on the right-hand side of \eqref{eq:eventStudy}. We start by noting that, by Markov's inequality and \eqref{eq:sfL2},
\begin{align*}
    P_{\phi^\star}\left(|s_g(\matr{D}) - s_f(\matr{D})| \geq \frac{\varepsilon}{2}\right)&= P\left(|s_g(\matr{D}) - s_f(\matr{D})|^2 \geq \frac{\varepsilon^2}{4}\right) \\
    &\quad\leq \frac{4}{\varepsilon^2}\|s_f-s_g\|_{L^2(\phi^\star)}^2 \lesssim \varepsilon^{-2} \|f-g\|_{L^2(\phi^p)}.
\end{align*}
Moreover, by the generalized H\"{o}lder's inequality with parameters $(4,2,\infty,4)$, we see that
\begin{align*}
    &\left\| \bar{y} s_g [S_{T,f}-S_{T,g}]1_{\{|s_g-s_f|\geq \varepsilon/2\}}\right\|_{L^1(\phi^\star)} \\
    &\le \left\| \bar{y} \right\|_{L^4(\phi^\star)} \left\| s_g \right\|_{L^2(\phi^\star)} \left\|S_{T,f}-S_{T,g}\right\|_{L^\infty(\phi^\star)} \left\|1_{\{|s_g-s_f|\geq \varepsilon/2\}}\right\|_{L^4(\phi^\star)}   \\
    &\leq 2 \|\bar{y}\|_{L^4(\phi^\star)} \|s_g\|_{L^2(\phi^\star)}  B_{\mathcal{S}_0} P(|s_g - s_f| \geq \varepsilon/2)^{1/4} \\
    &\lesssim \varepsilon^{-1/2} \|f-g\|_{L^2(\phi^p)}^{1/4}.
\end{align*}
Combining our bounds for the three terms on the right-hand side of \eqref{eq:eventStudy}, we have shown that
\begin{align}
\| \bar{y}s_g[S_{T,f}-S_{T,g}] \|_{L^1(\phi^\star)} \lesssim \varepsilon + \varepsilon^{-\alpha} \|f-g\|_{L^2(\phi^p)}^{\alpha/2} + \varepsilon^{-1/2}\|f-g\|_{L^2(\phi^p)}^{1/4}. \label{eq:t3bd}
\end{align}
The above provides our bound for the \eqref{eq:thirdTerm} term from the main expression.

We now study the \eqref{eq:fourthTerm} term from the main expression. We start by decomposing this term as
\begin{equation*}
    \int [s_f^2 S_{T,f}^2 - s_g^2 S_{T,g}^2]d\phi^\star =  \int S_{T,f}^2(s_f^2-s_g^2) d\phi^\star + \int s_g^2 [S_{T,f}^2-S_{T,g}^2] d\phi^\star,
\end{equation*}
where for brevity, we have suppressed the dependence on $s_f$, $s_g$, $S_{T,f}$, and $S_{T,g}$ on their arguments. By \eqref{eq:sfL2}, the first term is bounded by a constant times $\|f-g\|_{L^2(\phi^p)}$. For the second term, we note that the uniform bound on $S_{T,f}$ and $S_{T,g}$ shows that
$$\|s_g^2[S_{T,f}^2-S_{T,g}^2]\|_{L^1(\phi^\star)} \lesssim \|s_g^2[S_{T,f}-S_{T,g}]\|_{L^1(\phi^\star)} $$
Similarly to as we did when studying \eqref{eq:thirdTerm}, we can use \eqref{eq:indicators} and the triangle inequality to write
\begin{align*}
    \|s_g^2[S_{T,f}-S_{T,g}]\|_{L^1(\phi^\star)}&\le \left\|s_g^2[S_{T,f}-S_{T,g}]1_{\{s_g\le \varepsilon\}}\right\|_{L^1(\phi^\star)} \\
    &\quad + \left\|s_g^2[S_{T,f}-S_{T,g}]1_{\{s_g>\varepsilon,|s_f-s_g|<\varepsilon/2\}}\right\|_{L^1(\phi^\star)} \\
    &\quad+ \left\|s_g^2[S_{T,f}-S_{T,g}]1_{\{|s_g-s_f|\ge \varepsilon/2\}}\right\|_{L^1(\phi^\star)}.
\end{align*}
The first term on the right upper bounds by a constant times $\varepsilon^{2}$. 
The analyses of the second and third terms are similar to the analysis of the analogous terms from \eqref{eq:thirdTerm}. A minor difference between the study of these terms and that of \eqref{eq:thirdTerm} is that, when applying H\"{o}lder's inequality to separate the terms in each normed expression, we use (v) to ensure that $\|s_g\|_{L^{q'}(\phi^\star)}<\infty$ for some $q'>2$. This helps us deal with the fact that $s_g^2$, rather than $s_g$, appears in the normed expressions above. Due to the similarity of the arguments to those given for \eqref{eq:thirdTerm}, the calculations for controlling the second and third terms are omitted. 
After the relevant calculations, we end up showing that, like \eqref{eq:thirdTerm}, \eqref{eq:fourthTerm} is bounded by a constant times the right-hand side of \eqref{eq:t3bd}.

To study \eqref{eq:fifthTerm} from the main expression, we rewrite the integral as
\begin{align*}
    \int &\left[(\bar{f}(\matr{x})-f(x_0))s_f(\matr{d}) S_{T,f}(\matr{d},x_0) - (\bar{g}(\matr{x})-g(x_0))s_g(\matr{d}) S_{T,g}(\matr{d},x_0)\right] d\phi^\star \\
    &= \int s_f(\matr{d})S_{T,f}(\matr{d},x_0)[\bar{f}(\matr{x})-\bar{g}(\matr{x})+f(x_0)-g(x_0)] d\phi^\star \\
    &\quad + \int  S_{T,f}(\matr{d},x_0)(\bar{g}(\matr{x})+g(x_0))(s_f-s_g)(\matr{d}) d\phi^\star \\
    &\quad + \int s_g(\matr{d}) (\bar{g}(\matr{x})+g(x_0)) [S_{T,f}(\matr{d},x_0)-S_{T,g}(\matr{d},x_0)] d\phi^\star.
\end{align*}
Each of the terms in the expansion can be bounded using similar techniques to those used earlier in this proof. Combining our bounds on \eqref{eq:firstTerm} through \eqref{eq:fifthTerm}, we see that
$$|R(T,P_f)-R(T,P_g)| \lesssim \varepsilon^{-2}\|f-g\|_{L^2(\phi^p)} + \varepsilon.$$
As $f,g$ were arbitrary, we see that, for any sequence $\{f_k\}$ in $\mathcal{F}$ such that $f_k\rightarrow f$ in $L^2(\phi^p)$ as $k\rightarrow\infty$, it holds that $\limsup_k |R(T,P_{f_k})-R(T,P_f)|\lesssim \varepsilon$. As $\varepsilon\in(0,1)$ was arbitrary, this shows that $R(T,P_{f_k})\rightarrow R(T,P_f)$ as $k\rightarrow\infty$. Hence, $P\mapsto R(T,P)$ is continuous in this example.
\end{proof}

\section{Further details on numerical experiment settings}\label{app:numexp}

\subsection{Preliminaries}\label{app:numericalPreliminaries}
We now introduce notation that will be useful for defining $\Gamma_1$ in the two examples. In both examples, all priors in $\Gamma_1$ imply the same prior $\Pi_X$ over the distribution $P_X$ of the features. This prior $\Pi_X$ imposes that the $\Sigma$ indexing $P_X$ is equal in distribution to ${\rm diag}(W^{-1})^{-1/2} W^{-1} {\rm diag}(W^{-1})^{-1/2}$, where $W$ is a $p\times p$ matrix drawn from a Wishart distribution with scale matrix $2\,{\rm Id}_p$ and $20$ degrees of freedom, and $ {\rm diag}(W^{-1})$ denotes a matrix with the same diagonal as $W^{-1}$ and zero in all other entries. The expression for $\Sigma$ normalizes by ${\rm diag}(W^{-1})^{-1/2}$ to ensure that the diagonal of $\Sigma$ is equal to $\matr{1}_p$, which we require of distributions in $\mathcal{P}_X$. We let $\Gamma_\mu$ be a collection of Markov kernels $\kappa : \mathcal{P}_X\rightarrow\mathcal{R}$, so that, for each $\kappa$ and $P_X\in\mathcal{P}_X$, $\kappa(\cdot,P_X)$ is a distribution on $\mathcal{R}$. The collections $\Gamma_\mu$ differ in the two examples, and will be presented in the coming subsections. Let ${\rm Unif}(\mathcal{B})$ denote a uniform distribution over the permutations in $\mathcal{B}$. For each $\kappa\in\Gamma_\mu$, we let $\Pi_{\kappa}$ represent a prior on $\mathcal{P}_1$ from which a draw $P$ can be generated by sampling $P_X\sim \Pi_X$, $\mu|P_X\sim \kappa(\cdot,P_X)$, and $B|P_X,\mu\sim {\rm Unif}(\mathcal{B})$, and subsequently returning the distribution of $(X,\mu(BX)+\epsilon_P)$, where $X\sim P_X$ and $\epsilon_P\sim N(0,1)$ are independent. We let $\Gamma_1:=\{\Pi_{\kappa} : \kappa\in\Gamma_\mu\}$. For a general class of estimators $\mathcal{T}$, enforcing that each draw $P$ has a regression function $\mu_P$ of the form $x\mapsto \mu(Bx)$ for some permutation $B$ is useful because it allows us to restrict the class $\Gamma_\mu$ so that each function in this class only depends on the first $\mathfrak{s}$ coordinates of the input, while yielding a regression function $\mu_P$ that may depend on any arbitrary collection of $\mathfrak{s}$ out of the $p$ total coordinates. For the equivariant class that we consider (Algorithm~\ref{alg:architecture}), enforcing this turns out to be unnecessary -- the invariance of functions in $\mathcal{T}$ to permutations of the features implies that the Bayes risk of each $T\in\mathcal{T}$ remains unchanged if the random variable $B$ defining $\Pi_{\kappa}\in\Gamma_1$ is replaced by a degenerate random variable that is always equal to the identity matrix. Nonetheless, allowing $B$ to be a random draw from ${\rm Unif}(\mathcal{B})$ allows us to ensure that our implied collection of priors $\Gamma$ satisfies \ref{in:permpred}, \ref{in:shiftrescalepred}, and \ref{in:shiftrescaleoutcome}, thereby making the implied $\Gamma$ compatible with the preservation conditions imposed in Section~\ref{sec:characterization}. 

We now use the notation of \citet{Kingma&Ba2014} to detail the hyperparameters that we used. In all settings, we set $(\beta_2,\epsilon)=(0.999,10^{-8})$. Whenever we were updating the prior network, we set the momentum parameter $\beta_1$ to $0$, and whenever we were updating the estimator network, we set the momentum parameter to $0.25$. The parameter $\alpha$ differed across settings. In the sparse linear regression setting with $\mathfrak{s}=1$, we found that choosing $\alpha$ small helped to improve stability. Specifically, we let $\alpha=0.0002$ when updating both the estimator and prior networks. In the sparse linear regression setting with $\mathfrak{s}=5$, we used the more commonly chosen parameter setting of $\alpha=0.001$ for both networks. In the FLAM example, we chose $\alpha=0.001$ and $\alpha=0.005$ for the estimator and prior networks, respectively.

The learning rates were of the estimator and prior networks were decayed at rates $t^{-0.15}$ and $t^{-0.25}$, respectively. Such two-timescale learning rate strategies have proven to be effective in stabilizing the optimization problem pursued by generative adversarial networks \citep{Heuseletal2017}. As noted in \citet{Fiezetal2019}, using two-timescale strategies can cause the optimization problem to converge to a differential Stackelberg, rather than a differential Nash, equilibrium. Indeed, under some conditions, the two-timescale strategy that we use is expected to converge to a differential Stackelberg equilibrium in the hierarchical two-player game where a prior $\Pi$ is first selected from $\Gamma$, and then an estimator $T$ is selected from $\mathcal{T}$ to perform well against $\Pi$. An optimal prior $\Pi^\star$ in this game is called $\Gamma$-least favorable, in the sense that this prior maximizes $\inf_{T\in\mathcal{T}} r(T,\cdot)$ over $\Gamma$. For a given $\Gamma$-least favorable prior $\Pi^\star$, an optimal estimator $T^\star$ in this game is a Bayes estimator against $\Pi^\star$, that is, an estimator that minimizes $r(\cdot,\Pi^\star)$ over $\mathcal{T}$. This $T^\star$ may not necessarily be a $\Gamma$-minimax strategy, that is, $T^\star$ may not minimize $\sup_{\Pi\in\Gamma} r(\cdot,\Pi)$ over $\mathcal{T}$. Nevertheless, we note that, under appropriate conditions, the two notions of optimality necessarily agree. Though such a theoretical guarantee is not likely to hold in our experiments given the neural network parameterizations that we use, we elected to use this two-timescale strategy because of the improvements in stability that we saw.

In all settings, the prior and estimator were updated over $10^6$ iterations using batches of $100$ data sets. For each data set, performance is evaluated at $100$ values of $x_0$. 

\subsection{Sparse linear regression}\label{app:linearSettings}
We now introduce notation that will be useful for presenting the collection $\Gamma_\mu$ in the sparse linear regression example. 
For a function $G : \mathbb{R}\rightarrow\mathbb{R}$ and a distribution $P_X\in\mathcal{P}_X$, we let $\kappa_G(\cdot,P_X)$ be equal to the distribution of
\begin{align*}
    x\mapsto \left(U_0\frac{\left(e^{G(U_1)},\ldots,e^{G(U_{\mathfrak{s}})},0,\ldots,0\right)}{\sum_{j=1}^{\mathfrak{s}} e^{G(U_j)}}\right)^\top x,
\end{align*}
where $U_0\sim {\rm Unif}(-5,5)$ and $(U_1,\ldots,U_\mathfrak{s})\sim N(\matr{0}_\mathfrak{s},{\rm Id}_\mathfrak{s})$ are drawn independently. Notably, here $\kappa_G(\cdot,P_X)$ does not depend on $P_X$.  We let $\Gamma_\mu:=\left\{\kappa_G : G\in\mathscr{G}\right\}$, where $\mathscr{G}$ takes different values when $\mathfrak{s}=1$ and when $\mathfrak{s}=5$. When $\mathfrak{s}=1$, $\mathscr{G}$ consists of all four-hidden layer perceptrons with identity output activation, where each hidden layer consists of forty leaky ReLU units. When $\mathfrak{s}=5$, $\mathcal{G}$ consists of all four-hidden layer neural networks with identity output activation, but in this case each layer is a multi-input-output channel equivariant layer as described in Eq.~22 of \citet{Zaheer2017}. Each hidden layer is again equipped with a ReLU activation function. The output of each such network is equivariant to permutations of the $\mathfrak{s}=5$ inputs.

In each sparse linear regression setting considered, we initialized the estimator network by pretraining for 5,000 iterations against the initial fixed prior network. After these 5,000 iterations, we then began to adversarially update the prior network against the estimator network.

Five thousand Monte Carlo replicates were used to obtain the performance estimates in Table~\ref{tab:linear}.

\subsection{Fused lasso additive model}\label{app:flamSettings}

When discussing the FLAM example, we will write $x_j$ to denote the $j^{\rm th}$ feature, that is, we denote a generic $x\in\mathcal{X}$ by $x=(x_1,\ldots,x_p)$. We emphasize this to avoid any notational confusion with the fact that, elsewhere in the text, $X_i\in\mathcal{X}$ is used to denote the random variable corresponding to the $i^{\rm th}$ observation.

In the FLAM example, each prior $\kappa_G$ in $\Gamma_\mu$ is indexed by a function $G : \mathbb{R}^{\mathfrak{s}+2}\rightarrow [0,\infty)^{\mathfrak{s}}$ belonging to the collection of four-hidden layer perceptrons with identity output activation, where each hidden layer consists of forty leaky ReLU units. 
Specifically, $\kappa_G(\cdot,P_X)$ is a distribution over generalized additive models $x\mapsto \sum_{j=1}^p \mu_j(x_j)$ for which each component $\mu_j$ is piecewise-constant and changes values at most 500 times. To obtain a draw $\mu_P$ from $\kappa_G(\cdot,P_X)$, we can first draw 500 iid observations from $P_X$ and store these observations in the matrix $\tilde{\matr{X}}$. Each component $\mu_j$ can only have a jump at the 500 points in $\tilde{\matr{X}}_{*j}$. The magnitude of each jump is defined using the function $G$ and the sign of the jump is defined uniformly at random. More specifically, these increments are defined based on the independent sources of noise $(H_{jk} : j=1,\ldots,p; k=1,\ldots,500)$, which is an iid collection of Rademacher random variables, and $(U_k : k=1,\ldots,500)$, which is an iid collection of $N(\matr{0}_{\mathfrak{s}+2},{\rm Id}_{\mathfrak{s}+2})$ random variables. The component $\mu_j$ is chosen to be proportional to the function $f_j(x_j)= \sum_{k=1}^{500} H_{jk} G(U_k)_j I\{x_j\ge \tilde{\matr{X}}_{kj}\}$. The proportionality constant $c:=\sum_{j=1}^p \sum_{k=1}^{500} G(U_k)_j$ is defined so that the function $\mu_P(x)= 10 c^{-1} \sum_{j=1}^p f_j(x_j)$ saturates the constraint $\|v(\mu)\|_1\le 10$ that is imposed by $\mathcal{R}$. To recap, the random draw $\mu_P$ from $\kappa_G(\cdot,P_X)$ can be obtained by independently drawing $\tilde{\matr{X}}$, $(H_{j,k} : j,k)$, and $(U_k : k)$, and subsequently following the steps described above to define the corresponding proportionality constant $c$ and components $f_j$, $j=1,\ldots,p$.

We evaluated the performance of the learned prediction procedures using a variant of the simulation scenarios 1-4 from the paper that introduced FLAM \citep[Fig.~2 in][]{Petersenetal2016}. As presented in that work, the four scenarios have $p$ independent ${\rm Unif}(-2.5,2.5)$ features, with the components corresponding to $\mathfrak{s}_0=4$ of these features being nonzero. These scenarios offer a range of smoothness settings, with scenarios 1-4 enforcing that the components be (1) piecewise constant, (2) smooth, (3) a mix of piecewise constant and smooth functions, and (4) constant in some areas of its domain and highly variable in others. To evaluate our procedures trained with $\|v(\mu_P)\|_0\le 5$, we used the R function \texttt{sim.data} in the \texttt{flam} package \citep{flamPackage} to generate training data from the scenarios in \citet{Petersenetal2016} with $p=10$ features. We then generated new outcomes by rescaling the regression function by a positive multiplicative constant so that $\|v(\mu_P)\|_1=10$, and subsequently added standard Gaussian noise. To evaluate our procedures trained at sparsity level $\mathfrak{s}= 1$ in a given scenario, we defined a prior over the regression function that first randomly selects one of the four signal components, then rescales this component so that it has total variation equal to 10, and then sets all other components equal to zero. Outcomes were generated by adding Gaussian noise to the sampled regression function. We compared our approach to the FLAM method as implemented in the \texttt{flam} package when, in the notation of \citet{Petersenetal2016}, $\alpha=1$ and $\lambda$ was chosen numerically to enforce that the resulting regression function estimate $\hat{\mu}$ satisfied $\|v(\hat{\mu})\|_1\approx 10$. Choosing $\lambda$ in this fashion is reasonable in light of the fact that $\|v(\mu_P)\|_1=10$ for all settings considered.

Two thousand Monte Carlo replicates were used to obtain the performance estimates in Table~\ref{tab:flamPetersens5}.

\section{Performance of symmetrized estimators in experiments}\label{app:symmetrization}

We now present the additional experimental results that we alluded to in Section~\ref{sec:extensions}. These results were obtained by symmetrizing the meta-learned AMC100 and AMC500 estimators whose performance was reported in Section~\ref{sec:examples}. In particular, we symmetrized a given AMC estimator $T$ as
\begin{align*}
    T^{\rm sym}(\matr{x},\matr{y})(x_0):= \frac{1}{2}\left[T(\matr{x},\matr{y})-T(\matr{x},-\matr{y})(x_0)\right].
\end{align*}
When reporting our experimental results, we refer to the symmetrized estimator derived from the meta-learned AMC100 and AMC500 estimators as `symmetrized AMC100' and `symmetrized AMC500', respectively. We emphasize that these symmetrized estimators are derived directly from the AMC100 and AMC500 fits that we reported in Section~\ref{sec:examples} -- we did not rerun our AMC meta-learning algorithm to obtain these estimators.

Table~\ref{tab:linearsymm} reports the results for the linear regression example. In many settings, the two approaches performed similarly. However, in the sparse setting, the improvements that resulted from symmetrization sometimes resulted in the MSE being cut in half. In one setting (dense, interior, $n=100$), AMC100 outperformed symmetrized AMC100 slightly -- though not deducible from the table, we note here that the difference in MSE in this case was less than $0.003$, and it seems likely that this discrepancy is a result of Monte Carlo error. Table~\ref{tab:linearsymm} reports the results for the fused lasso additive model example. Symmetrization led to a reduction in MSE in most settings. In all other settings, the MSE remained unchanged.

\begin{table}[htb]
    \centering
    \begin{subtable}{\linewidth}\centering
    \caption{Sparse signal}
    \begin{tabular}{lllll}
    & \multicolumn{2}{c}{Boundary} & \multicolumn{2}{c}{Interior} \\
     & $n$=100 & 500 & 100 & 500 \\\hline
        OLS & 0.12 & \phantom{$<$}0.02 & 0.12 & 0.02  \\
        Lasso & 0.06 & \phantom{$<$}0.01 & 0.06 & 0.01  \\
        AMC100 (ours) & 0.02 & $<$0.01 & 0.11 & 0.09  \\
        Symmetrized AMC100 (ours) & 0.02 & $<$0.01 & \textbf{0.06} & \textbf{0.04}  \\
        AMC500 (ours) & 0.02 & $<$0.01   & 0.07 & 0.04 \\
        Symmetrized AMC500 (ours) & 0.02 & $<$0.01 & \textbf{0.06} & \textbf{0.03}  \\
    \end{tabular}
    \end{subtable}
    \begin{subtable}{\linewidth}\centering
    \caption{Dense signal}
    \begin{tabular}{lllll}
    & \multicolumn{2}{c}{Boundary} & \multicolumn{2}{c}{Interior} \\
     & $n$=100 & 500 & 100 & 500 \\\hline
        OLS & 0.13 & 0.02 & 0.13 & 0.02  \\
        Lasso & 0.11 & 0.02 & 0.09 & 0.02  \\
        AMC100 (ours) & 0.10 & 0.04 & \textbf{0.08} & 0.02  \\
        Symmetrized AMC100 (ours) & \textbf{0.09} & \textbf{0.03} & 0.09 & 0.02  \\
        AMC500 (ours) & 0.09 & 0.02 & 0.09 & 0.02 \\
        Symmetrized AMC500 (ours) & 0.09 & 0.02 & 0.09 & 0.02q  
    \end{tabular}
    \end{subtable}
    \caption{MSEs based on data sets of size $n$ in the linear regression settings. All Monte Carlo standard errors are less than 0.001. Symmetrized AMC100 entries appear in bold when they had lower MSE (rounded to the nearest hundredth) than the corresponding AMC100 entry, and vice versa. Similarly, symmetrized AMC500 entries appear in bold when they had lower MSE than the corresponding AMC500 entry, and vice versa.}
    \label{tab:linearsymm}
\end{table}

\begin{table*}[b!]
    \centering
    \begin{subtable}{\linewidth}\centering
    \caption{Sparse signal}
    \begin{tabular}{lllllllll}
     & \multicolumn{2}{c}{Scenario 1} & \multicolumn{2}{c}{Scenario 2} & \multicolumn{2}{c}{Scenario 3} & \multicolumn{2}{c}{Scenario 4} \\
     & $n$=100 & 500 & 100 & 500 & 100 & 500 & 100 & 500 \\\hline
        FLAM & 0.44 & 0.12 & 0.47 & 0.17 & 0.38 & 0.11 & 0.51 & 0.19  \\
        AMC100 (ours) & 0.34 & 0.20 & 0.18 & 0.08 & 0.27 & 0.14 & 0.17 & 0.08  \\
        Symmetrized AMC100 (ours) & \textbf{0.32} & \textbf{0.18} & 0.18 & 0.08 & \textbf{0.26} & \textbf{0.13} & \textbf{0.16} & 0.08  \\
        AMC500 (ours) & 0.48 & 0.12 & 0.19 & 0.06 & 0.35 & 0.10 & 0.23 & 0.08 \\
        Symmetrized AM5100 (ours) & \textbf{0.43} & 0.12 & \textbf{0.17} & \textbf{0.05} & \textbf{0.32} & \textbf{0.09} & \textbf{0.21} & \textbf{0.07}  
    \end{tabular}
    \end{subtable}
    \begin{subtable}{\linewidth}\centering
    \caption{Dense signal}
    \begin{tabular}{lllllllll}
     & \multicolumn{2}{c}{Scenario 1} & \multicolumn{2}{c}{Scenario 2} & \multicolumn{2}{c}{Scenario 3} & \multicolumn{2}{c}{Scenario 4} \\
     & $n$=100 & 500 & 100 & 500 & 100 & 500 & 100 & 500 \\\hline
        FLAM & 0.59 & 0.17 & 0.65 & 0.24  & 0.53 & 0.16 & 0.76 & 0.36 \\
        AMC100 (ours) & 1.20 & 0.91 & 0.47 & 0.39  & 0.87 & 0.57 & 0.30 & 0.30 \\
        Symmetrized AMC100 (ours) & \textbf{1.16} & \textbf{0.84} & \textbf{0.45} & \textbf{0.37} & \textbf{0.83} & \textbf{0.52} & \textbf{0.29} & 0.30  \\
        AMC500 (ours) & 0.58 & 0.15 & 0.37 & 0.08  & 0.46 & 0.12 & 0.36 & 0.09 \\
        Symmetrized AM5100 (ours) & \textbf{0.55} & 0.15 & \textbf{0.36} & 0.08 & \textbf{0.43} & \textbf{0.11} & \textbf{0.34} & 0.09 
    \end{tabular}
    \end{subtable}
    \caption{MSEs based on data sets of size $n$ in the FLAM settings. The Monte Carlo standard errors for the MSEs of FLAM and (symmetrized) AMC are all less than 0.04 and 0.01, respectively. Symmetrized AMC100 entries appear in bold when they had lower MSE (rounded to the nearest hundredth) than the corresponding AMC100 entry, and vice versa. Similarly, symmetrized AMC500 entries appear in bold when they had lower MSE than the corresponding AMC500 entry, and vice versa.}
    \label{tab:flamPetersens5symm}
\end{table*}

\bibliography{minipred}

\begin{thebibliography}{66}
\providecommand{\natexlab}[1]{#1}
\providecommand{\url}[1]{\texttt{#1}}
\expandafter\ifx\csname urlstyle\endcsname\relax
  \providecommand{\doi}[1]{doi: #1}\else
  \providecommand{\doi}{doi: \begingroup \urlstyle{rm}\Url}\fi

\bibitem[Berger(1985)]{Berger1985}
J.~O. Berger.
\newblock \emph{Statistical Decision Theory and Bayesian Analysis}.
\newblock Springer Science \& Business Media, 1985.

\bibitem[Bertinetto et~al.(2018)Bertinetto, Henriques, Torr, and
  Vedaldi]{Bertinettoetal2018}
L.~Bertinetto, J.~F. Henriques, P.~H. Torr, and A.~Vedaldi.
\newblock Meta-learning with differentiable closed-form solvers.
\newblock \emph{arXiv preprint arXiv:1805.08136}, 2018.

\bibitem[Billingsley(1999)]{Billingsley1999}
P.~Billingsley.
\newblock \emph{Convergence of probability measures}.
\newblock Wiley, 1999.

\bibitem[Bosc(2016)]{Bosc2016}
T.~Bosc.
\newblock Learning to learn neural networks.
\newblock \emph{arXiv preprint arXiv:1610.06072}, 2016.

\bibitem[Breiman(1996)]{breiman1996stacked}
L.~Breiman.
\newblock Stacked regressions.
\newblock \emph{Machine learning}, 24\penalty0 (1):\penalty0 49--64, 1996.

\bibitem[Breiman(2001)]{Breiman2001}
L.~Breiman.
\newblock Random forests.
\newblock \emph{Machine learning}, 45\penalty0 (1):\penalty0 5--32, 2001.

\bibitem[Brooks et~al.(1989)Brooks, Pope, and Marcolini]{brooks1989airfoil}
T.~F. Brooks, D.~S. Pope, and M.~A. Marcolini.
\newblock Airfoil self-noise and prediction.
\newblock 1989.

\bibitem[Cassotti et~al.(2015)Cassotti, Ballabio, Todeschini, and
  Consonni]{cassotti2015similarity}
M.~Cassotti, D.~Ballabio, R.~Todeschini, and V.~Consonni.
\newblock A similarity-based qsar model for predicting acute toxicity towards
  the fathead minnow (pimephales promelas).
\newblock \emph{SAR and QSAR in Environmental Research}, 26\penalty0
  (3):\penalty0 217--243, 2015.

\bibitem[Chamberlain(2000)]{Chamberlain2000}
G.~Chamberlain.
\newblock Econometric applications of maxmin expected utility.
\newblock \emph{Journal of Applied Econometrics}, 15\penalty0 (6):\penalty0
  625--644, 2000.

\bibitem[Chang(2006)]{Chang2006}
K.-C. Chang.
\newblock \emph{Methods in nonlinear analysis}.
\newblock Springer Science \& Business Media, 2006.

\bibitem[Cohn(2013)]{Cohn2013}
D.~L. Cohn.
\newblock \emph{Measure theory}.
\newblock Springer, 2013.

\bibitem[Conway(2010)]{conway2010course}
J.~B. Conway.
\newblock \emph{A course in functional analysis}, volume~96.
\newblock Springer, 2010.

\bibitem[Cybenko(1989)]{Cybenko1989}
G.~Cybenko.
\newblock Approximation by superpositions of a sigmoidal function.
\newblock \emph{Mathematics of control, signals and systems}, 2\penalty0
  (4):\penalty0 303--314, 1989.

\bibitem[Dalvi et~al.(2004)Dalvi, Domingos, Sanghai, and Verma]{Dalvietal2004}
N.~Dalvi, P.~Domingos, S.~Sanghai, and D.~Verma.
\newblock Adversarial classification.
\newblock In \emph{Proceedings of the tenth ACM SIGKDD international conference
  on Knowledge discovery and data mining}, pages 99--108, 2004.

\bibitem[Day(1961)]{Day1961}
M.~M. Day.
\newblock Fixed-point theorems for compact convex sets.
\newblock \emph{Illinois Journal of Mathematics}, 5\penalty0 (4):\penalty0
  585--590, 1961.

\bibitem[Dua and Graff(2017)]{Dua2019}
D.~Dua and C.~Graff.
\newblock {UCI} machine learning repository, 2017.
\newblock URL \url{http://archive.ics.uci.edu/ml}.

\bibitem[Efron and Morris(1972)]{efron1972limiting}
B.~Efron and C.~Morris.
\newblock Limiting the risk of bayes and empirical bayes estimators—part ii:
  The empirical bayes case.
\newblock \emph{Journal of the American Statistical Association}, 67\penalty0
  (337):\penalty0 130--139, 1972.

\bibitem[Fan(1953)]{Fan1953}
K.~Fan.
\newblock Minimax theorems.
\newblock \emph{Proceedings of the National Academy of Sciences of the United
  States of America}, 39\penalty0 (1):\penalty0 42, 1953.

\bibitem[Fiez et~al.(2019)Fiez, Chasnov, and Ratliff]{Fiezetal2019}
T.~Fiez, B.~Chasnov, and L.~J. Ratliff.
\newblock Convergence of learning dynamics in stackelberg games.
\newblock \emph{arXiv preprint arXiv:1906.01217}, 2019.

\bibitem[Finn et~al.(2017)Finn, Abbeel, and Levine]{Finnetal2017}
C.~Finn, P.~Abbeel, and S.~Levine.
\newblock Model-agnostic meta-learning for fast adaptation of deep networks.
\newblock In \emph{Proceedings of the 34th International Conference on Machine
  Learning-Volume 70}, pages 1126--1135. JMLR. org, 2017.

\bibitem[Friedman(2001)]{Friedman2001}
J.~H. Friedman.
\newblock Greedy function approximation: a gradient boosting machine.
\newblock \emph{Annals of statistics}, pages 1189--1232, 2001.

\bibitem[Geng et~al.(2020)Geng, Nassif, Manzanares, Reppen, and
  Sircar]{geng2020deep}
S.~Geng, H.~Nassif, C.~A. Manzanares, A.~M. Reppen, and R.~Sircar.
\newblock Deep pqr: Solving inverse reinforcement learning using anchor
  actions.
\newblock \emph{arXiv e-prints}, pages arXiv--2007, 2020.

\bibitem[Gerritsma et~al.(1981)Gerritsma, Onnink, and
  Versluis]{gerritsma1981geometry}
J.~Gerritsma, R.~Onnink, and A.~Versluis.
\newblock Geometry, resistance and stability of the delft systematic yacht hull
  series.
\newblock \emph{International shipbuilding progress}, 28\penalty0
  (328):\penalty0 276--297, 1981.

\bibitem[Glynn(1987)]{Glynn1987}
P.~W. Glynn.
\newblock Likelihood ratio gradient estimation: an overview.
\newblock In \emph{Proceedings of the 19th conference on Winter simulation},
  pages 366--375. ACM, 1987.

\bibitem[Goldblum et~al.(2019)Goldblum, Fowl, and Goldstein]{Goldblumetal2019}
M.~Goldblum, L.~Fowl, and T.~Goldstein.
\newblock Adversarially robust few-shot learning: A meta-learning approach.
\newblock \emph{arXiv preprint arXiv:1910.00982v2}, 2019.

\bibitem[Goodfellow et~al.(2014)Goodfellow, Shlens, and
  Szegedy]{Goodfellowetal2014}
I.~J. Goodfellow, J.~Shlens, and C.~Szegedy.
\newblock Explaining and harnessing adversarial examples.
\newblock \emph{arXiv preprint arXiv:1412.6572}, 2014.

\bibitem[Guo and Zhao(2019)]{guo2019improvement}
W.~Guo and G.~Zhao.
\newblock An improvement on the relatively compactness criteria.
\newblock \emph{arXiv preprint arXiv:1904.03427}, 2019.

\bibitem[Hartford et~al.(2018)Hartford, Graham, Leyton-Brown, and
  Ravanbakhsh]{Hartford2018}
J.~Hartford, D.~R. Graham, K.~Leyton-Brown, and S.~Ravanbakhsh.
\newblock Deep models of interactions across sets.
\newblock \emph{arXiv preprint arXiv:1803.02879}, 2018.

\bibitem[Heusel et~al.(2017)Heusel, Ramsauer, Unterthiner, Nessler, and
  Hochreiter]{Heuseletal2017}
M.~Heusel, H.~Ramsauer, T.~Unterthiner, B.~Nessler, and S.~Hochreiter.
\newblock Gans trained by a two time-scale update rule converge to a local nash
  equilibrium.
\newblock In \emph{Advances in neural information processing systems}, pages
  6626--6637, 2017.

\bibitem[Hochreiter and Schmidhuber(1997)]{Hochreiter&Schmidhuber1997}
S.~Hochreiter and J.~Schmidhuber.
\newblock Long short-term memory.
\newblock \emph{Neural computation}, 9\penalty0 (8):\penalty0 1735--1780, 1997.

\bibitem[Hochreiter et~al.(2001)Hochreiter, Younger, and
  Conwell]{Hochreiteretal2001}
S.~Hochreiter, A.~S. Younger, and P.~R. Conwell.
\newblock Learning to learn using gradient descent.
\newblock In \emph{International Conference on Artificial Neural Networks},
  pages 87--94. Springer, 2001.

\bibitem[Hornik(1991)]{Hornik1991}
K.~Hornik.
\newblock Approximation capabilities of multilayer feedforward networks.
\newblock \emph{Neural networks}, 4\penalty0 (2):\penalty0 251--257, 1991.

\bibitem[Hunt and Stein(1946)]{Hunt&Stein1946}
G.~Hunt and C.~Stein.
\newblock Most stringent tests of statistical hypotheses.
\newblock \emph{Unpublished manuscript}, 1946.

\bibitem[Kempthorne(1987)]{Kempthorne1987}
P.~J. Kempthorne.
\newblock Numerical specification of discrete least favorable prior
  distributions.
\newblock \emph{SIAM Journal on Scientific and Statistical Computing},
  8\penalty0 (2):\penalty0 171--184, 1987.

\bibitem[Kingma and Ba(2014)]{Kingma&Ba2014}
D.~P. Kingma and J.~Ba.
\newblock Adam: A method for stochastic optimization.
\newblock \emph{arXiv preprint arXiv:1412.6980}, 2014.

\bibitem[Le~Cam(2012)]{LeCam2012}
L.~Le~Cam.
\newblock \emph{Asymptotic methods in statistical decision theory}.
\newblock Springer Science \& Business Media, 2012.

\bibitem[Lee et~al.(2019)Lee, Maji, Ravichandran, and Soatto]{Leeetal2019}
K.~Lee, S.~Maji, A.~Ravichandran, and S.~Soatto.
\newblock Meta-learning with differentiable convex optimization.
\newblock In \emph{Proceedings of the IEEE Conference on Computer Vision and
  Pattern Recognition}, pages 10657--10665, 2019.

\bibitem[Lin et~al.(2019)Lin, Jin, and Jordan]{Linetal2019}
T.~Lin, C.~Jin, and M.~I. Jordan.
\newblock On gradient descent ascent for nonconvex-concave minimax problems.
\newblock \emph{arXiv preprint arXiv:1906.00331v6}, 2019.

\bibitem[Luedtke et~al.(2020)Luedtke, Carone, Simon, and
  Sofrygin]{Luedtkeetal2019}
A.~Luedtke, M.~Carone, N.~R. Simon, and O.~Sofrygin.
\newblock Learning to learn from data: using deep adversarial learning to
  construct optimal statistical procedures.
\newblock \emph{Science Advances \textnormal{(in press; available online late
  Feb or Mar 2020)}}, 2020.

\bibitem[Maron et~al.(2019)Maron, Fetaya, Segol, and Lipman]{Maron2019}
H.~Maron, E.~Fetaya, N.~Segol, and Y.~Lipman.
\newblock On the universality of invariant networks.
\newblock \emph{arXiv preprint arXiv:1901.09342}, 2019.

\bibitem[Munkres(2000)]{Munkres2000}
J.~Munkres.
\newblock \emph{Topology}.
\newblock Featured Titles for Topology Series. Prentice Hall, Incorporated,
  2000.
\newblock ISBN 9780131816299.
\newblock URL \url{https://books.google.com/books?id=XjoZAQAAIAAJ}.

\bibitem[Nabi et~al.(2020)Nabi, Nassif, Hong, Mamani, and
  Imbens]{nabi2020decoupling}
S.~Nabi, H.~Nassif, J.~Hong, H.~Mamani, and G.~Imbens.
\newblock Decoupling learning rates using empirical bayes priors.
\newblock \emph{arXiv preprint arXiv:2002.01129}, 2020.

\bibitem[Nelder and Wedderburn(1972)]{Nelder&Wedderburn1972}
J.~A. Nelder and R.~W. Wedderburn.
\newblock Generalized linear models.
\newblock \emph{Journal of the Royal Statistical Society: Series A (General)},
  135\penalty0 (3):\penalty0 370--384, 1972.

\bibitem[Nelson(1966)]{Nelson1966}
W.~Nelson.
\newblock Minimax solution of statistical decision problems by iteration.
\newblock \emph{The Annals of Mathematical Statistics}, pages 1643--1657, 1966.

\bibitem[Noubiap and Seidel(2001)]{Noubiap&Seidel2001}
R.~F. Noubiap and W.~Seidel.
\newblock An algorithm for calculating $\gamma$-minimax decision rules under
  generalized moment conditions.
\newblock \emph{The Annals of Statistics}, 29\penalty0 (4):\penalty0
  1094--1116, 2001.

\bibitem[Pedregosa et~al.(2011)Pedregosa, Varoquaux, Gramfort, Michel, Thirion,
  Grisel, Blondel, Prettenhofer, Weiss, Dubourg, Vanderplas, Passos,
  Cournapeau, Brucher, Perrot, and Duchesnay]{scikit-learn}
F.~Pedregosa, G.~Varoquaux, A.~Gramfort, V.~Michel, B.~Thirion, O.~Grisel,
  M.~Blondel, P.~Prettenhofer, R.~Weiss, V.~Dubourg, J.~Vanderplas, A.~Passos,
  D.~Cournapeau, M.~Brucher, M.~Perrot, and E.~Duchesnay.
\newblock Scikit-learn: Machine learning in {P}ython.
\newblock \emph{Journal of Machine Learning Research}, 12:\penalty0 2825--2830,
  2011.

\bibitem[Petersen(2018)]{flamPackage}
A.~Petersen.
\newblock \emph{flam: Fits Piecewise Constant Models with Data-Adaptive Knots},
  2018.
\newblock URL \url{https://CRAN.R-project.org/package=flam}.
\newblock R package version 3.2.

\bibitem[Petersen et~al.(2016)Petersen, Witten, and Simon]{Petersenetal2016}
A.~Petersen, D.~Witten, and N.~Simon.
\newblock Fused lasso additive model.
\newblock \emph{Journal of Computational and Graphical Statistics}, 25\penalty0
  (4):\penalty0 1005--1025, 2016.

\bibitem[Pier(1984)]{Pier1984}
J.-P. Pier.
\newblock \emph{Amenable locally compact groups}.
\newblock Wiley-Interscience, 1984.

\bibitem[Ravanbakhsh et~al.(2016)Ravanbakhsh, Schneider, and
  Poczos]{Ravanbakhsh2016}
S.~Ravanbakhsh, J.~Schneider, and B.~Poczos.
\newblock Deep learning with sets and point clouds.
\newblock \emph{arXiv preprint arXiv:1611.04500}, 2016.

\bibitem[Ravanbakhsh et~al.(2017)Ravanbakhsh, Schneider, and
  Poczos]{Ravanbakhsh2017}
S.~Ravanbakhsh, J.~Schneider, and B.~Poczos.
\newblock Equivariance through parameter-sharing.
\newblock In \emph{Proceedings of the 34th International Conference on Machine
  Learning-Volume 70}, pages 2892--2901. JMLR. org, 2017.

\bibitem[Ravi and Larochelle(2017)]{Ravi&Larochelle2016}
S.~Ravi and H.~Larochelle.
\newblock Optimization as a model for few-shot learning.
\newblock In \emph{International Conference on Learning Representations
  (ICLR)}, 2017.

\bibitem[Russell(1998)]{russell1998learning}
S.~Russell.
\newblock Learning agents for uncertain environments.
\newblock In \emph{Proceedings of the eleventh annual conference on
  Computational learning theory}, pages 101--103, 1998.

\bibitem[Santoro et~al.(2016)Santoro, Bartunov, Botvinick, Wierstra, and
  Lillicrap]{Santoroetal2016}
A.~Santoro, S.~Bartunov, M.~Botvinick, D.~Wierstra, and T.~Lillicrap.
\newblock Meta-learning with memory-augmented neural networks.
\newblock In \emph{International conference on machine learning}, pages
  1842--1850, 2016.

\bibitem[Schafer and Stark(2009)]{Schafer2009}
C.~M. Schafer and P.~B. Stark.
\newblock Constructing confidence regions of optimal expected size.
\newblock \emph{Journal of the American Statistical Association}, 104\penalty0
  (487):\penalty0 1080--1089, 2009.

\bibitem[Schmidhuber(1987)]{Schmidhuber1987}
J.~Schmidhuber.
\newblock \emph{Evolutionary principles in self-referential learning, or on
  learning how to learn: the meta-meta-... hook}.
\newblock PhD thesis, Technische Universit{\"a}t M{\"u}nchen, 1987.

\bibitem[Terkelsen(1973)]{Terkelsen1973}
F.~Terkelsen.
\newblock Some minimax theorems.
\newblock \emph{Mathematica Scandinavica}, 31\penalty0 (2):\penalty0 405--413,
  1973.

\bibitem[Thrun and Pratt(1998)]{Thrun&Pratt1998}
S.~Thrun and L.~Pratt.
\newblock Learning to learn: Introduction and overview.
\newblock In \emph{Learning to learn}, pages 3--17. Springer, 1998.

\bibitem[Tibshirani(1996)]{Tibshirani1996}
R.~Tibshirani.
\newblock Regression shrinkage and selection via the lasso.
\newblock \emph{Journal of the Royal Statistical Society: Series B
  (Methodological)}, 58\penalty0 (1):\penalty0 267--288, 1996.

\bibitem[Van~der Laan et~al.(2007)Van~der Laan, Polley, and
  Hubbard]{van2007super}
M.~J. Van~der Laan, E.~C. Polley, and A.~E. Hubbard.
\newblock Super learner.
\newblock \emph{Statistical applications in genetics and molecular biology},
  6\penalty0 (1), 2007.

\bibitem[Van~der Vaart et~al.(2006)Van~der Vaart, Dudoit, and van~der
  Laan]{van2006oracle}
A.~W. Van~der Vaart, S.~Dudoit, and M.~J. van~der Laan.
\newblock Oracle inequalities for multi-fold cross validation.
\newblock \emph{Statistics and Decisions}, 24\penalty0 (3):\penalty0 351--371,
  2006.

\bibitem[van Gaans(2003)]{vanGaans2003}
O.~van Gaans.
\newblock Probability measures on metric spaces.
\newblock Technical report, Technical report, Delft University of Technology,
  2003.

\bibitem[Vilalta and Drissi(2002)]{Vilaltaetal2002}
R.~Vilalta and Y.~Drissi.
\newblock A perspective view and survey of meta-learning.
\newblock \emph{Artificial intelligence review}, 18\penalty0 (2):\penalty0
  77--95, 2002.

\bibitem[Vinyals et~al.(2016)Vinyals, Blundell, Lillicrap, and
  Wierstra]{Vinyalsetal2016}
O.~Vinyals, C.~Blundell, T.~Lillicrap, and D.~Wierstra.
\newblock Matching networks for one shot learning.
\newblock In \emph{Advances in neural information processing systems}, pages
  3630--3638, 2016.

\bibitem[Yin et~al.(2018)Yin, Tang, Xu, and Wang]{Yinetal2018}
C.~Yin, J.~Tang, Z.~Xu, and Y.~Wang.
\newblock Adversarial meta-learning.
\newblock \emph{arXiv preprint arXiv:1806.03316}, 2018.

\bibitem[Zaheer et~al.(2017)Zaheer, Kottur, Ravanbakhsh, Poczos, Salakhutdinov,
  and Smola]{Zaheer2017}
M.~Zaheer, S.~Kottur, S.~Ravanbakhsh, B.~Poczos, R.~R. Salakhutdinov, and A.~J.
  Smola.
\newblock Deep sets.
\newblock In \emph{Advances in neural information processing systems}, pages
  3391--3401, 2017.

\end{thebibliography}

\end{document}